\pdfoutput=1
\documentclass{article}

\usepackage[letterpaper,margin=1.0in]{geometry}

\usepackage[utf8]{inputenc}
\usepackage[T1]{fontenc}
\usepackage{lmodern}
\usepackage{microtype}

\usepackage{amsmath}
\usepackage{amssymb}
\usepackage{amsfonts}
\usepackage{amsthm}
\usepackage{mathtools}

\usepackage{graphicx}
\usepackage{booktabs}
\usepackage{multirow}
\usepackage{hhline}
\usepackage{caption}
\usepackage[para,online,flushleft]{threeparttable}
\usepackage[table]{xcolor}
\usepackage{algorithm}
\usepackage{algorithmic}

\usepackage{enumitem}
\usepackage{bbding}

\usepackage[numbers,sort]{natbib}

\usepackage[hidelinks]{hyperref}
\usepackage{amsmath}\allowdisplaybreaks
\usepackage{amsfonts,bm}

\def\1{\bf{1}}

\newcommand{\Norm}[1]{\left\| #1 \right\|}

\def\inner#1#2{\langle #1, #2 \rangle}

\def\vx{{\bf{x}}}
\def\vy{{\bf{y}}}

\def\fB{{\mathcal{B}}}

\def\fX{{\mathcal{X}}}

\def\BP{{\mathbb{P}}}

\def\BR{{\mathbb{R}}}

\DeclareMathOperator*{\argmin}{arg\,min}

\DeclareMathOperator{\Tr}{Tr}

\usepackage{etoolbox}

\usepackage{amsthm}

\theoremstyle{plain}
\newtheorem{theorem}{Theorem}[section]
\newtheorem{lemma}[theorem]{Lemma}

\newtheorem{corollary}[theorem]{Corollary}
\newtheorem{assumption}[theorem]{Assumption}

\newtheorem{problem}[theorem]{Problem}

\theoremstyle{definition}
\newtheorem{definition}[theorem]{Definition}

\theoremstyle{remark}
\newtheorem{remark}[theorem]{Remark}

\makeatletter
\def\Ddots{\mathinner{\mkern1mu\raise\p@
\vbox{\kern7\p@\hbox{.}}\mkern2mu
\raise4\p@\hbox{.}\mkern2mu\raise7\p@\hbox{.}\mkern1mu}}
\makeatother

\makeatletter
\newcommand*{\rom}[1]{\expandafter\@slowromancap\romannumeral #1@}
\makeatother

\def\vx {{{\bf x}}}

\def\vy {{{\bf y}}}

\def\BR{{\mathbb{R}}}

\def\A{{\bf A}}

\def\g{{\bf g}}

\def\v{{\bf v}}

\def\X{{\bf X}}
\def\x{{\bf x}}
\def\Y{{\bf Y}}
\def\y{{\bf y}}
\def\Z{{\bf Z}}
\def\z{{\bf z}}
\def\0{{\bf 0}}
\def\1{{\bf 1}}

\def\OM{{\mathcal O}}

\def\CB{{\mathbb C}}
\def\RB{{\mathbb R}}
\def\EB{{\mathbb E}}

\def \bepsilon{{\boldsymbol{\epsilon}}}
\def \bmu{{\boldsymbol{\mu}}}
\def \bzeta{{\boldsymbol{\zeta}}}

\usepackage{algorithm}
\usepackage{algorithmic}
\usepackage{enumitem}
\usepackage{hhline}

\def \EBP #1{\EB\left[#1\right]}

\usepackage{braket}

\newcommand{\myNorm}[1]{\left\|#1 \right\|}
\usepackage{multirow}

\title{Quantum Speedups for Stochastic Optimization\\
with Heavy-Tailed Noise}

\author{%
  Bin Luo\\
The Chinese University of Hong Kong\\
\texttt{binluo@link.cuhk.edu.hk}\\
\and 
Chengchang Liu\thanks{Chengchang Liu is the corresponding author.}\\
Westlake University\\
\texttt{liuchengchang@westlake.edu.cn}\\
\and
Jonathan Allcock\\
Tencent Quantum Laboratory\\
\texttt{jonallcock@tencent.com}\\
\and
Shengyu Zhang\\
Tencent Quantum Laboratory\\
\texttt{shengyzhang@tencent.com}\\
\and
John C.S. Lui\\
The Chinese University of Hong Kong\\
\texttt{cslui@cse.cuhk.edu.hk}\\
}

\date{}

\begin{document}

\maketitle

\begin{abstract}
We study stochastic optimization with heavy-tailed gradient noise. 
We first propose a novel quantum mean estimator for multivariate heavy-tailed random variables that achieves lower query complexity than optimal classical estimators in the low-dimensional regime.
We further develop an unbiased quantum mean estimator by applying a generalized multi-level Monte Carlo technique. 
We prove quantum lower bounds showing that, when the dimension $d$ of the random vector is small and can be viewed as a constant, our quantum estimators are optimal up to logarithmic factors. We further
derive stronger dimension-dependent lower bounds for tail index
$p>4/3$, showing that a nontrivial dependence on the dimension is
unavoidable in the low-dimensional regime.
Based on these estimators, we propose a quantum normalized stochastic gradient descent method ($\texttt{QNSGD}$), which finds an $\epsilon$-stationary point using $\tilde{\mathcal{O}}\big(\sqrt d\,\epsilon^{-\frac{5p-4}{2p-2}}\big)$ queries to the quantum stochastic gradient oracle.
For a convex objective function, we propose a quantum projected stochastic gradient descent method ($\texttt{QPSGD}$), which computes a solution with $\epsilon$-optimal solution using $\tilde{\mathcal{O}}\big(\sqrt d\,\epsilon^{-\frac{3p-2}{2p-2}}+\epsilon^{-2}\big)$ queries in expectation.
These sharper bounds improve upon the classical lower bounds $\Omega\big(\epsilon^{-\frac{3p-2}{p-1}}\big)$ for nonconvex problems and $\Omega\big(\epsilon^{-\frac{p}{p-1}}\big)$ for convex problems in the low-dimensional regimes $d\lesssim\epsilon^{-\frac{p}{p-1}}$ and $d\lesssim\epsilon^{-\frac{2-p}{p-1}}$, respectively.
\end{abstract}

\section{Introduction}
We consider the stochastic optimization problem
\begin{align}
\label{eq:obj}
\min_{\x\in\fX\subseteq\RB^d} f(\x) = \EB_{\xi\sim \Xi} \left[F(\x;\xi)\right],
\end{align}
where $f$ may be non-convex and $\xi$ denotes a random seed drawn from a distribution $\Xi$.
Traditionally, convergence analyses of stochastic first-order methods rely on the bounded-variance assumption for stochastic gradients~\cite{ghadimi2013stochastic}. However, empirical evidence from image classification~\cite{simsekli2019tail}, large language models~\cite{zhang2020adaptive,ahnlinear}, and reinforcement learning~\cite{garg2021proximal} indicates that stochastic gradients often exhibit heavy-tailed behavior, thereby violating the bounded-variance assumption. Moreover, the hardness example in \cite{zhang2020adaptive} shows that standard stochastic gradient descent (\texttt{SGD}) with norm-agnostic step-size rules may fail to converge under heavy-tailed gradient noise.
Motivated by these observations, we study the stochastic optimization problem \eqref{eq:obj} under a heavy-tailed stochastic first-order oracle, which we formalize below.

\begin{assumption}
\label{ass:heavy-tailed}
We assume the stochastic gradient $\nabla F(\x;\xi)$ follows the heavy-tailed distribution such that $\EB_{\xi\sim \Xi}\left[\nabla F(\x;\xi)\right] = \nabla f(\x)\in \partial f(\x)$ and 
   \begin{align}
    \EB_{\xi\sim \Xi}\left[\|\nabla F(\x;\xi)-\nabla f(\vx)\|^p\right]\leq \sigma^{p},
\end{align}
for some $\sigma>0$ and $p\in (1,2]$.
\end{assumption}

Inspired by the importance of stochastic optimization under heavy-tailed gradient noise, numerous methods in~classical computing have been proposed in recent years.
\citet{zhang2020adaptive} showed that \texttt{SGD} with gradient clipping can find an $\epsilon$-stationary point in expectation.
Subsequent works \cite{cutkosky2021high,sadiev2023high,nguyen2023improved} developed refined clipping-based algorithms that establish high-probability convergence guaranties.
In addition, \citet{hubler2024gradient} proved that normalized \texttt{SGD} achieves the same oracle complexity as gradient clipping in the non-convex setting. More recently, \citet{liu2025online} studied online convex optimization with heavy-tailed noise and employed the online-to-nonconvex (O2NC) framework~\cite{cutkosky2023optimal} to recover optimal oracle complexity.
Despite these advances, under the stochastic first-order oracle model, all existing methods are subject to the lower bounds $\Omega\big(\epsilon^{-\frac{3p-2}{p-1}}\big)$ for non-convex optimization and $\Omega\big(\epsilon^{-\frac{p}{p-1}}\big)$ for convex optimization, as established in \cite{raginsky2009information, zhang2020adaptive}.


Recently, quantum computing has demonstrated provable speedups for a broad range of optimization problems, including convex \cite{chambolle2011first,van2020convex,chakrabarti2020quantum,zhang2024quantum,kimfast,augustino2025fast}, non-convex \cite{sidford2023quantum,gong2025robustness}, non-smooth \cite{chakrabarti2020quantum,liu2024quantum,leng2025quantum,ozgul2025quantum}, and minimax optimization \cite{li2021sublinear,liuquantum,gao2023logarithmic,wangnear}.
In particular, under the bounded-variance setting, 
\citet{sidford2023quantum} showed quantum speedups for stochastic optimization by using quantum mean estimators~\cite{Cornelissen_2022} to construct gradient estimators with a smaller batch size.

However, it remains an open question \textit{whether one can demonstrate a quantum speedup on stochastic optimization with heavy-tailed noise}. First, \textbf{no} quantum mean estimation algorithms are currently available for multivariate random variables with heavy-tailed distributions.
Although \citet{blanchet2024quadraticspeedupinfinitevariance} and \citet{wu2023quantum} proposed quantum mean estimators for heavy-tailed data, their methods apply only to univariate random~variables.
{For the multivariate case, all existing quantum mean estimators are designed for the bounded-variance setting~\cite{Cornelissen_2022,sidford2023quantum,tang2025moreefficientquantummultivariatemean}. 
However, these estimators may not apply to the heavy-tailed distributions with only a bounded $p$-th central moment for $p\in(1,2)$, because the variance of such a heavy-tailed random vector need not exist or be finite.}
Second, most existing classical algorithms use only a single stochastic gradient per iteration~\cite{zhang2020adaptive,cutkosky2020momentum,sadiev2023high,fatkhullin2025can,liu2025online}, which limits the applicability of quantum mean estimators to improve the sample complexity.

In this paper, we address this open question by developing novel quantum mean estimators for multivariate heavy-tailed random variables and leveraging them to design quantum stochastic optimization algorithms. We summarize our contributions as follows.

\begin{itemize}
\item 
We propose a novel quantum heavy-tailed mean estimator (\texttt{QHTME}) for estimating the mean of a $d$-dimensional random variable $\X$, satisfying $\EBP{\Norm{\X-\EBP{\X}}^{p}}\leq \sigma^{p}$ for some $\sigma>0\text{ and } p\in(1,2]$, given access to a quantum sampling oracle (cf. Definition~\ref{def:quantum-sampling-oracle}).
\texttt{QHTME} generalizes the multivariate bounded-variance estimators in~\cite{Cornelissen_2022,tang2025moreefficientquantummultivariatemean} and the univariate heavy-tailed estimator of~\cite{blanchet2024quadraticspeedupinfinitevariance}, corresponding to the special cases $p=2$ and $d=1$, respectively. Moreover, \texttt{QHTME} achieves a provable quantum speedup over optimal classical heavy-tailed mean estimators in the low-dimensional regime. 
A~detailed comparison with existing classical and quantum estimators is given in Table~\ref{tab:mean_estimators_comparison}.

\item 
We further refine \texttt{QHTME} to develop a quantum unbiased heavy-tailed mean estimator (\texttt{QUHTME}). 
This estimator extends the unbiased quantum mean estimator in \cite{sidford2023quantum} from the bounded-variance setting to general heavy-tailed distributions ($p\in(1,2]$). A comparison with existing classical and quantum mean estimators is provided in Table~\ref{tab:quantum_mean_estimators_comparison_heavy_tailed}.

\item  We establish quantum lower bounds for 
both high-probability and unbiased heavy-tailed mean estimation
in  Section~\ref{subsec: lower bounds on HT mean estimation}. We first obtain the baseline quantum lower bound of $\Omega((\sigma/\epsilon)^{\frac{p}{2(p-1)}})$ by reducing to an unstructured search problem. This lower bound already shows that the dependence on $\sigma$ and $\epsilon$ in our quantum algorithms \texttt{QHTME} and \texttt{QUHTME} is optimal up to logarithmic factors, when the dimension $d$ is small enough to be a constant.
Moreover, to derive a dimension-dependent quantum lower bound, we introduce a new hard primitive, namely \texttt{Recovery}$\circ$\texttt{Search}. This primitive combines a
bit-string recovery task with a rare-spike search structure, designed to
model the hard instances in heavy-tailed mean
estimation. By proving a quantum lower bound for
\texttt{Recovery}$\circ$\texttt{Search} and reducing it to the
heavy-tailed mean estimation problem, we show that, for
$p\in(4/3,2]$ and $1\leq d\leq(\sigma/\epsilon)^2$,
any quantum algorithm requires
$\Omega\left(
d^{\frac{3p-4}{4(p-1)}}
(\sigma/\epsilon)^{\frac{p}{2(p-1)}}
\right)$
queries. This result shows that a nontrivial dependence on $d$ is unavoidable.
Together with the inherited bounded-variance lower bound developed in \cite{Cornelissen_2022}, this yields
$\Omega((\sigma/\epsilon)^2)$ when
$d\geq(\sigma/\epsilon)^2$. The dimension dependence is tight when $p=2$, while determining the optimal joint
dependence on $d$, $\sigma$, and $\epsilon$ for $p<2$ remains open. The quantum lower bounds are summarized in Tables \ref{tab:mean_estimators_comparison} and \ref{tab:quantum_mean_estimators_comparison_heavy_tailed}.


\item 
For non-convex stochastic optimization, we propose a quantum normalized stochastic gradient descent method, \texttt{QNSGD}, which integrates \texttt{QHTME} with a gradient normalization technique. Under heavy-tailed noise, \texttt{QNSGD} generates iterates satisfying $T^{-1}\sum_{t=0}^{T-1}\|\nabla f(\x_t)\|\leq\epsilon$ with high probability using $\tilde{\OM}(\sqrt d\,\epsilon^{-\frac{5p-4}{2p-2}})$ queries to the quantum stochastic gradient oracle (cf. Definition~\ref{def: Quantum stochastic gradient oracle (QSGO)}). This improves upon the classical lower bound $\Omega(\epsilon^{-\frac{3p-2}{p-1}})$ whenever $d\lesssim\epsilon^{-\frac{p}{p-1}}$. A comparison with existing classical and quantum methods is provided in Table~\ref{tab:non-convex_heavy_tailed}.

\item 
For convex stochastic optimization, we propose a quantum projected stochastic gradient descent method, \texttt{QPSGD}, which leverages the unbiased quantum estimator \texttt{QUHTME}. Under heavy-tailed noise, \texttt{QPSGD} obtains a solution with expected suboptimality at most $\epsilon$ using $\tilde{\OM}(\sqrt d\,\epsilon^{-\frac{3p-2}{2(p-1)}}+\epsilon^{-2})$ expected queries to the quantum stochastic gradient oracle. This improves upon the classical lower bound $\Omega(\epsilon^{-\frac{p}{p-1}})$ for $p\in(1,2)$ whenever $d\lesssim\epsilon^{-\frac{2-p}{p-1}}$. A comparison with existing classical methods is given in Table~\ref{tab:convex_heavy_tailed}.
\end{itemize}
While our results are primarily theoretical, we briefly discuss their practicality and potential applications in Appendix~\ref{app:practicality}.

\begin{table}[h]
\centering
\caption{Comparison of estimators for producing an estimate $\tilde{\bmu}$ of the mean $\bmu:=\EBP{\X}$ of a $d$-dimensional random vector $\X$ satisfying $\EBP{\|\X-\bmu\|^{p}}\leq \sigma^{p}$ such that $\|\tilde{\bmu}-\bmu\|\leq \epsilon$.
}
\label{tab:mean_estimators_comparison}
\begin{threeparttable}
\small
\setlength{\tabcolsep}{5.5pt}
\renewcommand{\arraystretch}{1.22}
\begin{tabular*}{\textwidth}{@{\extracolsep{\fill}} l c c c c c @{}}
\toprule
& \textbf{References} & \textbf{Assumption on $p$} & \textbf{Dimension $d$} & \textbf{Queries} & \textbf{Bias} \\
\midrule
\multirow{3}{*}{\textsc{Classical}}&
\cite{lugosi2019meanestimationregressionheavytailed}
&$p=2$
&$d\geq 1$
&$\OM((\sigma/\epsilon)^{2})$ 
& $0$    \\

&\textbf{Theorem \ref{thm :multivariate-mom-error-bound}}
&$p\in(1,2]$
&$d\geq 1$
& $\tilde{\OM}((\sigma/\epsilon)^{\frac{p}{p-1}})$
& $\epsilon$   \\\cmidrule(lr){2-6}

&\cite{lugosi2019meanestimationregressionheavytailed}
&$p\in(1,2]$
&$d\geq 1$
& $\Omega((\sigma/\epsilon)^{\frac{p}{p-1}})$
& $-$ \\

\midrule
\multirow{8}{*}{\textsc{Quantum}}&\cite{tang2025moreefficientquantummultivariatemean}
&$p=2$
&$d\geq 1$
&$\tilde{\OM}(d^{\frac{1}{2}}\sigma/\epsilon)$ 
& $\epsilon$ \\

&\cite{blanchet2024quadraticspeedupinfinitevariance}
&$p\in(1,2]$
&$d= 1$
& $\tilde{\OM}((\sigma/\epsilon)^{\frac{p}{2(p-1)}})$
& $\epsilon$   \\

&

\textbf{ Theorem~\ref{thm:QHTME-for prob 1}}
&$p\in(1,2]$
&$d\geq 1$
&$\tilde{\OM}(d^{\frac{1}{2}}(\sigma/\epsilon)^{\frac{p}{2(p-1)}})$
& $\epsilon$
\\\cmidrule(lr){2-6}

&\cite{Cornelissen_2022}
&$p=2$
&\begin{tabular}{@{}c@{}}
$1\leq d\leq (\sigma/\epsilon)^{2}$\\[2pt]
$d\geq (\sigma/\epsilon)^{2}$
\end{tabular}
&\begin{tabular}{@{}c@{}}
$\Omega(d^{\frac{1}{2}}\sigma/\epsilon)$\\[2pt]
$\Omega((\sigma/\epsilon)^{2})$
\end{tabular}
& $-$ \\

&\textbf{Theorem \ref{thm:lower-bounds for high probability}}
&\begin{tabular}{@{}c@{}}
$p\in(4/3,2]$\\[2pt]
$p\in(4/3,2]$\\[2pt]
$p\in(1,4/3]$
\end{tabular}
&\begin{tabular}{@{}c@{}}
$1\leq d\leq (\sigma/\epsilon)^{2}$\\[2pt]
$d\geq (\sigma/\epsilon)^{2}$\\[2pt]
$d\geq 1$
\end{tabular}
&\begin{tabular}{@{}c@{}}
$\Omega(d^{\frac{3p-4}{4(p-1)}}(\sigma/\epsilon)^{\frac{p}{2(p-1)}})$\\[2pt]
$\Omega((\sigma/\epsilon)^{2})$\\[2pt]
$\Omega((\sigma/\epsilon)^{\frac{p}{2(p-1)}})$
\end{tabular}
& $-$\\


\bottomrule
\end{tabular*}
\end{threeparttable}
\end{table}

\begin{table*}[h]
\centering
\caption{
Comparison of  estimators for producing an estimate $\tilde{\bmu}$ of the mean $\bmu:=\EBP{\X}$ of a $d$-dimensional random vector $\X$ satisfying $\EBP{\|\X-\bmu\|^{p}}\leq \sigma^{p}$ such that $\EBP{\|\tilde{\bmu}-\bmu\|^{p}}\leq \epsilon^{p}$.
}
\label{tab:quantum_mean_estimators_comparison_heavy_tailed}
\small
\setlength{\tabcolsep}{5.5pt}
\renewcommand{\arraystretch}{1.22}
\begin{tabular*}{\textwidth}{@{\extracolsep{\fill}} l c c c c c @{}}
\toprule
& \textbf{References} & \textbf{Assumption on $p$} & \textbf{Dimension $d$} & \textbf{Queries} & \textbf{Bias} \\
\midrule
\multirow{2}{*}{\textsc{Classical}}&
Empirical Mean
&$p\in(1,2]$
&$d\geq 1$
&$\tilde{\mathcal{O}}(\sigma/\epsilon)^{\frac{p}{p-1}}$ 
& $0$    \\
& \cite{lugosi2019meanestimationregressionheavytailed}
&$p\in(1,2]$
&$d\geq 1$
&$\Omega((\sigma/\epsilon)^{\frac{p}{p-1}})$ 
& $-$    \\

\midrule
\multirow{7}{*}{\textsc{Quantum}}
&\cite{Cornelissen_2023}
&$p=2$
&$d=1$
&$\tilde{\OM}(\sigma/\epsilon)$ 
& $\delta\sigma$  \\

&\cite{sidford2023quantum}
&$p=2$
&$d\geq 1$
&$\tilde{\OM}(d^{\frac{1}{2}}\sigma/\epsilon)$ 
& $0$   \\

& \textbf{Theorem~\ref{theorem: query complexity of QUHTME}}
&$p\in(1,2]$
&$d\geq 1$
& $\tilde{\OM}(d^{\frac{1}{2}}(\sigma/\epsilon)^{\frac{p}{2(p-1)}})$
& $0$ \\

\cmidrule(lr){2-6}

&\cite{sidford2023quantum}
&$p=2$
&$1\leq d\leq (\sigma/\epsilon)^{2}$
&$\Omega(d^{\frac{1}{2}}\sigma/\epsilon)$ 
& $-$   \\
&\textbf{Theorem \ref{thm:lower-bounds for unbiased}}
&\begin{tabular}{@{}c@{}}
$p\in(4/3,2]$\\[2pt]
$p\in(4/3,2]$\\[2pt]
$p\in(1,4/3]$
\end{tabular}
&\begin{tabular}{@{}c@{}}
$1\leq d\leq (\sigma/\epsilon)^{2}$\\[2pt]
$d\geq (\sigma/\epsilon)^{2}$\\[2pt]
$d\geq 1$
\end{tabular}
&\begin{tabular}{@{}c@{}}
$\Omega(d^{\frac{3p-4}{4(p-1)}}(\sigma/\epsilon)^{\frac{p}{2(p-1)}})$\\[2pt]
$\Omega((\sigma/\epsilon)^{2})$\\[2pt]
$\Omega((\sigma/\epsilon)^{\frac{p}{2(p-1)}})$
\end{tabular}
& $-$\\
\bottomrule
\end{tabular*}
\end{table*}


\begin{table*}[t]
\centering
\caption{Comparison of different methods to find the $\epsilon$-stationary point, i.e. $\|\nabla f(\x)\|\leq \epsilon$ in stochastic nonconvex optimization. $p$ is the index of the heavy-tailed assumption on the stochastic gradient in Assumption~\ref{ass:heavy-tailed}. 
}
\label{tab:non-convex_heavy_tailed}
\small
\setlength{\tabcolsep}{7pt}
\renewcommand{\arraystretch}{1.22}
\begin{tabular*}{0.92\textwidth}{@{\extracolsep{\fill}} l c c c @{}}
\toprule
& \textbf{References} & \textbf{Assumption on $p$} & \textbf{Oracle Complexity} \\
\midrule
\multirow{3}{*}{\textsc{Classical}}&
\cite{ghadimi2013stochastic}
& $p=2$
&$\mathcal{O}(\epsilon^{-4})$ 
   \\
&\cite{zhang2020adaptive,hubler2024gradient,liu2025online}
& $p\in(1,2]$
&$\mathcal{O}(\epsilon^{-\frac{3p-2}{p-1}})$ 
\\\cmidrule(lr){2-4}
&
Lower Bound~\cite{zhang2020adaptive}
& $p\in(1,2]$
& $\Omega\big(\epsilon^{-\frac{3p-2}{p-1}}\big)$
\\

\midrule
\multirow{2}{*}{\textsc{Quantum}}
&\cite{sidford2023quantum}
&$p=2$
& $\tilde{\OM}(d^{\frac{1}{2}}\epsilon^{-3})$
\\


&\textbf{Theorem~\ref{thm:non-convex_heavy_tailed}}
&$p\in(1,2]$
&$\tilde{\OM}(\sqrt d\,\epsilon^{-\frac{5p-4}{2p-2}})$

  \\


\bottomrule
\end{tabular*}
\end{table*}

\begin{table*}[t]
\centering
\caption{Comparison of different methods to find the $\epsilon$ optimal solution, i.e. $f(\x)-f^*\leq \epsilon$ in stochastic convex optimization. $p$ is the index in Assumption~\ref{ass:heavy-tailed}. 
}
\label{tab:convex_heavy_tailed}
\small
\setlength{\tabcolsep}{7pt}
\renewcommand{\arraystretch}{1.22}
\begin{tabular*}{0.92\textwidth}{@{\extracolsep{\fill}} l c c c @{}}
\toprule
& \textbf{References} & \textbf{Assumption on $p$} & \textbf{Oracle Complexity} \\
\midrule
\multirow{3}{*}{\textsc{Classical}}&
\cite{ghadimi2013stochastic}
& $p=2$
&$\mathcal{O}(\epsilon^{-2})$ 
 \\

&\cite{fatkhullin2025can,liu2025online,zhang2020adaptive}
& $p\in(1,2]$
& $\mathcal{O}(\epsilon^{-\frac{p}{p-1}})$

\\
\cmidrule(lr){2-4}
&
     Lower Bound~\cite{raginsky2009information}

& $p\in(1,2]$
& $\Omega(\epsilon^{-\frac{p}{p-1}})$
\\
\midrule

     \textsc{Quantum}


&\textbf{Theorem~\ref{thm:convex_heavy_tailed}}
&$p\in(1,2]$
&$\tilde{\OM}(\sqrt d\,\epsilon^{-\frac{3p-2}{2(p-1)}}+\epsilon^{-2})$

  \\


\bottomrule
\end{tabular*}
\end{table*}

\section{Preliminaries}\label{sec: preliminaries}
We use $\|\cdot\|$ to denote the spectral norm for matrices and the Euclidean norm for vectors, respectively.
We denote by $\operatorname{Geom}(\rho)$ the geometric distribution with parameter $\rho$, and by $\Pi_{\fX}(\cdot)$ the projection operator onto the set $\fX$. 
$\Tilde{\OM}(\cdot)$ denotes the big-$\OM$ notation that suppresses poly-logarithmic factors in $d,\delta,n$ and $\epsilon$. All logarithms are taken in base $2$.

The quantum access models are expressed using the Dirac notation of quantum computing \cite{nielsen2010quantum,wilde2013quantum,dewolf2023quantumcomputinglecturenotes}.  
We use $\{\ket{i}\}_{i=0}^{n-1}$ to denote the computational basis of
$\mathbb C^n$, where $\ket{0}$ denotes the first standard basis vector. A normalized
vector $\mathbf v\in\mathbb C^n$ can be written as a quantum state
$\ket{\mathbf v}=\sum_{i=0}^{n-1}v_i\ket{i}$.

Throughout this paper, real numbers are encoded in the computational basis using a fixed-point binary representation. Specifically, a real number $x$ is encoded as  a basis state $\ket{x}=\ket{x_{1}\ldots x_{q}}\in\CB^{2^{q}}$, where the binary string $x_{1}\ldots x_{q}$ corresponds to a fixed-point encoding of $x$. Similarly, for
$\x\in\mathbb R^d$, the notation $\ket{\x}$ denotes the computational-basis encoding of
the coordinates of $\x$, rather than an amplitude encoding of the vector $\x$. We assume~access to sufficiently many qubits to avoid overflow for all real numbers appearing in the algorithms.


With these conventions, we first formalize the classical and quantum access models for a random variable $\X\in \BR^d$. 

\begin{definition}[Classical sampling oracle]\label{def: Classical sampling oracle}
Given a random variable $\X: \Omega\rightarrow E\subseteq \mathbb{R}^{d}$ on a probability space $(\Omega, 2^{\Omega}, \mathbb{P})$, we define a classical sampling oracle as the process of drawing a sample $\omega \in\Omega$ according to $\BP$ and observing the value of $\X(\omega)\in \BR^d$.
\end{definition}

\begin{definition}[Quantum sampling oracle]\label{def:quantum-sampling-oracle}
Let $\X$ be a $d$-dimensional finite-valued random variable taking values in $E\subset \mathbb{R}^{d}$, and let
$p_{\X}(\x):=\Pr_{\omega\sim\mathbb{P}}[\X(\omega)=\x]$ for each $\x\in E$. A quantum sampling oracle for $\X$ is a
unitary operator $O_{\X}$ satisfying
\begin{equation}\label{eq:quantum-sampling-oracle}
     O_{\X}:\ket{0}\ket{0}
    \rightarrow
    \sum_{\x\in E}\sqrt{p_{\X}(\x)}\,\ket{\x}\ket{\psi_\x},
\end{equation}
where $\ket{0}$ denotes fixed reference states, and each
$\ket{\psi_\x}$ is a normalized auxiliary state.
\end{definition}

Note that measuring the first register of $O_{\X}\ket{0}\ket{0}$ in the computational basis yields a classical sample distributed as $\X$, so the classical sampling oracle can be simulated by $O_{\X}$. Although Definition~\ref{def:quantum-sampling-oracle} is stated for finite-valued random variables, it should be viewed as a finite-precision representation. For heavy-tailed distributions, this corresponds to a suitably truncated and discretized approximation. We also note that prior work \cite{Cornelissen_2022,sidford2023quantum} adopts different notions of quantum sampling oracle. A comparison and a discussion of compatibility with heavy-tailed assumptions are deferred to Appendix~\ref{appendix: quantum access model for random variables}.

We now define the classical and quantum access models to the stochastic gradient $\nabla F(\x;\xi)$ satisfying Assumption \ref{ass:heavy-tailed}.

\begin{definition}[Classical stochastic gradient oracle]\label{def: Classical stochastic gradient oracle CSGO}
We define a classical stochastic gradient oracle as the process of implementing a random function that, when queried at $\x$, samples a random seed $\xi$ from a probability distribution $\Xi$ and returns $\nabla F(\x; \xi)$ satisfying Assumption \ref{ass:heavy-tailed}.
\end{definition}

\begin{definition}[Quantum stochastic gradient oracle]
\label{def: Quantum stochastic gradient oracle (QSGO)}
Let $\nabla F(\x;\xi)$ be the stochastic gradient at a query point $\x$,
and let $p_{\x}(\v):=\Pr_{\xi\sim\Xi}[\nabla F(\x;\xi)=\v],$ where $\v\in E$ and $E\subset\mathbb R^d$ is the finite range of $\nabla F(\x;\xi)$.
A quantum stochastic gradient oracle at $\x$ is a unitary operator
$O_{\nabla F,\x}$ satisfying
\[
    O_{\nabla F,\x}: \ket{0}\ket{0}
    \rightarrow
    \sum_{v\in E}\sqrt{p_{\x}(\v)}\,\ket{\v}\ket{\psi_{\x,\v}},
\]
where $\ket{0}$ denotes fixed reference states, and each
$\ket{\psi_{\x,\v}}$ is a normalized auxiliary state.
\end{definition}

We note that \citet{sidford2023quantum} also defined a quantum stochastic gradient oracle, but their model requires coherent access to the query point $\x$. 
A detailed comparison is deferred to Appendix~\ref{appendix: quantum stochastic gradient oracle}.

Below, we now present the optimal classical heavy-tailed mean estimator, which will be used as a subroutine in our subsequent algorithm. Related work on classical heavy-tailed mean estimation, together with the proof of Theorem \ref{thm :multivariate-mom-error-bound}, is provided in Appendix \ref{subsec: CHTME}.
\begin{theorem}[Classical heavy-tailed multivariate mean estimation]\label{thm :multivariate-mom-error-bound} 
    Let $\X$ be a $d$-dimensional random variable with mean $\bmu:=\EBP{\X}$ and finite $p$-th central moment $\EBP{\|\X-\bmu\|^p}$~for some $p\in(1,2]$.
    Let $\delta \in (0,1)$ and $k=\lceil24\ln(1/\delta)\rceil$. For integer $n$ such that $n\geq 2k$,  the classical heavy-tailed mean estimator {\rm\texttt{CHTME}}$(\X, n,\delta)$ outputs a mean estimate $\tilde{\bmu}$ such that 
\begin{equation*}
    \Norm{\tilde{\bmu}-\bmu}\leq 3(8C_{p})^{1/p} \left(\mathbb{E}[\Norm{\X-\bmu}^p]\right)^{1/p}  \left(\frac{48\ln(1/\delta)}{n}\right)^{\frac{p-1}{p}},
\end{equation*}
by using $n$ queries to the classical sampling oracle, where $C_{p}$ is a $p$-dependent constant.
\end{theorem}


\section{Quantum Mean Estimation for  Heavy-Tailed Random Variables}\label{sec: QME for HTRV}

In this section, we develop novel quantum mean estimators for  $d$-dimensional heavy-tailed random variables.
We begin by formalizing the estimation tasks of interest.
\begin{problem}\label{prob: obtain mean estimate for heavy tail random variable with bounded mean}
    Let  $\X$ be a $d$-dimensional random variable with mean $\bmu:= \EBP{\X}$ satisfying $\mathbb{E}[\|\X-\bmu\|^{p}]\leq \sigma^{p}$ for some $p\in (1,2]$, and constant $\sigma>0$.
For any $\epsilon\in (0,\sigma)$, the goal is to output an  estimate $\tilde{\bmu}$ of $\bmu$ satisfying $\|\tilde{\bmu} - \bmu\| \leq \epsilon.$
\end{problem}
We also consider the construction of unbiased estimators whose estimation error has a controlled $p$-th raw moment.
\begin{problem}\label{prob: obtain unbiased mean estimate for heavy tail random variable}
  Let $\X$ be a $d$-dimensional random variable  with mean $\bmu:= \EBP{\X}$ satisfying $\mathbb{E}[\|\X-\bmu\|^{p}]\leq \sigma^{p}$ for some $p\in (1,2]$, and constant $\sigma>0$.
For any $\epsilon\in (0,\sigma)$, the goal is to output an \textbf{unbiased} estimate $\tilde{\bmu}$ of $\bmu$ satisfying  $\EBP{\tilde{\bmu}}=\bmu$ and $\EBP{\|\tilde{\bmu} - \bmu\|^{p}} \leq \epsilon^{p}.$
\end{problem}

\subsection{Quantum Heavy-Tailed Mean Estimator}\label{subsec: Quantum Heavy-Tailed Mean Estimator}
 Here we propose a quantum mean estimator to solve Problem~\ref{prob: obtain mean estimate for heavy tail random variable with bounded mean}.
We first introduce an existing quantum mean estimator for bounded-variance random variables.

\begin{lemma}[Theorem 3.4, \cite{Cornelissen_2022}]\label{thm:qbvme}
Let $\X$ be a $d$-dimensional random variable with mean $\bmu$ and covariance matrix $\Sigma$.
Given two reals $\delta\in(0,1)$ and $n\ge \log(d/\delta)$, the quantum multivariate estimator
{\rm\texttt{QEstimator}}$_d(\X,n,\delta)$ outputs $\widetilde{\bmu}$ such that
\[
    \|\widetilde{\bmu}-\bmu\|_\infty
    \le
    \frac{\sqrt{\operatorname{Tr}(\Sigma)\log(d/\delta)}}{n}
\]
with probability at least $1-\delta$.
It uses $\widetilde O(n)$ queries to the quantum sampling oracle $O_X$.
\end{lemma}
To estimate $\bmu:=\EBP{\X}$ under a heavy-tailed distribution,  \texttt{QEstimator} cannot be applied directly, as the heavy-tailed random variable $\X$ does not necessarily have bounded variance. 

To overcome this difficulty, our proposed estimator, \texttt{QHTME}, follows
a center--truncate--estimate procedure. In the centering step, we first use the
classical heavy-tailed mean estimator \texttt{CHTME} to obtain a rough estimate
$\bzeta$ of $\bmu$, and then center the random variable as $\Y:=\X-\bzeta$.
This step is necessary because the heavy-tailed assumption controls the moment
of $\X$ around the unknown mean $\bmu$, whereas directly truncating $\X$ around
the origin may introduce a dependence on $\bmu$. 

In the truncation step, we remove samples of $\Y$ whose norms exceed a suitably
chosen threshold $B$. The resulting truncated random variable $\Z_B$ has bounded
trace covariance and can therefore be handled by the bounded-variance quantum
mean estimator in Lemma~\ref{thm:qbvme}. The truncation threshold $B$ is chosen to
control the bias introduced by discarding extreme samples.

Finally, in the estimation step, we apply
$\texttt{QEstimator}$ to estimate the mean of $\Z_B$ and add the rough center
$\bzeta$ back to the estimate. The parameters are selected so that the
truncation bias and the quantum mean-estimation error together are at most
$\epsilon$. We formalize the resulting procedure in Algorithm \ref{alg: Quantum Heavy-Tailed Mean Estimator (QHTME)} and state its performance guarantees in Theorem~\ref{thm:QHTME-for prob 1}, whose proof is deferred to Appendix \ref{appendix: Proof of QHTME for prob 1}.

\begin{algorithm}[t]
\caption{Quantum Heavy-Tailed Mean Estimator
($\texttt{QHTME}(\X,\sigma,p,d,\epsilon,\delta)$)}
\label{alg: Quantum Heavy-Tailed Mean Estimator (QHTME)}
\begin{algorithmic}[1]
\STATE \textbf{Input:} Random variable $\X$, heavy-tailed parameter $\sigma>0$, tail index $p\in(1,2]$, dimension $d$, target error $\epsilon\in(0,\sigma]$, failure probability $\delta\in(0,1)$.

\STATE Set
\[
n_{\mathrm{c}}
\leftarrow
\left\lceil
2\left\lceil 24\ln\left(\frac{3}{\delta}\right)\right\rceil
3^{\frac{p}{p-1}}
(8C_p)^{\frac1{p-1}}
\right\rceil.
\]

\STATE Obtain a rough estimate by the classical heavy-tailed mean estimator of Theorem \ref{thm :multivariate-mom-error-bound}:
\[
\bzeta \leftarrow \texttt{CHTME}\left(\X,n_{\mathrm{c}},\frac{\delta}{3}\right).
\]

\STATE Define the centered random variable $\Y:=\X-\bzeta .$

\STATE Set the truncation radius
\[
B
\leftarrow
\left(
\frac{3\cdot 2^p\sigma^p}{\epsilon}
\right)^{\frac1{p-1}} .
\]

\STATE Define the truncated random variable $\Z_B
:=
\Y\cdot \mathbf 1{\{\|\Y\|\le B\}} .$

\STATE Set an upper bound on the trace covariance of $\Z_B$: $L
\leftarrow
2^p\sigma^p B^{2-p}.$

\STATE Set 
\[
n_{q}\gets \left\lceil\max \left\{\log(3d/\delta),\frac{3\sqrt{dL\log(3d/\delta)}}{2\epsilon}\right\}\right\rceil
\]

\STATE Estimate the mean of $\Z_B$ by a bounded-variance  quantum mean estimator in Lemma \ref{thm:qbvme}:
\[
\widetilde \bmu_B
\leftarrow
\texttt{QEstimator}_d
\left(
\Z_B,n_{q}, \frac{\delta}{3}
\right).
\]

\STATE \textbf{Output:} $\widetilde\bmu
\leftarrow
\bzeta+\widetilde \bmu_B $.
\end{algorithmic}
\end{algorithm}

\begin{theorem}\label{thm:QHTME-for prob 1}
Algorithm \ref{alg: Quantum Heavy-Tailed Mean Estimator (QHTME)} ({\rm$\texttt{QHTME}$}) solves Problem~\ref{prob: obtain mean estimate for heavy tail random variable with bounded mean} using $
\tilde{\mathcal{O}}\!\left(
d^{\frac{1}{2}}\,\sigma^{\frac{p}{2(p-1)}}\epsilon^{-\frac{p}{2(p-1)}}
\right)$ queries to the quantum sampling oracle $O_{\X}$ with probability at least $1-\delta$. 
\end{theorem}

\vspace{3pt}
\begin{remark}\label{remark: compare QHTME with Wu Yulian}
    ${\rm\texttt{QHTME}}$ applies to $d-$dimensional random vectors ($d\geq 1$) with bounded $p$-th central moment $(p\in (1,2])$, which generalizes the univariate setting considered in \cite[Theorem 5.2]{blanchet2024quadraticspeedupinfinitevariance} and the bounded-variance setting considered in \cite[Theorem 3.4]{Cornelissen_2022}. 
\end{remark}

\vspace{3pt}
\begin{remark}\label{remark: compare QHTME with classical HTME}
    \(\texttt{QHTME}\) has a provable advantage over optimal classical mean estimators in terms of \emph{query complexity}. 
    By Corollary \ref{corollary:CHTME-for prob 1}, optimal classical estimators \texttt{CHTME} require $\tilde{\mathcal{O}}((\sigma/\epsilon)^{\frac{p}{p-1}})$ queries to the classical sampling oracle to solve Problem~\ref{prob: obtain mean estimate for heavy tail random variable with bounded mean}.
    In contrast, \texttt{QHTME} requires only $\tilde{\OM}(\sqrt{d}(\sigma/\epsilon)^{\frac{p}{2(p-1)}})$ queries to the quantum sampling oracle.
   As a result, \texttt{QHTME} achieves a quantum speedup over the optimal classical estimators in the low-dimensional regime, specifically when $d \le \OM((\sigma/\epsilon)^{\frac{p}{p-1}})$.
   
\end{remark}



\subsection{Quantum Unbiased Heavy-Tailed Mean Estimator}\label{subsec: Quantum Unbiased Heavy-Tailed Mean Estimator}

The quantum heavy-tailed mean estimator \texttt{QHTME} (Algorithm~\ref{alg: Quantum Heavy-Tailed Mean Estimator (QHTME)}) is, in general, biased. This limitation prevents its direct use in many stochastic optimization frameworks~\cite{fang2018spider,cutkosky2023optimal,sidford2023quantum,liu2025online}, and motivates our study of Problem~\ref{prob: obtain unbiased mean estimate for heavy tail random variable}. To address this problem, we construct an unbiased quantum estimator by extending the multilevel Monte Carlo (MLMC) variance-reduction technique~\cite{giles2015multilevel,sidford2023quantum} to the heavy-tailed setting with bounded $p$-th central moment for $p\in(1,2]$, and combining this generalized MLMC framework with our estimator \texttt{QHTME}. Our approach proceeds in two steps.

First, we design a modified quantum heavy-tailed mean estimator, denoted by $\texttt{QHTME}^{+}(\X,\sigma,p,d,\epsilon)$, whose estimation error admits a controlled $p$-th moment deterministically, while retaining the same query complexity as \texttt{QHTME}. The construction combines a single quantum estimate produced by \texttt{QHTME} with a classical estimate, and generates an additional independent classical estimate when the two differ significantly. Details of $\texttt{QHTME}^{+}$ are given in Appendix~\ref{appendix: QHTME-plus}.

Second, we embed $\texttt{QHTME}^{+}$ into a generalized MLMC framework to obtain an unbiased quantum estimator, denoted by $\texttt{QUHTME}(\X,\sigma,p,d,\epsilon)$ and presented in Algorithm~\ref{alg:QUHTME}. The procedure invokes $\texttt{QHTME}^{+}$ at three randomly selected $p$-dependent accuracy levels and combines the resulting estimates to produce an unbiased output. Although the accuracy parameters range over an infinite set, we show that the quantum estimator \texttt{QUHTME} is able to solve Problem \ref{prob: obtain unbiased mean estimate for heavy tail random variable} while maintaining the same expected query complexity as $\texttt{QHTME}^{+}$. The proof is deferred to Appendix \ref{appendix: QUHTME}.

\begin{algorithm}[htbp]
\caption{Quantum Unbiased Heavy-Tailed Mean Estimator ($\texttt{QUHTME}(\X, \sigma, p, d,\epsilon)$)}
\label{alg:QUHTME}
\begin{algorithmic}[1]
\STATE \textbf{Input}:  random variable $\X$, heavy-tailed parameter $\sigma$, tail index $p$, problem dimension $d$, error $\epsilon$.

\STATE Set $c\gets 24/(2^{\frac{p-1}{2}}-1)$ and $\tilde{\bmu}_0 \gets \texttt{QHTME}^{+}\left(X,\sigma, p, d, \frac{1}{c^{1/p}}\epsilon\right)$
\STATE Randomly sample $j \sim \mathrm{Geom}(1/2) \in \mathbb{N}$
\STATE $\tilde{\bmu}_j \gets \texttt{QHTME}^{+}\!\left(\X, \sigma, p, d, \frac{2^{-3(p-1)j/(2p)}}{c^{1/p}}\epsilon\right)$
\STATE $\tilde{\bmu}_{j-1} \gets \texttt{QHTME}^{+}\!\left(\X, \sigma, p,d,
\frac{2^{-3(p-1)(j-1)/(2p)}}{c^{1/p}}\epsilon\right)$
\STATE $\tilde{\bmu} \gets \tilde{\bmu}_0 + 2^{j}\bigl(\tilde{\bmu}_j - \tilde{\bmu}_{j-1}\bigr)$
\STATE \textbf{Output:} $\tilde{\bmu}$
\end{algorithmic}
\end{algorithm}

\begin{theorem}\label{theorem: query complexity of QUHTME}
Algorithm \ref{alg:QUHTME} ({\rm \texttt{QUHTME}}) solves Problem \ref{prob: obtain unbiased mean estimate for heavy tail random variable} using
   $\tilde{\OM}\left(d^{\frac{1}{2}}(\sigma/\epsilon)^{\frac{p}{2(p-1)}}\right)$ queries to the quantum sampling oracle $O_{\X}$ in expectation.
\end{theorem}

\begin{remark}
\texttt{QHTME}$^{+}$ and \texttt{QUHTME} generalize the quantum mean estimators of \cite[Algorithms 1 and 3]{sidford2023quantum} from the bounded-variance regime ($p=2$) to the heavy-tailed regime ($p\in(1,2]$).
\end{remark}

\subsection{Quantum Lower Bounds on Heavy-Tailed Mean Estimation Problems}
\label{subsec: lower bounds on HT mean estimation}

Here we establish quantum lower bounds for solving Problems \ref{prob: obtain mean estimate for heavy tail random variable with bounded mean} and \ref{prob: obtain unbiased mean estimate for heavy tail random variable}. Proofs are deferred to Appendix \ref{appendix: proof of quantum lower bounds for high probability} and Appendix \ref{appendix: proof of quantum lower bounds for unbiased}.

\begin{theorem}[Quantum lower bounds for Problem \ref{prob: obtain mean estimate for heavy tail random variable with bounded mean}]
\label{thm:lower-bounds for high probability}
Any quantum algorithm that solves
    Problem~\ref{prob: obtain mean estimate for heavy tail random variable with bounded mean}
    with success probability at least $2/3$ must, in the worst case, make $N$ queries to the quantum sampling oracle, where 

\[ N\in 
\begin{cases}
    \Omega\left(
            \left(
                \frac{\sigma}{\epsilon}
            \right)^{\frac{p}{2(p-1)}}
        \right), & \text{when }p\in(1, 4/3], \\
    \Omega\left(
            d^{\frac{3p-4}{4(p-1)}}
            \left(
                \frac{\sigma}{\epsilon}
            \right)^{\frac{p}{2(p-1)}}
        \right), & \text{when }p\in (4/3,2] \text{ and } 1\leq d\leq
        \left(
            \frac{\sigma}{\epsilon}
        \right)^2,\\
    \Omega\left(
            \left(
                \frac{\sigma}{\epsilon}
            \right)^2
        \right), &\text{when }p\in (4/3,2] \text{ and }d\geq
        \left(
            \frac{\sigma}{\epsilon}
        \right)^2.
\end{cases}
\]



\end{theorem}

\begin{theorem}[Quantum lower bounds for Problem \ref{prob: obtain unbiased mean estimate for heavy tail random variable}]
\label{thm:lower-bounds for unbiased}
Any quantum algorithm that solves
    Problem~\ref{prob: obtain unbiased mean estimate for heavy tail random variable}
    must, in the worst case, make a number of $N$ queries to the quantum sampling oracle, where
    \[ \mathbb{E}[N]\in 
\begin{cases}
    \Omega\left(
            \left(
                \frac{\sigma}{\epsilon}
            \right)^{\frac{p}{2(p-1)}}
        \right), & \text{when }p\in(1, 4/3], \\
    \Omega\left(
            d^{\frac{3p-4}{4(p-1)}}
            \left(
                \frac{\sigma}{\epsilon}
            \right)^{\frac{p}{2(p-1)}}
        \right), & \text{when }p\in (4/3,2] \text{ and } 1\leq d\leq
        \left(
            \frac{\sigma}{\epsilon}
        \right)^2,\\
    \Omega\left(
            \left(
                \frac{\sigma}{\epsilon}
            \right)^2
        \right), &\text{when }p\in (4/3,2] \text{ and }d\geq
        \left(
            \frac{\sigma}{\epsilon}
        \right)^2.
\end{cases}
\]



\end{theorem}

\begin{remark}
Theorems~\ref{thm:lower-bounds for high probability}
and~\ref{thm:lower-bounds for unbiased} imply that, for every $d\geq1$ and $p\in(1,2]$, solving
Problems~\ref{prob: obtain mean estimate for heavy tail random variable with bounded mean}
and~\ref{prob: obtain unbiased mean estimate for heavy tail random variable}
requires
\[
    \Omega\left(
        \left(
            \frac{\sigma}{\epsilon}
        \right)^{\frac{p}{2(p-1)}}
    \right)
\]
queries to the quantum sampling oracle.
This matches the dependence of \texttt{QHTME} and \texttt{QUHTME} on
$\sigma$ and $\epsilon$, up to logarithmic factors, when the dimension
$d$ is small enough to be treated as a constant. Thus, our algorithms are
optimal in their dependence on the heavy-tailed parameter $\sigma$ and the target
accuracy $\epsilon$. The stronger quantum lower bounds, e.g., 
$$
 \Omega\left(
            d^{\frac{3p-4}{4(p-1)}}
            \left(
                \frac{\sigma}{\epsilon}
            \right)^{\frac{p}{2(p-1)}}
        \right) \text{ and } \Omega\left(
            \left(
                \frac{\sigma}{\epsilon}
            \right)^2
        \right),
$$
additionally characterize how the
difficulty of the problem changes with the dimension and the heavy-tailed index $p$. However, these lower bounds do not yet match the full joint dependence of \texttt{QHTME} and \texttt{QUHTME} on $d, \sigma$ and $\epsilon$. Closing this gap remains an open problem.
\end{remark}

\section{Quantum Stochastic Nonconvex Optimization with Heavy-Tailed Noise}
\label{sec:non-convex}
In this section, we present quantum algorithms for finding stationary points of the objective~\eqref{eq:obj}, where $\fX = \RB^{d}$ and $f(\cdot)$ is non-convex.
We begin by defining the notion of an $\epsilon$-stationary point.

\begin{definition}
A point $\hat{\x}\in \RB^{d}$ is called an $\epsilon$-stationary point of $f$ if $\mathbb{E}[\|\nabla f(\hat{\x})]\|\leq \epsilon$.
\end{definition}

We impose the following standard regularity assumptions on the objective function.

\begin{assumption}
\label{ass:smooth}
 We assume that the gradient of $f(\cdot)$ is Lipschitz continuous, such that for all $\vx,\vy\in\RB^d$, it holds that
 $\|\nabla f(\x)-\nabla f(\y)\|\leq L\|\x-\y\|$.
    We also assume that $f(\cdot)$ is lower bounded, and write $f^*:=\inf_{\x\in\RB^d}f(\x)>-\infty$.
\end{assumption}

Unlike most existing approaches that use gradient clipping to handle heavy-tailed noise~\cite{zhang2020adaptive,sadiev2023high}, we consider the following normalized stochastic gradient descent (\texttt{NSGD}) update:
\begin{align}
\label{eq:NSGD}
    \x_{t+1}=\x_t -\eta {\g_t}/{\|\g_t\|},
\end{align}
where $\g_t$ is an estimator of $\nabla f(\x_t)$.
The following lemma establishes the descent property of the \texttt{NSGD} update.

\begin{lemma}[Lemma 2,~\cite{cutkosky2020momentum}]
\label{lm:descent}
 Let ${\bepsilon_t} := \g_t - \nabla f(\x_t)$. If $f(\cdot)$ satisfies Assumption~\ref{ass:smooth}, then the {\rm \texttt{NSGD}} update~\eqref{eq:NSGD} satisfies
 $$  f(\x_{t+1})-f(\x_t)\leq -\frac{\eta}{3}\|\nabla f(\x_t)\|+\frac{8\eta}{3}\|\bepsilon_t\| + \frac{L\eta^2}{2}.$$
\end{lemma}

Lemma~\ref{lm:descent} shows that the descent behavior of \texttt{NSGD} depends explicitly on the magnitude of the gradient estimation error $\|\bepsilon_{t}\|$. By applying \texttt{QHTME} to estimate the stochastic gradient via the quantum stochastic gradient oracle, we obtain $\|\bepsilon_t\|\leq\epsilon_{\rm g}$ with high probability. Lastly, with an appropriate choice of the step size~$\eta$, the objective value $f(\cdot)$ decreases sufficiently along the~\texttt{NSGD} update. Motivated by this observation, we propose a quantum normalized stochastic gradient descent method (\texttt{QNSGD}), presented in Algorithm~\ref{alg: QNSGD}, and state its query complexity in the following theorem.

\begin{algorithm}
\caption{Quantum Normalized Stochastic Gradient Descent (\texttt{QNSGD})}
\label{alg: QNSGD}
\begin{algorithmic}[1]
\STATE \textbf{Input:} $\eta$, $\sigma$, $p$, $d$, $\epsilon_{\rm g}$, $\delta$, and $T$.
\FOR{$t=0,\cdots, T-1$}
\STATE $\g_t = \texttt{QHTME}(\nabla F(\x_t;\xi), \sigma, p, d, \epsilon_{\rm g}, \delta/T)$.
\STATE $\x_{t+1} = \x_t -\eta \frac{\g_t}{\|\g_t\|}$
\ENDFOR
\STATE \textbf{Output:} $\hat{\x}\sim\operatorname{Unif}\{\x_0,\ldots,\x_{T-1}\}$
\end{algorithmic}
\end{algorithm}

\begin{theorem}
\label{thm:non-convex_heavy_tailed}
  Suppose that $f(\cdot)$ satisfies Assumptions~
    \ref{ass:heavy-tailed} and \ref{ass:smooth}, let $\Delta_f:=f(\x_0)-f^*$, and assume $0<\epsilon<\sigma$ and $\delta\in(0,1)$. Run {\rm \texttt{QNSGD}} with
    $ \eta = \epsilon/(3L)$, $T = \lceil36L\Delta_f\epsilon^{-2}\rceil$, and $\epsilon_{\rm g}=\epsilon/32$.
    Then, with probability at least $1-\delta$, the iterates satisfy
    $$
        \frac{1}{T}
        \sum_{t=0}^{T-1}
        \|\nabla f(\x_t)\|
        \leq\epsilon.
    $$
    Consequently, {\rm \texttt{QNSGD}} takes
    $$ \tilde{\mathcal{O}}\!\left(L\Delta_f\sqrt d\,\sigma^{\frac{p}{2(p-1)}}\epsilon^{-\frac{5p-4}{2(p-1)}}\right)$$ queries to the quantum stochastic gradient oracle such that $\mathbb{E}[\|\nabla f(\hat{\vx})\|]\leq \epsilon$.
\end{theorem}

\begin{remark}
Theorem~\ref{thm:non-convex_heavy_tailed} recovers the  $\tilde{\OM}\left(d^{1/2}\epsilon^{-3}\right)$ query complexity established in \cite{sidford2023quantum}  for stochastic non-convex optimization under the bounded variance assumption.
Moreover, the theorem indicates that quantum speedups extend to heavy-tailed gradient noise for all $p\in(1,2]$. The sharper bound improves upon the classical lower bound $\Omega(\epsilon^{-\frac{3p-2}{p-1}})$ whenever $d\lesssim(\sigma/\epsilon)^{\frac{p}{p-1}}$.
\end{remark}

\paragraph{Comparison with classical \texttt{NSGD}.}

\citet{hubler2024gradient} analyzed \texttt{NSGD} with the mini-batch gradient estimator used in Algorithm~\ref{alg: QNSGD},
$
\g_t
= \frac{1}{|\fB|}
\sum_{\xi\in\fB}
\nabla f(\x_t;\xi).
$
They showed that it can find an $\epsilon$-stationary point using $\OM(\epsilon^{-\frac{3p-2}{p-1}})$ queries to the classical stochastic gradient oracle, matching the corresponding classical lower bound.
To ensure that the gradient estimation error is small enough, s.t. $\|\bepsilon_t\|\leq \epsilon$,  
their approach requires a mini-batch size $|\fB | = \OM(\epsilon^{-\frac{p}{p-1}})$.
In contrast, our approach \texttt{QNSGD} applies \texttt{QHTME} to obtain $\g_t$ and ensure $\|\bepsilon_t\|\leq\epsilon_{\rm g}$ with probability at least $1-\delta/T$ using $\tilde{\OM}\!\left(\sqrt d\,(\sigma/\epsilon_{\rm g})^{\frac{p}{2(p-1)}}\right)$ queries to the quantum stochastic gradient oracle at each iteration.
As a result, \texttt{QNSGD} achieves a quantum speedup over the classical \texttt{NSGD} in the low-dimensional regime.

\section{Quantum Stochastic Convex Optimization with Heavy-Tailed Noise}
\label{sec:convex}
In this section, we consider the case where the objective $f(\cdot)$ in \eqref{eq:obj} is convex but possibly non-smooth. We impose~the following assumptions on $f(\cdot)$ and the constraint set $\fX$.

\begin{assumption}
\label{ass:convex}
    The objective function $f(\cdot)$ and the constraint set $\fX$ satisfy the following conditions: (i) $\fX$ is a nonempty, closed, and convex set with bounded diameter $D>0$; (ii) The function $f(\cdot)$ is convex on $\fX$; (iii) $\|\nabla f(\vx)\|\leq G$.
\end{assumption}

In addition, we denote by $\x^*$ an optimal solution of \eqref{eq:obj}, chosen from the optimal solution set $\fX^{*} := \argmin_{\x\in\fX} f(\x)$ so that $\x^* \in\fX^*$.
Unlike existing approaches based on gradient clipping \cite{zhang2020adaptive} or online-to-convex framework~\cite{liu2025online} for handling heavy-tailed noise, we consider the projected stochastic gradient descent (\texttt{PSGD}):
\begin{align}
\label{eq:SGD}
    \x_{t+1} = \Pi_{\fX} (\x_t-\eta \g_t),
\end{align}
where $\g$ is a stochastic estimator of $\nabla f(\x)$.
The following lemma bounds the suboptimality of the averaged iterate produced by \texttt{PSGD}.
\begin{lemma}
\label{lm:sgd}
  Run {\rm \texttt{PSGD}} \eqref{eq:SGD} for $T$ iterations, and define the averaged iterate $\bar{\x}_T = \frac{1}{T}\sum_{t=0}^{T-1}\x_t$ as well as the gradient estimation error $\bepsilon_t = \g_t - \nabla f(\x_t)$. If $\EBP{\g_t\mid\x_t}=\nabla f(\x_t)$ and $f(\cdot)$ satisfies Assumptions \ref{ass:heavy-tailed} and \ref{ass:convex}, then for any $p\in(1,2]$,
   $$\EBP{f(\bar{\x}_T)}-f(\x^{*}) \leq \frac{D^2}{2\eta T} +\eta G^2 + C_p\eta^{p-1}D^{2-p}\frac{1}{T}\sum_{t=0}^{T-1}\EBP{\|\bepsilon_t\|^p},$$
    where $C_p:={(4p-4)^{p-1}}/{p^p}.$
\end{lemma}

By properly tuning the step size $\eta$ and controlling the error term $\frac{1}{T}\sum_{t=0}^{T-1}\EBP{\|\bepsilon_t\|^p}$, one can guarantee the convergence of \texttt{PSGD}.
To this end, we apply \texttt{QUHTME} introduced in Section~\ref{subsec: Quantum Unbiased Heavy-Tailed Mean Estimator} to construct an \textbf{unbiased estimator} $\g_t$ of $\nabla f(\x_t)$ with controlled heavy-tailed noise. This leads to~the quantum projected stochastic gradient descent algorithm (\texttt{QPSGD}), presented in Algorithm~\ref{alg: QPSGD}.
The following theorem presents the query complexity of \texttt{QPSGD} for convex stochastic optimization under heavy-tailed noise.

\begin{algorithm}
\caption{Quantum Projected Stochastic Gradient Descent (\texttt{QPSGD})}
\label{alg: QPSGD}
\begin{algorithmic}[1]
\STATE \textbf{Input:} $\eta$, $T$, $\sigma$, $p$, $d$, and $\tilde{\epsilon}$.
\FOR{$t=0,\cdots, T-1$}
\IF{$\tilde{\epsilon}<\sigma$}
\STATE $\g_t = \texttt{QUHTME}(\nabla F(\x_t;\xi), \sigma, p, d, \tilde{\epsilon})$.
\ELSE
\STATE $\g_t\gets\nabla F(\x_t;\xi_t)$ using one oracle query.
\ENDIF
\STATE $\x_{t+1} = \Pi_{\fX}(\x_t -\eta \g_t)$
\ENDFOR
\STATE \textbf{Output:} $\hat{\x}=\frac{1}{T}\sum_{t=0}^{T-1}\x_t$ 
\end{algorithmic}
\end{algorithm}


\begin{theorem}\label{thm:convex_heavy_tailed}
Suppose that $f(\cdot)$ satisfies Assumptions~\ref{ass:heavy-tailed} and \ref{ass:convex}. Run {\rm \texttt{QPSGD}} (Algorithm~\ref{alg: QPSGD}) with 
\begin{equation*}
    \eta = \frac{\epsilon}{3G^2},~~T = 3G^2D^2\epsilon^{-2},
    \text{and}~~~\tilde{\epsilon}=\left(\frac{3^{p-1}}{6C_p}\right)^{\frac{1}{p}}\epsilon^{\frac{2-p}{p}}G^{\frac{2(p-1)}{p}}D^{-\frac{2-p}{p}}.
\end{equation*}
Then the output $\hat{\x}$ satisfies $\EBP{f(\hat{\x})}-f(\x^{*})\leq \epsilon$.
Moreover, the expected number of queries made by {\rm \texttt{QPSGD}} is at most
\[
\tilde{\OM}\!\left(
   \sqrt d\,G D^{\frac{3p-2}{2(p-1)}}\sigma^{\frac{p}{2(p-1)}}\epsilon^{-\frac{3p-2}{2(p-1)}} + G^2D^2\epsilon^{-2}
\right).
\]
\end{theorem}


\paragraph{Comparison with \cite{liu2025online,fatkhullin2025can}.}
Our approach is partially motivated by \citet{liu2025online}, who established regret bounds for online gradient descent under heavy-tailed noise and then applied the online-to-convex conversion framework to obtain a query complexity of $\OM(\epsilon^{-\frac{p}{p-1}})$ for stochastic convex optimization.
\citet{fatkhullin2025can} also applied a similar projected stochastic gradient descent method to handle heavy-tailed noise. 
Both approaches use a single sample per iteration, so their iteration complexities are also $\OM(\epsilon^{-\frac{p}{p-1}})$. 
In contrast, Theorem \ref{thm:convex_heavy_tailed} shows that the iteration complexity of \texttt{QPSGD} is $T=\OM(\epsilon^{-2})$, which is better than that of \cite{liu2025online,fatkhullin2025can}.
Moreover, although \texttt{QPSGD} uses mini-batch sampling to construct quantum gradient estimators at each iteration, the quantum speedups provided by our newly developed quantum mean estimators \texttt{QUHTME} ensure that the total query complexity of \texttt{QPSGD} still improves upon the classical methods in a low dimension regime.
The sharper bound improves upon the classical lower bound for $p\in(1,2)$ whenever $d\lesssim\epsilon^{-\frac{2-p}{p-1}}$.

\section{Conclusion}
\label{sec:conclu}

In this paper, we studied quantum stochastic optimization under heavy-tailed noise. 
We developed a novel quantum mean estimator for random vectors with heavy-tailed distributions, called \texttt{QHTME}, and demonstrated that it achieves a quantum speedup over optimal classical mean estimators. 
We further refined \texttt{QHTME} into an unbiased quantum mean estimator, namely \texttt{QUHTME}.
We established quantum lower bounds, showing that both \texttt{QHTME} and \texttt{QUHTME} are optimal up to logarithmic factors when the dimension $d$ is a constant. Our stronger
dimension-dependent bounds further show that a nontrivial dependence
on the dimension is unavoidable for tail index $p>4/3$ in the low-dimensional regime.
Building on these quantum tools, we proposed \texttt{QNSGD} and \texttt{QPSGD} for non-convex and convex stochastic optimization, respectively, and established quantum speedups in both settings.
The sharper mean-estimation bound yields query complexities $\tilde{\OM}(\sqrt d\,\epsilon^{-\frac{5p-4}{2p-2}})$ for \texttt{QNSGD} and $\tilde{\OM}(\sqrt d\,\epsilon^{-\frac{3p-2}{2p-2}}+\epsilon^{-2})$ for \texttt{QPSGD}, improving upon the optimal classical rates in the low-dimensional regime.


\bibliographystyle{plainnat}
\bibliography{reference}

\newpage
\appendix
\onecolumn

\section{Discussion on the Quantum Oracles}

\subsection{Discussion on the Quantum Access Model for Random Variables}\label{appendix: quantum access model for random variables}

\paragraph{Comparisons with existing quantum sampling oracles.} Unlike in the classical setting, existing work \cite{Cornelissen_2022} assumes that quantum access to the distribution $\BP$ and to the random variable $\X$ is typically provided separately, via a probability oracle and a binary oracle, respectively.

\begin{definition}[Probability oracle]\label{def: probability oracle}
Let $\X : \Omega \to E$ be a random variable on a probability space $(\Omega, 2^{\Omega}, \mathbb{P})$, and let $\mathcal{H}_{\Omega}$ be a Hilbert space with basis states $\{|\omega\rangle\}_{\omega \in \Omega}$. A probability oracle for the probability distribution $\BP$ is a~unitary operator acting on $\mathcal{H}_{\Omega}$ such that 
\begin{equation*}
U_{\mathbb{P}} : |0\rangle \mapsto \sum_{\omega \in \Omega} \sqrt{\mathbb{P}(\omega)}\,|\omega\rangle,
\end{equation*}
assuming $0\in \Omega$.

\end{definition}

\begin{definition}[Binary oracle]\label{def: binary oracle}
Let $\X : \Omega \to E$ be a finite-valued random variable on a probability space $(\Omega, 2^{\Omega}, \mathbb{P})$.
Let $\mathcal{H}_{\Omega}$ and $\mathcal{H}_{E}$ be two Hilbert spaces with basis states
$\{|\omega\rangle\}_{\omega \in \Omega}$ and $\{|x\rangle\}_{x \in E}$, respectively.
A binary oracle for $\X$ is a unitary $\mathcal{B}_{\X}$ acting on $\mathcal{H}_{\Omega} \otimes \mathcal{H}_{E}$ such that
\begin{equation*}
    \mathcal{B}_{\X} : |\omega\rangle|\vec{0}\rangle \mapsto |\omega\rangle|\X(\omega)\rangle
\end{equation*}
for all $\omega \in \Omega$, assuming $\vec{0} \in E$.

\end{definition}


However, this quantum access model is stronger than the classical sampling oracle defined in Definition~\ref{def: Classical sampling oracle}. In particular, it provides coherent access to the underlying random seed $\omega$, as the probability oracle $U_{\mathbb{P}}$ prepares a superposition over all $\omega \in \Omega$, which can then be queried via the binary oracle $B_X$. This allows quantum algorithms to process all samples in superposition.

In contrast, the classical sampling procedure only returns independent samples of the form $\X(\omega)$, and does not provide access to the underlying seed $\omega$. We therefore choose to base our results on the quantum sampling oracle of~\ref{def:quantum-sampling-oracle}, which can be viewed as a coherent analogue of the classical sampling oracle. Although our quantum access model is weaker than the one considered in~\cite{Cornelissen_2022}, this difference does not affect the upper bound in Lemma~\ref{thm:qbvme}. In particular, the algorithm \texttt{QEstimator} in~\cite{Cornelissen_2022} only relies on coherent access to the values of the random variable, rather than the underlying random seed.

We also note that \cite[Definition 1]{sidford2023quantum} adopts a closely related definition of a quantum sampling oracle, where the quantum state is expressed as a superposition over a continuous probability distribution using an integral representation. \begin{definition}[Quantum sampling oracle in \cite{sidford2023quantum}]
For a $d$-dimensional random variable $\X$, its quantum sampling oracle $O_\X$ is defined as
\begin{equation}
O_\X \ket{0} \;\mapsto\; \int_{\x \in \mathbb{R}^d} \sqrt{p_\X(\x)} \, dx \, \ket{\x} \otimes \ket{\mathrm{garbage}(\x)},
\end{equation}
where $p_\X(\cdot)$ represents the probability density function of $\X$.
\end{definition}
This formulation should be interpreted as an ``\textit{idealized model}'', since physical implementations necessarily operate on finite-precision representations. In practice, continuous random variables must be discretized to be encoded into quantum states. Under a suitable discretization scheme, the integral representation in \cite[Definition 1]{sidford2023quantum} can be approximated by a finite superposition over a discrete support. In this sense, their definition is consistent with ours up to discretization, and both models capture the same underlying notion of coherent sampling from a distribution.

\paragraph{Compatibility with heavy-tailed assumptions.}
Our oracle model assumes that the random variable $\X$ takes values in a finite set $E \subset \mathbb{R}^d$, which is necessary for representing quantum states using a finite number of qubits. At first glance, this may appear at odds with the heavy-tailed setting considered in this paper, where the underlying distribution may have unbounded support and possibly infinite variance.

This discrepancy can be resolved by interpreting our oracle as operating on a truncated and discretized approximation of an underlying heavy-tailed distribution. Concretely, let $\X$ be a random variable satisfying a finite $p$-th moment condition for some $p \in (1,2]$, while its variance may be infinite when $p<2$. We consider a truncated version
\[
\X^{(R)} := \X \cdot \mathbf{1}\{\|\X\| \le R\}
\]
for a sufficiently large truncation radius $R$, and discretize its support using a sufficiently fine grid. The resulting finite-valued random variable can then be encoded into a quantum state as in Definition~\ref{def:quantum-sampling-oracle}.

Under this interpretation, the finite support $E$ should be viewed as a finite-precision approximation of a continuous distribution, where both the truncation radius and discretization granularity are chosen so that the induced bias is negligible at the target accuracy. In particular, for sufficiently large $R$, the bias introduced by truncation can be made arbitrarily small, while the $p$-th central moment of $\X^{(R)}$ remains finite and comparable to that of $\X$.

Although $\X^{(R)}$ has finite support and hence finite variance, this variance may grow with $R$ and can be arbitrarily large as the approximation becomes more faithful to the original heavy-tailed distribution. As a result, the finite-support implementation does not reduce the problem to a bounded-variance regime in any meaningful uniform sense. Our analysis instead relies on $p$-th moment bounds, which are stable under truncation and discretization up to controlled constants.

Finally, we note that such truncation and discretization are not specific to the quantum setting. In classical computing, sampling from continuous or unbounded distributions also necessarily involves finite-precision representations and approximation schemes. For example, the Ziggurat method~\cite{marsaglia2000ziggurat} can be viewed as a structured instance of a truncate-and-discretize approach, combined with efficient rejection sampling for generating random variables from unimodal distributions such as the Gaussian or exponential.

\subsection{Discussion on the Quantum Stochastic Gradient Oracle}\label{appendix: quantum stochastic gradient oracle}

We note that~\cite[Definitions 3 and 4]{sidford2023quantum} is based on a quantum stochastic gradient oracle that requires coherent access to the query point $\x$. By contrast, our model (Definition \ref{def: Quantum stochastic gradient oracle (QSGO)}) only assumes quantum access to the stochastic gradient distribution at each fixed query point. In optimization problems, the query point $\x$ represents the current parameter vector (e.g., the weights of a classical or quantum neural network), which is updated adaptively by the algorithm and typically ranges over a continuous or very high-precision domain. Requiring coherent access to $\ket{\x}$ therefore amounts to the (strong) assumption of a quantum interface that can reversibly encode and query the stochastic gradient distribution for all possible parameter values in superposition.

\section{Discussion on Practicality and Potential Applications}
\label{app:practicality}

We briefly discuss the practical scope of our algorithms, focusing on the quantum access model, the role of the dimension dependence, and potential application domains.

\paragraph{Quantum oracle implementation.}
Our algorithms are formulated in the quantum oracle models for stochastic optimization and mean estimation: a quantum stochastic gradient oracle (Definition \ref{def: Quantum stochastic gradient oracle (QSGO)}) coherently prepares a stochastic gradient sample, while a quantum sampling oracle (Definition \ref{def:quantum-sampling-oracle}) coherently prepares a random vector whose mean is to be estimated. These models should be interpreted as the quantum analogue of the classical stochastic first-order oracle and sampling oracle, respectively. In particular, if the corresponding classical sampling or stochastic-gradient routine is implemented by an explicit reversible circuit, then the standard reversible-computation construction yields a quantum circuit implementing the associated unitary oracle and its inverse, up to the constant overhead for reversible simulation \cite[Chapter 1.4.1]{nielsen2010quantum}. Thus, the relevant comparison in this paper is a black-box query-complexity comparison between classical and quantum algorithms under analogous oracle access assumptions. Within this framework, our results demonstrate that quantum algorithms can achieve a provable advantage over their classical counterparts in the low-dimensional regime.

\paragraph{Hardware implementation.} The quantum (unbiased) heavy-tailed mean estimators used in our quantum stochastic optimization algorithms require coherent oracle access, inverse oracle invocation, and sufficiently deep controlled quantum operations. The quantum circuits required to implement our algorithms are too deep to run on existing devices (which have limited qubit counts and high error rates), and require large-scale fault-tolerant quantum computers to implement. This is also the case for the related literature on quantum univariate and multivariate mean estimation \cite{Cornelissen_2022,tang2025moreefficientquantummultivariatemean,wu2023quantum,blanchet2024quadraticspeedupinfinitevariance,Montanaro_2015}. Until the technology becomes available to implement these algorithms on real devices, demonstrating theoretical, provable quantum speedups for different optimization problems remains an important task in the machine learning and optimization communities. In this sense, our paper is aligned with a broader body of work on quantum algorithms for learning and optimization. \cite{zhang2023quantumlowerboundsfinding,liu2024quantum,zhang2024quantum,wangnear,liuquantum,sidford2023quantum}.

\paragraph{Dimension dependence and the low-dimensional regime.}
The quantum speedups obtained by \texttt{QNSGD} and \texttt{QPSGD} are in the low-dimensional regime. For nonconvex optimization, \texttt{QNSGD} uses
\[
\widetilde O\!\left(
\sqrt d\,\epsilon^{-\frac{5p-4}{2p-2}}
\right)
\]
queries to the quantum stochastic gradient oracle, compared with the classical lower bound $\Omega\!\left(\epsilon^{-\frac{3p-2}{p-1}}\right).$ Hence \texttt{QNSGD} gives a speedup when $d \lesssim \epsilon^{-\frac{p}{p-1}}$. For convex optimization, \texttt{QPSGD} uses
\[
\widetilde O\!\left(
\sqrt d\,\epsilon^{-\frac{3p-2}{2p-2}}+\epsilon^{-2}
\right)
\]
queries, compared with the classical lower bound
$\Omega\!\left(\epsilon^{-\frac{p}{p-1}}\right),$ 
and gives a speedup when
$d \lesssim \epsilon^{-\frac{2-p}{p-1}}$ for $p\in(1,2)$.

Although we only demonstrated quantum speedup in the low-dimensional regime, it remains meaningful in some settings. 
\begin{itemize}
    \item First, while in classical machine learning (ML), the large $d$ setting is typically of interest, in quantum ML, it is desirable to have much parameter-sparse neural networks. This stems from the fact that (i) backpropagation cannot be performed in quantum neural networks, and thus gradients are typically estimated via either the parameter-shift~\cite[Section C.1]{Cerezo_2021} or finite difference methods~\cite{spall2002multivariate}, both of which require 
 $\mathcal{O}(d)$ network evaluations to compute a single gradient; and (ii) increasing the number of parameters can lead to ``barren plateaus''~\cite{McClean_2018}, the quantum equivalent of vanishing gradients. To keep the quantum neural networks efficiently trainable, highly structured networks with a sparse number of parameters are desirable. Our algorithms require access to mini-batches of data in quantum superposition (which is implicit in Definition \ref{def:quantum-sampling-oracle}) and therefore can be used for estimating the gradients and training the parameters in quantum neural networks. In such low-dimensional or parameter-sparse settings, our algorithms for finding (locally) optimal parameters are potentially more efficient than their classical counterparts.
 \item 
 Second, our quantum algorithms may also be useful for accelerating optimization problems encountered in the classical domain, as many high-dimensional optimization methods reduce each update to a low-dimensional randomized subproblem. Examples include sketch-and-project methods for linear systems~\cite{Gower_2015}, algorithm for approximating the best point problem~\cite[Problem 5]{sidford2023quantum}, and stochastic quasi-Newton methods with sketching~\cite{zhang2021fasterstochasticquasinewtonmethods}. These methods suggest that the effective dimension of an update can be much smaller than the problem dimension, even when the original problem is high-dimensional. In a stochastic setting with heavy-tailed gradient noise, if the low-dimensional projected gradient or sketched stochastic quantity satisfies a bounded $p$-th moment condition, then our quantum heavy-tailed mean estimators can be potentially used as subroutines for estimating these low-dimensional quantities and provide a quantum speedup.
\end{itemize}

\paragraph{Natural application scenarios.}
A natural application domain of our quantum stochastic optimization algorithms is training parameterized quantum models under heavy-tailed stochastic gradients. The stochasticity may arise from mini-batching, randomized measurements, coordinate sampling, or stochastic circuit components. The data may be classical but available through coherent access mechanisms such as QRAM, or it may be intrinsically quantum, for example, when the training sample is produced by another quantum circuit. In these settings, the goal is to optimize quantum \cite{Hur_2022,cong2019quantum} or hybrid quantum-classical \cite{henderson2020quanvolutional} models whose parameter dimension is moderate and whose stochastic gradients may fail to have bounded variance.

More broadly, the quantum heavy-tailed mean estimators developed in this paper are independent algorithmic primitives. Heavy-tailed expectation estimation appears in learning and decision-making problems beyond stochastic optimization, including reinforcement learning with heavy-tailed rewards~\cite{zhuang2021noregretreinforcementlearningheavytailed,huang2024tacklingheavytailedrewardsreinforcement} and policy-gradient methods with heavy-tailed gradient noise~\cite{garg2021proximalpolicyoptimizationsheavytailed}. If the relevant environment, transition, reward, or gradient oracle can be accessed coherently, then \texttt{QHTME} or \texttt{QUHTME} may serve as a quantum subroutine for reducing the query complexity of these problems.

\paragraph{Scope and limitations.}
Our results are theoretical query-complexity results, and we do not make claims about quantum advantage on near-term quantum hardware devices.  The main contribution is to show that, under the standard coherent oracle model, stochastic optimization with heavy-tailed noise admits provable quantum query complexity speedups in regimes where the effective dimension is sufficiently small. This positions the algorithms as fault-tolerant quantum subroutines for robust stochastic optimization, rather than immediate replacements for classical optimizers on current hardware.

\section{Proof in Section \ref{sec: QME for HTRV}}
\subsection{Classical Heavy-Tailed Mean Estimator}\label{subsec: CHTME}
\citet{whitehouse2025meanestimationbanachspaces} was the first work to explicitly formalize a classical heavy-tailed mean estimator. Prior to this, the geometric median-of-means approach of \citet{Minsker_2015} already implied such an estimator, although they did not state this result explicitly. In Appendix B, \citet{whitehouse2025meanestimationbanachspaces} further provides a formal treatment of the geometric median-of-means construction. However, both their original estimator and this reformulation require the heavy-tailed parameter $\sigma$ as an input. In contrast, we present below an alternative formalization of the geometric median-of-means approach that eliminates the need for prior knowledge of $\sigma$, but achieves the same error bound as the previous two classical estimators in \cite{whitehouse2025meanestimationbanachspaces}.

\begin{algorithm}[htpb]
\caption{Classical Heavy-tailed Mean Estimator($\texttt{CHTME}(\X,n,\delta)$)}
\label{alg:multivariate_mom}
\begin{algorithmic}[1]
\REQUIRE Samples $\X_1,\dots,\X_n \in \mathbb{R}^d$, confidence level $\delta \in (0,1)$

\STATE Set $k \gets  \lceil24 \ln(1/\delta)\rceil$
\STATE Set $m \gets  \lfloor n / k \rfloor$

\STATE Partition $\{\X_i\}_{i=1}^{km}$ into $k$ disjoint blocks $\{B_j\}_{j=1}^k$ of size $m$

\FOR{$j = 1$ to $k$}
    \STATE Compute block mean
    \[
        \Z_j \gets \frac{1}{m} \sum_{i \in B_j} \X_i
    \]
\ENDFOR

\STATE Compute the geometric median of $\{\Z_1,\dots,\Z_k\}$:
\[
   \tilde{\bmu}
    \;\gets\;
    \arg\min_{x \in \mathbb{R}^d}
    \sum_{j=1}^k \| \Z_j - x \|
\]

 \RETURN $ \tilde{\bmu}$
\end{algorithmic}
\end{algorithm}

To provide the performance guarantee for the classical mean estimator \texttt{CHTME}, we will use the following Lemmas.

\begin{lemma}[Proximity lemma for the geometric median] \label{lem:proxy} Let $\X_1,\ldots, \X_k$ be independent samples of a $d$-dimensional variable with mean $\bmu$, and let $ \tilde{\bmu}$ be its geometric median. If a $(1/2 +\alpha)k$ fraction of the samples lie within Euclidean distance $r$ of $\bmu$, then 
$$\Norm{ \tilde{\bmu}-\bmu} \leq \left(1 + \frac{1}{2\alpha}\right) r.$$

\end{lemma}
\begin{proof}
    Define function $f: \mathbb{R}^{d}\rightarrow \mathbb{R}$ satisfying \begin{equation*}
        f(y)=\sum_{i=1}^{k} \Norm{y-\X_{i}}.
    \end{equation*}
    We define $ \tilde{\bmu}:= \arg\min_{y\in \mathbb{R}^{d}} f(y)$. Then it holds that $f( \tilde{\bmu})\leq f(\bmu)$ for the mean $\bmu\in \mathbb{R}^{d}$. 

    For arbitrary $y\in\mathbb{R}^{d}$, we denote $\rho := \|y-\bmu\|$. We consider the following two sets:
    \begin{equation*}
        G:=\{\X_{i}\in \{\X_{i}\}_{i=1}^{k}: \Norm{\X_{i}- \bmu}\leq r \}, \quad B:=\{\X_{i}\in \{\X_{i}\}_{i=1}^{k}: \Norm{\X_{i}- \bmu}> r \}.
    \end{equation*}
    Specifically, the set $G$ contains the ``good'' points among the samples $\{\X_{i}\}_{i=1}^{k}$ satisfying $\Norm{\X_{i}-\bmu}\leq r$. We assume the cardinality of $G$ is $|G|=m$. The set $B$ contains the ``bad'' samples, which are far from the mean $\bmu$. There are $k-m$ such bad samples.
   
    For any good point $\X_i$ in $G$, by triangle inequality, we have
    \begin{equation*}
        \|y-\X_{i}\|\geq \Norm{y-\bmu} - \Norm{\X_{i}-\bmu}\geq \rho -r.
    \end{equation*}
    For any bad point $\X_{i}$ in $B$, by triangle inequality,
    \begin{equation*}
        \Norm{y-\X_{i}}\geq \Norm{\X_{i}-\bmu}-\Norm{y-\bmu}=\Norm{\X_{i}-\bmu}-\rho.
    \end{equation*}
    Therefore, we have
    \begin{equation*}
        f(y) =\sum_{\X_{i}\in G} \Norm{y-\X_{i}}+\sum_{\X_i\in B} \Norm{y-\X_{i}}\geq m(\rho-r)+\sum_{\X_{i}\in B}\Norm{\X_{i}-\bmu}-(k-m)\rho.
    \end{equation*}
    Besides, we know that
    \begin{equation*}
        f(\bmu)=\sum_{\X_{i}\in G} \Norm{\bmu -\X_{i}}+\sum_{\X_{i}\in B} \Norm{\bmu-\X_{i}}\leq mr+\sum_{\X_{i}\in B} \Norm{\bmu-\X_{i}}.
    \end{equation*}
    By computing the subtraction, we have
    \begin{equation*}
        f(y)-f(\bmu)\geq (m-(k-m))\rho-2mr=(2m-k)\rho-2mr.
    \end{equation*}
    Suppose that there are $(1/2+\alpha)k$ good samples in $G$, i.e., $m=(1/2+\alpha)k$. Then, we have 
    \begin{equation*}
        f(y)-f(\bmu)\geq 2\alpha k \rho - 2\left(\frac{1}{2}+\alpha\right) kr=2k\left(\alpha \rho - \left(\frac{1}{2}+\alpha\right)r\right).
    \end{equation*}
    Therefore, when $m=(1/2+\alpha)k$, if we have $\rho>\left(1+\frac{1}{2\alpha}\right) r$, we have $f(y)>f(\bmu)$, where $y$ is an arbitrary point in $\mathbb{R}^{d}$. In particular, when $y= \tilde{\bmu}$, i.e., $\| \tilde{\bmu}-\bmu\|> \left(1+\frac{1}{2\alpha}\right) r$, then $f( \tilde{\bmu})>f(\bmu)$, which contradicts the definition of $ \tilde{\bmu}:=\arg\min_{y\in \mathbb{R}^{d}}f(y)$. Therefore, it must hold that
    \begin{equation*}
        \| \tilde{\bmu}-\bmu\|\leq \left(1+\frac{1}{2\alpha}\right) r.
    \end{equation*}
\end{proof}

\begin{lemma}[Vector-valued von Bahr--Esseen inequality~\cite{vonBahrEsseen1965}]\label{lem:vBE} Let $\X_1,\ldots, \X_n$ be independent $d$-dimensional real random variables, with $\EBP{\X_i}=0$ and $\EBP{\Norm{\X_i}^p} < \infty$ for some $p\in[1,2]$. Then,

$$\EBP{\Norm{\sum_{i=1}^n \X_i}^p} \leq C_p \sum_{j=1}^n \EBP{\Norm{\X_j}^p}$$

where $C_p$ is a ($p$-dependent) constant that does not depend on $d$.
\end{lemma}

\begin{theorem}
    Let $\X$ be a $d$-dimensional random variable with mean $\bmu:=\EBP{\X}$ and finite $p$-th central moment $\EBP{\|\X-\bmu\|^p}$~for some $p\in(1,2]$.
    Let $\delta \in (0,1)$ and $k=\lceil24\ln(1/\delta)\rceil$. For integer $n$ such that $n\geq 2k$,  the classical heavy-tailed mean estimator {\rm \texttt{CHTME}}$(\X, n,\delta)$ outputs a mean estimate $ \tilde{\bmu}$ such that
\begin{equation*}
    \Norm{ \tilde{\bmu}-\bmu}\leq 3(8C_{p})^{1/p} \left(\mathbb{E}[\Norm{\X-\bmu}^p]\right)^{1/p}  \left(\frac{2\lceil24\ln(1/\delta)\rceil}{n}\right)^{\frac{p-1}{p}},
\end{equation*}
by using $n$ queries to the classical sampling oracle in Definition \ref{def: Classical sampling oracle}.
\end{theorem}

\begin{proof} Let $\Y_i = \X_i-\bmu$. Then,

\begin{align*}
\mathbb{E}[\lVert \Z_{j}-\bmu \rVert^p ] &= \mathbb{E}\left\lVert   \frac{1}{m} \sum_{i=1}^m (\X_{i}-\bmu)  \right\rVert^p  = \mathbb{E}\left\lVert  \frac{1}{m}\sum_{i=1}^m \Y_{i} \right\rVert^p \\
&=\frac{1}{m^p}\mathbb{E} \left\lVert  \sum_{i=1}^m \Y_{i} \right\rVert^p   \\
&\le \frac{C_{p}}{m^p} \sum_{i=1}^m \mathbb{E}[\lVert \Y_{i} \rVert^p ] &\text{Lemma } \ref{lem:vBE} \\
&\leq \frac{C_{p}\mathbb{E}[\Norm{\X-\bmu}^p]}{m^{p-1}} & \X_{i}\text{ are independent.} \\
&\leq  C_{p}\mathbb{E}[\Norm{\X-\bmu}^p] \left(\frac{k}{n-k}\right)^{p-1} & \left(m=\left\lfloor\frac{n}{k}\right\rfloor\geq \frac{n}{k}-1=\frac{n-k}{k}\right)\\
&\leq  C_{p}\mathbb{E}[\Norm{\X-\bmu}^p] \left(\frac{2k}{n}\right)^{p-1}. & \left(n\geq 2k\Rightarrow n-k\geq \frac{n}{2}\right)    
\end{align*}

By Markov's inequality, we have

\begin{equation*}
P(\Norm{\Z_j-\bmu} > r) = P(\Norm{\Z_j-\bmu}^p > r^p) \le \frac{\mathbb{E}[\Norm{\Z_j-\bmu}^p]}{r^p} \le C_p \frac{\mathbb{E}[\Norm{\X-\bmu}^p] }{r^{p}} \left(\frac{2k}{n}\right)^{p-1}:= p_\text{fail}.    
\end{equation*}


By setting
\begin{equation*}
    r=(8C_{p})^{1/p} \left(\mathbb{E}[\Norm{\X-\bmu}^p]\right)^{1/p}  \left(\frac{2k}{n}\right)^{\frac{p-1}{p}},
\end{equation*}
we have $p_{\text{fail}}=\frac{1}{8}$.

Define the random variables $B_j = \mathbb{I}[\Norm{\Z_j-\bmu} > r]$, i.e., with probability $p < p_\text{fail}$ we have $B_j=1$. We also define the random variables $G_j = \mathbb{I}[\Norm{\Z_j-\bmu} \leq r]$. Note that $\sum_{j=1}^{k}(G_{j}+B_{j})=k$.   Let $Q_j$ be a Bernoulli random variable with $p(Q_j=1)=p_\text{fail}$.  Then, since $p < p_\text{fail}$ we have
\begin{align*}
    P(\sum_{j=1}^{k} B_j \ge M) \le P(\sum_{j=1}^{k} Q_j \ge M)
\end{align*}
Take $M = (1+\Delta) k \cdot p_\text{fail}$.
By the multiplicative Chernoff bound for $\Delta\in[0,1]$,
\begin{align*}
    P\left(\sum_{j=1}^{k} B_j \ge  (1+\Delta) k \cdot p_\text{fail}\right) \le P\left(\sum_{j=1}^{k} Q_{j} \ge  (1+\Delta) k \cdot p_\text{fail}\right) &\le e^{-k\cdot p_\text{fail}\Delta^2/3}.
\end{align*}


In order to have $(1+\Delta)p_\text{fail} = 1-(1/2+\alpha)$, we choose $\Delta= 8(1/2-\alpha)-1=3-8\alpha$. 
Therefore, we know that
\begin{equation*}
    P\left(\sum_{j=1}^{k} G_j \geq   (1/2+\alpha) k \right) \geq 1- e^{-k\cdot (3-8\alpha)^2/24}.
\end{equation*}
It indicates that, with probability at least $ 1- e^{-k\cdot (3-8\alpha)^2/24}$, there are $(1/2+\alpha)k$ good samples among $\{\Z_{j}\}_{j=1}^{k}$ satisfying $\Norm{\Z_{j}-\bmu}\leq r$. By Lemma \ref{lem:proxy}, we know that the geometric median $ \tilde{\bmu}$ satisfies
\begin{equation*}
    \Norm{ \tilde{\bmu}-\bmu}\leq \left(1+\frac{1}{2\alpha}\right) r,
\end{equation*}
with probability at least $ 1- e^{-k\cdot (3-8\alpha)^2/24}$.

By choosing $\alpha=\frac{1}{4}$ so that $\Delta=1$, and $k=\lceil24\ln(\frac{1}{\delta})\rceil$, we have
\begin{equation*}
    \Norm{ \tilde{\bmu}-\bmu}\leq \left(1+\frac{1}{2\alpha}\right) r=3(8C_{p})^{1/p} \left(\mathbb{E}[\Norm{\X-\bmu}^p]\right)^{1/p}  \left(\frac{2\lceil24\ln(1/\delta)\rceil}{n}\right)^{\frac{p-1}{p}},
\end{equation*}
with probability at least $ 1- e^{-k/24}\geq 1-\delta$, by using $n\geq 2k\geq  2\lceil24\ln(1/\delta)\rceil$ classical samples obtained by the classical sampling oracle. 
\end{proof}

\begin{corollary} \label{corollary:CHTME-for prob 1}
{\rm \texttt{CHTME}} solves Problem~\ref{prob: obtain mean estimate for heavy tail random variable with bounded mean} using $n(\epsilon)
:=
\tilde{\mathcal{O}}\!\left(
 \sigma^{\frac{p}{p-1}}\epsilon^{-\frac{p}{p-1}}
\right)$ queries to the classical sampling oracle  with probability at least $1-\delta$. 
\end{corollary}

\subsection{Proof of Theorem \ref{thm:QHTME-for prob 1}}\label{appendix: Proof of QHTME for prob 1}

\begin{lemma}\label{lem:rough-centering}
    Let $\Y:=\X-\bzeta$. If
\[
    \EBP{\Norm{\X-\bmu}^p}\le \sigma^p
    \qquad\text{and}\qquad
    \Norm{\bzeta-\bmu}\le \sigma,
\]
then
\[
    \EBP{\Norm{\Y}^p}\le 2^p\sigma^p.
\]
\end{lemma}

\begin{proof}
We have
\[
    \Y=\X-\bzeta=(\X-\bmu)+(\bmu-\bzeta).
\]
Using $(a+b)^p\le 2^{p-1}(a^p+b^p)$ for $a,b\ge0$ and $p\in(1,2]$, we obtain
\[
    \Norm{\Y}^p
    \le
    2^{p-1}\left(\Norm{\X-\bmu}^p+\Norm{\bmu-\bzeta}^p\right).
\]
Taking expectation gives
\[
    \EBP{\Norm{\Y}^p}
    \le
    2^{p-1}\left(\EBP{\Norm{\X-\bmu}^p+\Norm{\bmu-\bzeta}^p}\right)
    \le
    2^{p-1}(\sigma^p+\sigma^p)
    =2^p\sigma^p.
\]
\end{proof}

\begin{lemma}[Bias bound]
\label{lem:bias-bound}
For $B>0$, define
\[
    \Z_B:=\Y\cdot \mathbf{1}\{\Norm{\Y}\le B\}.
\]
Then
\[
    \Norm{\EBP{\Y}-\EBP{\Z_B}}
    \le
    \frac{2^p\sigma^p}{B^{p-1}}.
\]
\end{lemma}

\begin{proof}
Since
\[
    \Y-\Z_B=\Y\cdot \mathbf{1}\{\Norm{\Y}>B\},
\]
we have
\[
    \Norm{\EBP{\Y}-\EBP{\Z_B}}
    \le
    \EBP{\Norm{\Y}\cdot \mathbf{1}\{\Norm{\Y}>B\}}.
\]
On the event $\{\Norm{\Y}>B\}$,
\[
    \Norm{\Y}
    =
    \frac{\Norm{\Y}^p}{\Norm{\Y}^{p-1}}
    \le
    \frac{\Norm{\Y}^p}{B^{p-1}}.
\]
Therefore,
\[
    \EBP{\Norm{\Y}\cdot \mathbf{1}\{\Norm{\Y}>B\}}
    \le
    \frac{\EBP{\Norm{\Y}^p}}{B^{p-1}}.
\]
The final bound follows from Lemma~\ref{lem:rough-centering}.
\end{proof}

\begin{lemma}[Second-moment bound]
\label{lem:second-moment}
For every $B>0$,
\[
    \Tr\bigl(\Sigma_{\Z_B}\bigr)
    \le
    \EBP{\Norm{\Z_B}^2}
    \le
    2^p\sigma^p B^{2-p}.
\]
\end{lemma}

\begin{proof}
By definition,
\[
    \Norm{\Z_B}^2=\Norm{\Y}^2\cdot \mathbf{1}\{\Norm{\Y}\le B\}.
\]
On the event $\{\Norm{\Y}\le B\}$,
\[
    \Norm{\Y}^2
    =
    \Norm{\Y}^p\Norm{\Y}^{2-p}
    \le
    B^{2-p}\Norm{\Y}^p.
\]
Thus
\[
    \Norm{\Z_B}^2\le B^{2-p}\Norm{\Y}^p.
\]
Taking expectation and using Lemma~\ref{lem:rough-centering} proves the right hand inequality. Moreover,
\[
    \Tr(\Sigma_{\Z_B})
    =
    \EBP{\Norm{\Z_B-\EBP{\Z_B}}^2}
    =
    \EBP{\Norm{\Z_B}^2}-\Norm{\EBP{\Z_B}}^2
    \le
    \EBP{\Norm{\Z_B}^2}.
\]
\end{proof}

We now prove Theorem \ref{thm:QHTME-for prob 1}.
\begin{proof}
We first consider the success probability of Algorithm \ref{alg: Quantum Heavy-Tailed Mean Estimator (QHTME)}, which is affected by two subroutines $\texttt{CHTME}$ and $\texttt{QEstimator}$, each of which may fail with probability $\delta/3$. In fact, by the union bound,  the success probability of Algorithm \ref{alg: Quantum Heavy-Tailed Mean Estimator (QHTME)} is at least $1-2\delta/3\geq 1-\delta$. Our subsequent analysis assumes all the subroutines in Algorithm \ref{alg: Quantum Heavy-Tailed Mean Estimator (QHTME)} are successfully implemented.

We obtain a rough estimate $\bzeta$ of the mean $\bmu$ via the classical heavy-tailed mean estimator \texttt{CHTME}. By Theorem \ref{thm :multivariate-mom-error-bound} and the choice of $$n_{c}:=\left\lceil
2\left\lceil 24\ln\left(\frac{6}{\delta}\right)\right\rceil
3^{\frac{p}{p-1}}
(8C_p)^{\frac1{p-1}}
\right\rceil$$ in Algorithm \ref{alg: Quantum Heavy-Tailed Mean Estimator (QHTME)}, we know that
\begin{equation}\label{equ: rough estimate via CHTME}
    \Norm{\bzeta-\bmu}\leq \sigma.
\end{equation}
Note that when calling the classical mean estimator \texttt{CHTME}, we need $n_{c}$ classical samples. In fact, these $n_{c}$ classical samples can be obtained via querying the quantum sampling oracle $O_{X}$ $n_{c}$ times and performing the measurements in the computational basis.

Algorithm~\ref{alg: Quantum Heavy-Tailed Mean Estimator (QHTME)} sets
\[
    B:=\left(\frac{3\cdot 2^p\sigma^p}{\epsilon}\right)^{\frac{1}{p-1}}
\]
as the truncation threshold.
By Lemma~\ref{lem:bias-bound} and \eqref{equ: rough estimate via CHTME}, these give 
\begin{equation}\label{equ: tail error}
    \Norm{\EBP{\Y}-\EBP{\Z_B}}
    \le
    \frac{2^p\sigma^p}{B^{p-1}}
    =
    \frac{\epsilon}{3}.
\end{equation}
By Lemma~\ref{lem:second-moment} and \eqref{equ: rough estimate via CHTME},
\[
    \Tr(\Sigma_{\Z_B})
    \le
    2^p\sigma^p B^{2-p}.
\]
For convenience, define
\[
   L:= 2^p\sigma^p B^{2-p}
\]
Then
\begin{equation}
\label{eq:VB-bound}
    \Tr(\Sigma_{\Z_{B}})\leq L.
\end{equation}

Define
\[
    \Lambda:=\log\left(\frac{3d}{\delta}\right)
\]
and recall that Algorithm~\ref{alg: Quantum Heavy-Tailed Mean Estimator (QHTME)}
sets
\[
    n_q
    :=
    \left\lceil
    \max\left\{
        \Lambda,
        \frac{3\sqrt{dL\Lambda}}{2\epsilon}
    \right\}
    \right\rceil.
\]
Since
\[
    n_q\geq \Lambda
    =
    \log\left(\frac{d}{\delta/3}\right),
\]
the condition on $n$ in Lemma~\ref{thm:qbvme} is satisfied. Therefore,
applying Lemma~\ref{thm:qbvme} to $\Z_B$ with parameter $n=n_q$ and
failure probability $\delta/3$, we obtain, with probability at least
$1-\delta/3$,
\[
    \Norm{
        \widetilde{\bmu}_B-\EBP{\Z_B}
    }_{\infty}
    \leq
    \frac{
        \sqrt{
            \Tr(\Sigma_{\Z_B})
            \log(3d/\delta)
        }
    }{n_q}.
\]
Using $\Norm{\bm{v}}\leq\sqrt{d}\Norm{\bm{v}}_\infty$ and
\eqref{eq:VB-bound}, it follows that
\begin{align*}
    \Norm{
        \widetilde{\bmu}_B-\EBP{\Z_B}
    }
    &\leq
    \sqrt{d}\,
    \Norm{
        \widetilde{\bmu}_B-\EBP{\Z_B}
    }_\infty \\
    &\leq
    \frac{
        \sqrt{
            dL\log(3d/\delta)
        }
    }{n_q}.
\end{align*}
Moreover, by the definition of $n_q$,
\[
    n_q
    \geq
    \frac{
        3\sqrt{
            dL\log(3d/\delta)
        }
    }{2\epsilon}.
\]
Consequently,
\begin{equation}
\label{equ: mean estimation error}
    \Norm{
        \widetilde{\bmu}_B-\EBP{\Z_B}
    }
    \leq
    \frac{2\epsilon}{3}.
\end{equation}

Note that one query to the quantum sampling oracle $O_{\Z_B}$ for the truncated random variable $\Z_B$ can be simulated using one query to the quantum sampling oracle $O_\X$ for the original random variable $\X$. Indeed, after applying $O_\X$, we obtain a coherent sample $\x$ from $\X$. Since both $\bzeta$ and $B$ are classical parameters, we can reversibly compute
\[
    \y=\x-\bzeta,
\]
evaluate the comparison $\|\y\|\le B$, and write the truncated value
\[
\z_B=
\begin{cases}
\y, & \|\y\|\le B,\\
0, & \|\y\|>B
\end{cases}
\]
into a fresh output register.
Therefore, $\texttt{QEstimator}$ step requires $n_{q}$ queries to the quantum sampling oracle $O_{\X}$.


The algorithm outputs
\[
    \widetilde\bmu:=\bzeta+\widetilde \bmu_B.
\]
Since
\[
    \bmu=\bzeta+\EBP{\Y},
\]
we have
\[
    \widetilde\bmu-\bmu
    =
    \widetilde \bmu_B-\EBP{\Y}
    =
    \bigl(\widetilde \bmu_B-\EBP{\Z_B}\bigr)+\bigl(\EBP{\Z_B}-\EBP{\Y}\bigr).
\]
Thus, by \eqref{equ: tail error} and \eqref{equ: mean estimation error},
\[
    \Norm{\widetilde\bmu-\bmu}
    \le
    \Norm{\widetilde \bmu_B-\EBP{\Z_B}}
    +
    \Norm{\EBP{\Z_B}-\EBP{\Y}}
    \le
    \frac{2\epsilon}{3}+\frac{\epsilon}{3}
    =
    \epsilon.
\]
This proves the accuracy guarantee.

By the definition of $n_q$,
\begin{align*}
    n_q
    &\leq
    1
    +
    \log\left(\frac{3d}{\delta}\right)
    +
    \frac{
        3\sqrt{
            dL\log(3d/\delta)
        }
    }{2\epsilon}.
\end{align*}
The call to $\texttt{CHTME}$ uses $n_c=\mathcal{O}(\log(1/\delta))$
queries, while Lemma~\ref{thm:qbvme} shows that the call to
$\texttt{QEstimator}_d$ uses $\widetilde{\mathcal{O}}(n_q)$ queries.
Therefore, the total query complexity is
\begin{align*}
    n_c+\widetilde{\mathcal{O}}(n_q)
    &=
    \widetilde{\mathcal{O}}\left(
        \log\left(\frac{1}{\delta}\right)
        +
        \log\left(\frac{3d}{\delta}\right)
        +
        \frac{
            \sqrt{
                dL\log(3d/\delta)
            }
        }{\epsilon}
    \right) \\
    &=
    \widetilde{\mathcal{O}}\left(
        1+\frac{\sqrt{dL}}{\epsilon}
    \right).
\end{align*}

Since
\[
    L
    =
    2^p\sigma^p B^{2-p}
    \qquad\text{and}\qquad
    B
    =
    \left(
        \frac{3\cdot 2^p\sigma^p}{\epsilon}
    \right)^{\frac{1}{p-1}},
\]
we have
\[
    \frac{\sqrt{dL}}{\epsilon}
    =
    \Theta\left(
        \sqrt{d}
        \left(
            \frac{\sigma}{\epsilon}
        \right)^{\frac{p}{2(p-1)}}
    \right).
\]
Since $d\geq1$ and $\epsilon\leq\sigma$, the additive constant can be
absorbed. Hence, the total query complexity is
\[
    \widetilde{\mathcal{O}}\left(
        \sqrt{d}
        \left(
            \frac{\sigma}{\epsilon}
        \right)^{\frac{p}{2(p-1)}}
    \right).
\]
This completes the proof of Theorem~\ref{thm:QHTME-for prob 1}.
\end{proof}

\subsection{A Variant of Quantum Heavy-Tailed Mean Estimator \texorpdfstring{$\texttt{QHTME}^{+}$}{QHTME+}}\label{appendix: QHTME-plus}

We first derive a variant of the quantum mean estimator \texttt{QHTME}, denoted by \texttt{QHTME$^{+}$}, which outputs an estimate whose estimation error admits a bounded $p$-th raw moment with probability $1$. The key technical ingredient is a new theoretical insight formalized in Lemma \ref{lemma: general pre lemma for unbiased estimator}.~This~lemma shows that, by using a carefully designed acceptance-rejection rule, two random variables $\X_2$ and $\X_3$ with bounded $p$-th central moments can be used to control the $p$-th raw moment of the estimation error of a third random variable $\X_1$, provided that $\X_1$ is sufficiently close to the common expectation of $\X_2$ and $\X_3$. A related idea appeared in \cite{sidford2023quantum}, where an analogous result was established for the special case $p=2$.
In contrast, we generalize their result to all $p\in(1,2]$ in the following Lemma~\ref{lemma: general pre lemma for unbiased estimator}.



\begin{lemma}\label{lemma: general pre lemma for unbiased estimator}
Let $\X_1, \X_2, \X_3\in \mathbb{R}^{d}$ be independent~random variables such that $\bmu:=\EBP{\X_2}= \EBP{\X_3}$, $\EBP{\|\X_{2} - \bmu\|^{p}}  \leq \sigma^{p}$, and $\EBP{\|\X_{3}- \bmu\|^{p}}\leq \sigma^{p}$, for some $p\in (1,2]$ and $\sigma\geq 0$. Assume further that $\|\X_{1}-\bmu\|\leq \epsilon$ with probability at least $1-\delta$, for some $\epsilon>0$ and $\delta\in (0,1)$. Define a random variable $\Z$ by $\Z=\X_1$ if $\|\X_1- \X_2\|\leq D$ and $\Z=\X_{3}$ otherwise, where $D>\epsilon$. Then,
\begin{equation*}
    \EBP{\|\Z-\bmu\|^{p}}\leq \epsilon^{p}+\frac{\sigma^{p+1}}{D-\epsilon} + \delta(2^{p-1}D^{p}+(2^{p-1}+1)\sigma^{p}).
\end{equation*}
Therefore, if $D:=\epsilon+\frac{\sigma^{p+1}}{\epsilon^{p}}$, $\epsilon<\sigma$, and $\delta\leq(\epsilon/\sigma)^{(p^{2}+p)}$, then $\EBP{\|\Z-\bmu\|^p}\leq 13\cdot\epsilon^{p}$.
\end{lemma}

\begin{proof}
Our proof idea is inspired by \citet{sidford2023quantum} but is more general than their result. 
Let $S$ denote the event that $\|\X_{1} - \bmu\|\leq \epsilon$ and let $T$ denote the event that $\|\X_{1} - \X_{2}\|\leq D$. Then, by the law of total expectation, we have
\begin{align*}
        \EBP{\|\Z-\bmu\|^{p}} &= \Pr(S)\EBP{\|\Z-\bmu\|^{p} \mid S}+ \Pr(\overline{S}) \EBP{\|\Z-\bmu\|^{p} \mid \overline{S}}\\
        &\leq \EBP{\|\Z-\bmu\|^{p} \mid S}+ \delta\EBP{\|\Z-\bmu\|^{p} \mid \overline{S}},
\end{align*}
where we used the fact that $1-\delta\leq\Pr(S)\leq 1$ and $0\leq \Pr(\overline{S})=1-\Pr(S)\leq \delta$. We next prove the lemma by further upper bounding both $\EBP{\|\Z-\bmu\|^{p} \mid S}$ and $\EBP{\|\Z-\bmu\|^{p} \mid \overline{S}}$.

First, we upper bound $\EBP{\|\Z-\bmu\|^{p} \mid S}$. By the law of total expectation and the definition of $\Z$, we have that
\begin{align*}
    \EBP{\|\Z-\bmu\|^{p} \mid S} &= \Pr(T \mid S) \EBP{\|\Z-\bmu\|^{p} \mid T\cap S} + \Pr(\overline{T}\mid S) \EBP{\|\Z-\bmu\|^{p} \mid \overline{T}\cap S}\\
    &\leq \Pr(T\mid S) \epsilon^{p}+ \Pr(\overline{T} \mid S) \sigma^{p}.
\end{align*}
    The second line follows from the facts that
    \begin{itemize}
        \item if $T$ and $S$ both hold, then $\Z=\X_{1}$ and $\|\X_{1}-\bmu\|\leq \epsilon$, indicating that $\EBP{\|\Z-\bmu\|^{p} \mid T\cap S} \leq \epsilon^{p}$;
        \item if $T$ does not hold but $S$ holds, then $\Z=\X_{3}$ and $\X_{3}$ is independent of $S$ and $T$, indicating that $\EBP{\|\Z-\bmu\|^{p} \mid \overline{T}\cap S} =\EBP{\|\X_{3}-\bmu\|^{p}}\leq \sigma^{p}$.
    \end{itemize}
    Now, we proceed to upper bound $\Pr(\overline{T} \mid S)$. Note that, if $\|\X_{2} - \bmu\|\leq D-\epsilon$, then when $S$ holds ($\|\X_{1} - \bmu\|\leq \epsilon$), we have that
    \begin{equation*}
        \|\X_{1} - \X_{2}\| \leq \|\X_{1} - \bmu\| + \|\X_{2} - \bmu\| \leq D,
    \end{equation*}
    by the triangle inequality. Therefore, by Markov inequality to $\X_{2}$ and using that $\X_{2}$ is independent of $S$, we have that
    \begin{equation*}
        \Pr(T\mid S)\geq \Pr\left(\myNorm{\X_{2} - \bmu}\leq D-\epsilon \mid S\right)=\Pr\left(\myNorm{\X_{2} - \bmu}\leq D-\epsilon \right)\geq 1- \frac{\EBP{\|\X_{2}-\bmu\|}}{D-\epsilon}.
    \end{equation*}
    Note that $\EBP{\|\X_{2} - \bmu\|}\leq \left(\EBP{\myNorm{\X_{2}-\bmu}^{p}}\right)^{\frac{1}{p}}\leq \sigma$ by H\"older's inequality and the upper bound of $\EBP{\myNorm{\X_{2}-\bmu}^{p}}$. This implies that
    \begin{equation*}
        \EBP{\|\Z-\bmu\|^{p} \mid S}\leq \epsilon^{p} +\frac{\sigma^{p+1}}{D-\epsilon}.
    \end{equation*}
    Next, we upper bound $\EBP{\|\Z - \bmu\|^{p} \mid \overline{S}}$. Note that if $T$ holds ($\|\X_{1} - \X_{2}\|\leq  D$), then 
    \begin{equation*}
        \|\Z-\bmu\|^{p}\leq 2^{p-1}(\|\Z-\X_{2}\|^{p}+ \|\X_{2} - \bmu\|^{p})=2^{p-1}(\|\X_{1} -\X_{2}\|^{p}+ \|\X_{2} - \bmu\|^{p})\leq 2^{p-1}(D^{p}+ \|\X_{2} - \bmu\|^{p}).
    \end{equation*}
    If $T$ does not hold, then $\Z= \X_{3}$. Therefore, in either case, it holds that
    \begin{equation*}
        \EBP{\|\Z - \bmu\|^{p} \mid \overline{S}}\leq \EBP{2^{p-1} D^{p}+ 2^{p-1} \|\X_{2} - \bmu\|^{p} + \|\X_{3}-\bmu\|^{p}}=2^{p-1} D^{p} + (2^{p - 1} +1)\sigma^{p}.
    \end{equation*}
    Therefore, we have
    \begin{align*}
        \EBP{\myNorm{\Z-\bmu}^{p}}&\leq \EBP{\myNorm{\Z-\bmu}^{p}} + \delta \EBP{\myNorm{\Z-\bmu}^{p}\mid \overline{S}}\\
        &\leq \epsilon^{p} + \frac{\sigma^{p+1}}{D-\epsilon} + \delta (2^{p-1} D^{p} + (2^{p - 1} +1)\sigma^{p}).
    \end{align*}
    Suppose $\epsilon<\sigma$, by choosing $D=\epsilon+\frac{\sigma^{p+1}}{\epsilon^{p}}$ and $\delta = (\epsilon/\sigma)^{p^{2}+p}$, we can compute that
    \begin{align*}
          \EBP{\myNorm{\Z-\bmu}^{p}}
        &\leq \epsilon^{p} + \frac{\sigma^{p+1}}{D-\epsilon} + \delta (2^{p-1} D^{p} + (2^{p - 1} +1)\sigma^{p})\\
        &\leq \epsilon^{p}+\epsilon^{p}+\delta(2^{p-1} 2^{p-1}(\epsilon^{p}+ \frac{\sigma^{p(p+1)}}{\epsilon^{p^{2}}})+(2^{p-1}+1)\sigma^{p})\\
        &\leq \epsilon^{p}+\epsilon^{p}+ 2^{2(p-1)}\frac{\epsilon^{p^{2}+2p}}{\sigma^{p(p+1)}}+2^{2(p-1)}\epsilon^{p}+(2^{p-1}+1)\frac{\epsilon^{p^{2}+p}}{\sigma^{p^{2}}}\\
        &\leq (3+2^{2(p-1)}+2^{2(p-1)}+2^{p-1}) \epsilon^{p}\\
        &\leq 13 \epsilon^{p}.
    \end{align*}
    The second line uses the definition of $D=\epsilon+\frac{\sigma^{p+1}}{\epsilon^{p}}$ and Jensen's inequality. The third line uses the definition of $\delta = \frac{\epsilon^{p^{2}+p}}{\sigma^{(p+1)p}}$. The fourth line uses the fact that $\epsilon<\sigma$. The fifth line uses the fact that $p\in (1,2]$.
\end{proof}

\begin{algorithm}[htpb]
\caption{$\texttt{QHTME}^{+}(\X, \sigma, p, d, \epsilon)$}
\label{alg:QHTME_plus}
\begin{algorithmic}[1]
\STATE \textbf{Input}: random variable $\X$, heavy-tailed parameter $\sigma$, tail index $p$, problem dimension $d$, error $\epsilon\in (0,\sigma)$.
\STATE  $\delta \gets (\epsilon/(13\sigma))^{p(p+1)}$,
$D \gets \epsilon/13 + 13^{p}\sigma^{p+1}/\epsilon^p$.
\STATE Set $\X_1 \gets \texttt{QHTME}(\X, \sigma,p,d, \epsilon/13,\delta)$
\STATE Randomly draw one classical sample $\X_2$ of $\X$
\IF{$\|\X_1 - \X_2\| \le D$}
    \STATE $\tilde{\bmu}\gets \X_1$
\ELSE
    \STATE Randomly draw one classical sample $\X_3$ of $\X$
     \STATE $\tilde{\bmu}\gets \X_3$
\ENDIF
\STATE \textbf{Output}: $\tilde{\bmu}$ 
\end{algorithmic}
\end{algorithm}

We are ready to present a refined quantum mean estimator, denoted $\texttt{QHTME}^{+}$, given in Algorithm \ref{alg:QHTME_plus}. 
This algorithm outputs an estimate for a heavy-tailed random variable whose estimation error has a bounded $p$-th raw moment with probability $1$. 
In $\texttt{QHTME}^{+}$, $\X_1$ is generated by the quantum mean estimator $\texttt{QHTME}$ (Algorithm~\ref{alg: Quantum Heavy-Tailed Mean Estimator (QHTME)}), while $\X_2$ and $\X_3$ are obtained by querying the quantum sampling oracle $O_{\X}$ and measuring the resulting quantum states in the computational basis. 
Since generating $\X_2$ and $\X_{3}$ requires only two additional queries to $O_{\X}$, the total query complexity of $\texttt{QHTME}^{+}$ differs from that of $\texttt{QHTME}$ by at most a constant factor. We summarize the performance guarantee of $\texttt{QHTME}^{+}$ in Lemma \ref{lemma: quantum query complexity of QHTME-plus}.


\begin{lemma}
    \label{lemma: quantum query complexity of QHTME-plus}
    Let $\X$ be a $d$-dimensional random variable with mean $\bmu:=\EBP{\X}$  satisfying $\EBP{\|\X-\bmu\|^{p}}\leq \sigma^{p}$, for some $p\in (1,2]$ and $\sigma\geq 0$. Then, for any $\epsilon\in (0, \sigma)$, Algorithm~\ref{alg:QHTME_plus}, {\rm\texttt{QHTME$^+$}}$(\X, \sigma, p, d, \epsilon)$, uses 
    \begin{equation*}
       \tilde{\OM}\left(\frac{d^{\frac{1}{2}} \sigma^{\frac{p}{2(p-1)}}}{\epsilon^{\frac{p}{2(p-1)}}}\right)
    \end{equation*}
    queries to the quantum sampling oracle $O_{\X}$ and outputs an estimate $\tilde{\bmu}$ of $\bmu$ satisfying $ \EBP{\|\tilde{\bmu} - \bmu\|^{p}} \leq \epsilon^{p}.$

\end{lemma}

\begin{proof}
    Note that the random variable $\X$ satisfies $\EBP{\|\X-\bmu\|^{p}}\leq \sigma^{p}$.
    By using  Lemma \ref{lemma: general pre lemma for unbiased estimator}, and replacing $\epsilon, \sigma$ with $ \frac{\epsilon}{13}$ and $\sigma$, respectively, the output $\tilde{\bmu}$ of $\texttt{QHTME}^{+}$ satisfies
    \begin{equation*}
        \EBP{\|\tilde{\bmu}-\bmu\|^{p}}\leq \epsilon^{p},
    \end{equation*}
    where $\epsilon\in(0,\sigma)$.
    
    As for the query complexity, Algorithm \ref{alg:QHTME_plus} only calls $\texttt{QHTME}(\X,\sigma,p,d, \epsilon/13, \delta)$ once and one can prepare the classical samples $\X_2$ and $\X_{3}$ by querying the quantum sampling oracle $O_{\X}$ once for each and performing the measurements on the computational basis. Since we have that $\EBP{\|\X-\bmu\|^{p}}\leq \sigma^{p}$, by Theorem \ref{thm:QHTME-for prob 1}, the overall query complexity equals 
    \begin{equation*}
\tilde{\mathcal{O}}\left(\frac{d^{\frac{1}{2}} \sigma^{\frac{p}{2(p-1)}}}{\epsilon^{\frac{p}{2(p-1)}}}\right).
    \end{equation*}
\end{proof}

\subsection{Proof of Theorem \ref{theorem: query complexity of QUHTME}}\label{appendix: QUHTME}

\begin{proof}
Our proof mainly follows \citet{sidford2023quantum}. Observe that the output $\tilde{\bmu}$ of Algorithm \ref{alg:QUHTME} can be expressed as
    \begin{equation*}
        \tilde{\bmu} = \tilde{\bmu}_0 + 2^{J} (\tilde{\bmu}_{J} -\tilde{\bmu}_{J-1}), \quad J\sim \operatorname{Geom}\left(\frac{1}{2}\right)\in \mathbb{N}.
    \end{equation*}
    Then, we have an unbiased output satisfying \begin{equation*}
        \EBP{\tilde{\bmu}} = \EBP{\tilde{\bmu}_{0}}+\sum_{j=1}^{\infty}\Pr(J=j) 2^{j} (\EBP{\tilde{\bmu}_{j}}- \EBP{\tilde{\bmu}_{j-1}})=\EBP{\tilde{\bmu}_{\infty}}=\bmu,
    \end{equation*}
    where we used the fact that $\Pr(J=j)=2^{-j}$ in the second equality.  By the triangle inequality and Jensen's inequality ($(a+b)^{p}\leq 2^{p-1}(a^{p}+b^{p}), \forall p\in(1,2]$), we have
    \begin{equation*}
        \EBP{\myNorm{\tilde{\bmu}-\bmu}^{p}}\leq 2^{p-1} \EBP{\myNorm{\tilde{\bmu}-\tilde{\bmu}_{0}}^{p}}+2^{p-1} \EBP{\myNorm{\tilde{\bmu}_{0}- \bmu}^{p}},
    \end{equation*}
    where \begin{equation*}
        \EBP{\myNorm{\tilde{\bmu}-\tilde{\bmu}_{0}}^{p}}= \sum_{j=1}^{\infty} \Pr(J=j)2^{p j} \EBP{\myNorm{\tilde{\bmu}_{j} - \tilde{\bmu}_{j-1}}^{p}} = \sum_{j=1}^{\infty} 2^{(p-1)j} \EBP{\myNorm{\tilde{\bmu}_{j} - \tilde{\bmu}_{j-1}}^{p}}.
    \end{equation*}
    For each $j$, we have
    \begin{equation*}
        \EBP{\myNorm{\tilde{\bmu}_{j} - \tilde{\bmu}_{j-1}}^{p}}\leq 2^{p-1} \EBP{\myNorm{\tilde{\bmu}_{j}- \bmu}^{p}} +2^{p-1} \EBP{\myNorm{\tilde{\bmu}_{j-1}-\bmu}^{p}}.
    \end{equation*}
    By Lemma \ref{lemma: quantum query complexity of QHTME-plus} and line 4 in Algorithm \ref{alg:QUHTME}, we have
    \begin{equation*}
        \EBP{\myNorm{\tilde{\bmu}_{j}-\bmu}^{p}} \leq \frac{\epsilon^{p}}{c \cdot 2^{\frac{3(p-1)j}{2}}}, \forall j\geq 1,
    \end{equation*}
    which indicates that
    \begin{equation*}
         \EBP{\myNorm{\tilde{\bmu}_{j} - \tilde{\bmu}_{j-1}}^{p}}\leq \frac{2^{p-1}\epsilon^{p}}{c\cdot 2^{3(p-1)(j-1)/2}}+\frac{2^{p-1}\epsilon^{p}}{c\cdot 2^{3(p-1)j/2}}\leq \frac{8\epsilon^{p}}{c\cdot 2^{3(p-1)j/2}}.
    \end{equation*}
The second inequality above follows from the facts that $2^{p-1}\leq 2$ for all $p\in(1,2]$ and 
\begin{equation*}
    \frac{2^{p-1}}{c \cdot 2^{3(p-1)(j-1)/2}}\leq \frac{2}{c \cdot 2^{3(p-1)(j-1)/2}}\leq \frac{2\cdot 2^{\frac{3(p-1)}{2}}}{c\cdot 2^{3(p-1)j/2}}\leq  \frac{4\sqrt{2}}{c\cdot 2^{3(p-1)j/2}}\leq \frac{6}{c\cdot 2^{3(p-1)j/2}}.
\end{equation*} Therefore, since we let $c:= 24/(2^{\frac{p-1}{2}}-1)$ in line 3 of Algorithm \ref{alg:QUHTME} and the series $\sum_{j=1}^{\infty}\frac{1}{2^{(p-1)j/2}}$ converges to $\frac{1}{2^{(p-1)/2}-1}$ for any given $p\in (1,2]$, then it holds that
    \begin{equation*}
         \EBP{\myNorm{\tilde{\bmu}-\tilde{\bmu}_{0}}^{p}}\leq \sum_{j=1}^{\infty} 2^{(p-1)j} \frac{8\epsilon^{p}}{c\cdot 2^{3(p-1)j/2}}\leq\sum_{j=1}^{\infty}  \frac{8\epsilon^{p}}{c\cdot 2^{(p-1)j/2}}  \leq \frac{1}{3}\epsilon^{p}.
    \end{equation*}
    Since it holds that $\EBP{\myNorm{\tilde{\bmu}_{0}-\bmu}^{p}}\leq \frac{\epsilon^{p}}{c}=\frac{(2^{\frac{p-1}{2}}-1)\epsilon ^{p}}{24} \leq \frac{\epsilon^{p}}{10}$ by Lemma \ref{lemma: quantum query complexity of QHTME-plus} and line 2 in Algorithm \ref{alg:QUHTME}, then we must have
    \begin{equation*}
         \EBP{\myNorm{\tilde{\bmu}-\bmu}^{p}}\leq 2^{p-1} \EBP{\myNorm{\tilde{\bmu}-\tilde{\bmu}_{0}}^{p}}+2^{p-1} \EBP{\myNorm{\tilde{\bmu}_{0}- \bmu}^{p}}\leq \frac{2^{p-1}}{3}\epsilon^{p}+ \frac{2^{p-1}}{10}\epsilon^{p}\leq \epsilon^{p}.
    \end{equation*}
    Moreover, since it holds that, for a given $p\in (1,2]$
    \begin{equation*}
        \sum_{j=1}^{\infty} \Pr(J=j) 2^{\frac{3j}{4}}=\sum_{j=1}^{\infty}2^{-j} 2^{\frac{3j}{4}}=\sum_{j=1}^{\infty}2^{\frac{-j}{4}}=\frac{1}{2^{\frac{1}{4}}-1}
    \end{equation*}
    and 
    \begin{equation*}
         \sum_{j=1}^{\infty} \Pr(J=j) 2^{\frac{3(j-1)}{4}}=\sum_{j=1}^{\infty}2^{-j} 2^{\frac{3(j-1)}{4}}=\frac{1}{2-2^{3/4}}
    \end{equation*}
    are finite constants,  
    The expected number of queries is 
    \begin{equation*}
\tilde{\mathcal{O}}\left(\frac{d^{\frac{1}{2}} \sigma^{\frac{p}{2(p-1)}}}{\epsilon^{\frac{p}{2(p-1)}}}\right) \cdot \left(1+ \sum_{j=1}^{\infty} \Pr(J=j) \cdot \left(2^{\frac{3j}{4}}+ 2^{\frac{3(j-1)}{4}}\right)\right)=  \tilde{\mathcal{O}}\left(\frac{d^{\frac{1}{2}} \sigma^{\frac{p}{2(p-1)}}}{\epsilon^{\frac{p}{2(p-1)}}}\right).
    \end{equation*}
\end{proof}

\subsection{Lower Bounds on Mean Estimation of Heavy-Tailed Random Variables}\label{appendix: lower bound of estimating the mean}

\begin{theorem}[Bounded-error quantum query lower bound for heavy-tailed mean estimation]\label{thm: lower bounds 1 for heavy-tailed mean estimation}
Fix $p\in(1,2]$ and $\sigma>0$. Consider the problem of estimating the mean
$\bmu=\EBP{\X}$ of a d-dimensional real-valued random variable $\X$ satisfying $$\EBP{\Norm{\X-\bmu}^p} \leq \sigma^p .$$
For any estimator which, with probability at least $2/3$, outputs a mean estimate
$\tilde{\bmu}$ satisfying $$\Norm{\tilde{\bmu}-\bmu}\leq \epsilon ,$$ the query complexity is 
$$\Omega\!\left( \left(\frac{\sigma}{\epsilon}\right)^{\frac{p}{2(p-1)}} \right)$$ by using the quantum sampling oracle.
\end{theorem}

\begin{proof}
It suffices to prove the result for the case $d=1$, since the general $d$-dimensional statement follows by considering random vectors supported entirely on the first coordinate. In the one-dimensional setting, our argument is a direct adaptation of the proof technique in \cite[Section 6]{blanchet2024quadraticspeedupinfinitevariance}, replacing the bounded $p$-th raw moment assumption with a bounded $p$-th central moment assumption. For completeness, we include the full proof here.

We consider reducing our mean estimation problem to the unstructured search problem. Let $N\geq 2$ be an integer to be chosen later, and consider Boolean functions
$f:[N]:=\{1,\ldots,N\}\to\{0,1\}$ under the promise that either
\[
    f(i)=0 \quad \text{for all } i\in[N],
\]
or there exists a unique $j\in[N]$ such that
\[
    f(j)=1, \qquad f(i)=0 \ \text{for all } i\neq j .
\]
As established in \cite[Theorem 5]{blanchet2024quadraticspeedupinfinitevariance}, with probability at least $2/3$,  distinguishing these two cases requires
$\Omega(N)$ classical queries \cite{classicallowerboundforsearch} and $\Omega(\sqrt N)$ quantum queries \cite{Bennett_1997}.

Given such an $f$, define a random variable
\[
    \X_f := M f(U),
\]
where $U$ is uniformly distributed on $[N]$ and $M>0$ is a parameter. One query to a classical sampling oracle for $\X_f$ can be simulated by randomly choose a number in $[N]$ and then using one query to $f$; similarly, in the quantum setting, one query to the corresponding quantum sampling oracle $O_{\X_{f}}$ satisfying
$$O_{\X_{f}} \ket{0}\ket{0}= \sum_{i=1}^{N}\frac{1}{\sqrt{N}} \ket{M f(i)}\ket{i}$$
can be simulated by preparing a uniform superposition by Hadamard gates and one query to the binary oracle $\mathcal{B}_{f}$ for $f$ satisfying $$\mathcal{B}_{f}: \ket{i}\ket{0}\rightarrow \ket{i}\ket{f(i)}.$$ Hence, any mean-estimation algorithm for the class under consideration yields a search-distinguishing algorithm with the same order of query complexity.

If $f\equiv 0$, then $\X_f=0$ almost surely, so
\[
    \bmu_f:=\EBP{\X}_f =0,
    \qquad
    \EBP{ |\X_f-\bmu_f|^p}=0 .
\]
If $f$ has a unique marked element, then
\[
    \mathbb P(\X_f=M)=\frac1N,
    \qquad
    \mathbb P(\X_f=0)=1-\frac1N,
\]
and therefore
\[
    \bmu_f=\EBP{\X_f}=\frac{M}{N}.
\]
Choose
\[
    M:=3\epsilon N .
\]
Then in the marked case,
\[
    \bmu_f=3\epsilon ,
\]
whereas in the unmarked case $\bmu_f=0$.

We now verify the moment constraint. In the marked case,
\[
\begin{aligned}
    \EBP{ |\X_f-\bmu_f|^p}
    &=
    \frac1N \left|M-\frac{M}{N}\right|^p
    +
    \left(1-\frac1N\right)\left|\frac{M}{N}\right|^p  \\
    &\leq
    \frac{M^p}{N}
    +
    \frac{M^p}{N^p}
    \leq
    \frac{2M^p}{N}.
\end{aligned}
\]
Substituting $M=3\epsilon N$ gives
\[
    \EBP{ |\X_f-\bmu_f|^p}
    \leq
    2\cdot 3^p \epsilon^p N^{p-1}.
\]
Thus the condition $\EBP{|\X_f-\bmu_f|^p}\leq \sigma^p$ holds whenever
\[
    N \leq c_p \left(\frac{\sigma}{\epsilon}\right)^{\frac{p}{p-1}},
\]
for a constant $c_p>0$ depending only on $p$.

Now suppose there exists a mean-estimation algorithm which, for every distribution satisfying the above moment condition, outputs $\tilde{\bmu}$ with
$|\tilde{\bmu}-\bmu_f|\leq \epsilon$ with probability at least $2/3$ using $T$ queries. We use it to distinguish the two search cases as follows: run the estimator on $\X_f$ and output ``marked'' if
\[
    |\tilde{\bmu}|\geq \frac{3\epsilon}{2},
\]
and output ``unmarked'' otherwise.

If $f\equiv 0$, then $\bmu_f=0$, so on the success event,
\[
    |\tilde\bmu|\leq \epsilon < \frac{3\epsilon}{2}.
\]
If $f$ has a unique marked element, then $\bmu_f=3\epsilon$, so on the success event,
\[
    |\tilde\bmu|
    \geq
    3\epsilon-\epsilon
    =
    2\epsilon
    >
    \frac{3\epsilon}{2}.
\]
Therefore, the induced algorithm distinguishes the two promised search cases with probability at least $2/3$ using $T$ queries.

By Grover's lower bound \cite{Bennett_1997}, it implies
\[
    T=\Omega(\sqrt N)
    =
    \Omega\!\left(
    \left(\frac{\sigma}{\epsilon}\right)^{\frac{p}{2(p-1)}}
    \right)
\]
in the quantum sampling oracle. This proves the claim.
\end{proof}

\begin{corollary}[Expected-query quantum lower bound for heavy-tailed mean estimation under $L_p$ error]
\label{coro: expected query lower bound for heavy-tailed mean estimation with p-th moment error}
Fix $p\in(1,2]$ and $\sigma>0$. Let
$\mathcal P_\sigma$ denote the class of all $d$-dimensional random vectors
$\X$ with mean $\bmu=\EBP{\X}$ satisfying
\[
    \EBP{\|\X-\bmu\|^p}\leq\sigma^p.
\]
Suppose a quantum algorithm, for every $\X\in\mathcal P_\sigma$, outputs
an estimator $\widetilde{\bmu}$ satisfying
\[
    \EBP{\|\widetilde{\bmu}-\bmu\|^p}\leq\epsilon^p.
\]
Then its expected query complexity is at least
\[
    \Omega\!\left(
        \left(\frac{\sigma}{\epsilon}\right)^{\frac{p}{2(p-1)}}
    \right)
\]
with the quantum sampling oracle.
\end{corollary}

\begin{proof}


Let $\mathcal A$ be an algorithm satisfying the stated guarantee and $\EBP{Q}$
denotes the worst-case expected number of oracle queries made by $\mathcal A$. 
Set $T:=\EBP{Q}$. By Markov's inequality,
\[
    \Pr(Q>6T)\leq \frac16 .
\]
On the other hand, from
\[
    \EBP{\Norm{\tilde\bmu-\bmu}^{p}}\leq \epsilon^p ,
\]
Markov's inequality gives
\[
    \Pr\!\left(
        \Norm{\tilde\bmu-\bmu}>6^{1/p}\epsilon
    \right)
    \leq
    \frac{1}{6}.
\]
Therefore,
\[
    \Pr\!\left(
        Q\leq 6T
        \ \text{and}\
        \Norm{\tilde\bmu-\bmu}\leq 6^{1/p}\epsilon
    \right)
    \geq
    1-\frac16-\frac16
    =
    \frac23 .
\]

Now define a truncated algorithm $\mathcal A'$ as follows: run $\mathcal A$,
but stop it once it has made more than $6T$ oracle queries. If this happens,
output an arbitrary value. Then $\mathcal A'$ makes at most $6T$ queries in
the worst case. Moreover, by the inequality above, $\mathcal A'$ outputs an
estimate satisfying
\[
    \Norm{\tilde\bmu-\bmu}\leq 6^{1/p}\epsilon
\]
with probability at least $2/3$.

Thus, $\mathcal A'$ is a valid high-probability mean estimator with accuracy
$6^{1/p}\epsilon$ and worst-case query complexity at most $6T$. By the
high-probability quantum lower bound in Theorem \ref{thm: lower bounds 1 for heavy-tailed mean estimation},
\[
    6T
    \geq
    \Omega\!\left(
        \left(\frac{\sigma}{6^{1/p}\epsilon}\right)^{\frac{p}{2(p-1)}}
    \right).
\]
Since $6^{1/p}$ is a constant depending only on $p$, this implies
\[
    T
    \geq
    \Omega\!\left(
        \left(\frac{\sigma}{\epsilon}\right)^{\frac{p}{2(p-1)}}
    \right).
\]

\end{proof}

\begin{corollary}[Expected-query quantum lower bound for unbiased heavy-tailed mean estimation]
\label{coro:expected-query-lower-bound-for-unbaised-heavy-tailed-mean-estimation}
Fix $p\in(1,2]$, $\sigma>0$, and $\epsilon>0$. Let
$\mathcal P_\sigma$ denote the class of all $d$-dimensional random vectors
$\X$ with mean $\bmu=\EBP{\X}$ satisfying
\[
    \EBP{\|\X-\bmu\|^p}\leq\sigma^p.
\]
Suppose a quantum algorithm, for every $\X\in\mathcal P_\sigma$, outputs
an estimator $\widetilde{\bmu}$ satisfying
\[
    \EBP{\widetilde{\bmu}}=\bmu
    \qquad\text{and}\qquad
    \EBP{\|\widetilde{\bmu}-\bmu\|^p}\leq\epsilon^p.
\]
Then its worst-case expected number of queries to the quantum sampling
oracle is at least
\[
    \Omega\!\left(
        \left(\frac{\sigma}{\epsilon}\right)^{
            \frac{p}{2(p-1)}
        }
    \right).
\]
\end{corollary}

\begin{proof}
Any unbiased estimator satisfying
\[
    \EBP{\|\widetilde{\bmu}-\bmu\|^p}\leq\epsilon^p
\]
is, in particular, an estimator satisfying the assumptions of
Corollary~\ref{coro: expected query lower bound for heavy-tailed mean estimation with p-th moment error}.
The claimed lower bound therefore follows immediately from that corollary.
\end{proof}

Here, we establish $d$-dependent quantum lower bounds 
for solving Problems~\ref{prob: obtain mean estimate for heavy tail random variable with bounded mean}  and \ref{prob: obtain unbiased mean estimate for heavy tail random variable}, which are stronger than $\Omega\left(\left(\frac{\sigma}{\epsilon}\right)^{\frac{p}{2(p-1)}}\right)$ when the heavy-tailed index satisfies $p\in (4/3,2]$. 
Our lower bounds are obtained via a reduction to a $\texttt{Recovery}^{d}\circ \texttt{Search}^{M}$ problem.
Before introducing $\texttt{Recovery}^{d}\circ \texttt{Search}^{M}$, we first consider the simpler $\texttt{PairParity}^{d}\circ\texttt{Search}^{M}$ problem. We will see that the quantum lower bound for $\texttt{Recovery}^{d}\circ \texttt{Search}^{M}$ can be obtained via a reduction to $\texttt{PairParity}^{d}\circ\texttt{Search}^{M}$.

\begin{definition}[$\texttt{PairParity}^{d}\circ\texttt{Search}^{K}$ problem]
Let $d,K\in\mathbb N$ with $d$ even. An input is a matrix
\[
    A\in\{0,1\}^{d\times K}
\]
satisfying the row-wise promise
\[
    \Norm{A_{i,*}}_{1} :=\sum_{j=1}^{K} A_{i,j}\in\{0,1\},
    \qquad \forall i\in[d],
\]
and the pairwise-balance promise
\[
    \mathbf 1\{\Norm{A_{2r-1,*}}_{1}=1\}
    +
    \mathbf 1\{\Norm{A_{2r,*}}_{1}=1\}
    =1,
    \qquad \forall r\in[d/2].
\]
Define the hidden vector $b^{(A)}\in\{0,1\}^d$ by
\[
    b_i^{(A)}:=\mathbf 1\{\Norm{A_{i,*}}_{1}=1\}.
\]
Equivalently, the pairwise-balance promise requires
\[
    b_{2r-1}^{(A)}+b_{2r}^{(A)}=1,
    \qquad \forall r\in[d/2].
\]
The algorithm is given query access to the following oracle
\[
    O_A:\ket{i,j,c}\mapsto (-1)^{cA_{i,j}}\ket{i,j,c},
    \qquad (i,j)\in[d]\times[K],
\]
where $c\in \{0,1\}$.

The goal is to output
\[
    \texttt{PairParity}^{d}\circ\texttt{Search}^{K}(A)
    :=
    \bigoplus_{r=1}^{d/2} b_{2r-1}^{(A)}.
\]
\end{definition}

\begin{theorem}[Quantum lower bound for
$\texttt{PairParity}^{d}\circ\texttt{Search}^{K}$]
\label{thm:pair-parity-search-lower-bound}
Let $d,K\in\mathbb N$ with $d$ even, and let
$\rho\in(1/2,1)$ be any fixed constant. Any quantum algorithm that solves
$\texttt{PairParity}^{d}\circ\texttt{Search}^{K}$ with success
probability at least $\rho$ must make
\[
    \Omega_{\rho}(d\sqrt K)
\]
queries to the oracle
\[
    O_A:\ket{i,j,c}\mapsto(-1)^{c A_{i,j}}\ket{i,j,c},
\]
where the implicit constant in \(\Omega_\rho(\cdot)\) may depend on \(\rho\).
\end{theorem}

\begin{proof}
Let
\[
    m:=\frac d2.
\]
We first express $\texttt{PairParity}^{d}\circ\texttt{Search}^{K}$ as a
composition of two partial Boolean functions.

Define the inner promise function
\[
    g:D_g\to\{0,1\},
    \qquad
    D_g:=\{x\in\{0,1\}^{K}:\Norm{x}_1\in\{0,1\}\},
\]
by
\[
    g(x):=\mathrm{OR}_K(x).
\]
Thus $g(0^K)=0$ and $g(e_j)=1$ for every $j\in[K]$.

Define the outer promise function
\[
    f:D_f\to\{0,1\},
\]
where
\[
    D_f:=
    \left\{
        b\in\{0,1\}^{d}:
        b_{2r-1}+b_{2r}=1,\ \forall r\in[m]
    \right\},
\]
by
\[
    f(b):=\bigoplus_{r=1}^{m} b_{2r-1}.
\]
Then $\texttt{PairParity}^{d}\circ\texttt{Search}^{K}$ is precisely the
promise composition
\[
    F=f\circ g^d,
\]
where the input is
\[
    A=(A_{1,*},\ldots,A_{d,*}),
\]
each row $A_{i,*}$ lies in $D_g$, and
\[
    (g(A_{1,*}),\ldots,g(A_{d,*}))\in D_f.
\]

Note that the oracle \(O_A\) is unitarily equivalent to the standard binary
oracle
\[
    B_A:\ket{i,j,c}
    \mapsto
    \ket{i,j,c\oplus A_{i,j}}
\]
via Hadamard gates on the third  register:
\[
    B_A
    =
    (I_{i,j}\otimes H_c)\,
    O_A\,
    (I_{i,j}\otimes H_c).
\]
Hence the standard adversary lower bounds in the Boolean query model
apply to the present oracle model.

We now lower bound the adversary values of the inner and outer functions.

First consider the inner function $g$. We use Ambainis' relation method \cite[Theorem 2]{ambainis2000quantum}.
Let
\[
    X_g:=\{0^K\},
    \qquad
    Y_g:=\{e_1,\ldots,e_K\}.
\]
For every $x\in X_g$ and $y\in Y_g$, we have $g(x)=0$ and $g(y)=1$.
Define
\[
    R_g:=X_g\times Y_g.
\]
Then every $x\in X_g$ is related to $K$ inputs in $Y_g$, and every
$y\in Y_g$ is related to one input in $X_g$. In Ambainis' notation,
\[
    m_g=K,
    \qquad
    m_g'=1.
\]
Moreover, for a fixed $x=0^K$ and coordinate $j\in[K]$, there is at most
one $y\in Y_g$ such that $(x,y)\in R_g$ and $x_j\neq y_j$, namely
$y=e_j$. Similarly, for a fixed $y=e_t$ and coordinate $j$, there is at
most one $x\in X_g$ such that $(x,y)\in R_g$ and $x_j\neq y_j$. Therefore
\[
    \ell_g=1,
    \qquad
    \ell_g'=1.
\]
The result in  \cite[Theorem 2]{ambainis2000quantum} gives
\[
    Q(g)
    =
    \Omega\left(
        \sqrt{\frac{m_gm_g'}{\ell_g\ell_g'}}
    \right)
    =
    \Omega(\sqrt K).
\]
Equivalently, the same relation defines a feasible positive adversary
matrix, and hence
\[
    \mathrm{ADV}^{+}(g)\ge \Omega(\sqrt K).
\]
Since the general adversary bound allows negative weights and hence dominates the positive adversary bound,
\[
    \mathrm{ADV}^{\pm}(g)\ge \mathrm{ADV}^{+}(g)=\Omega(\sqrt K).
\]

Next consider the outer function $f$. Although $f$ is defined on $d$ input
bits, the pairwise promise leaves only $m=d/2$ free bits. To apply
Ambainis' method, identify each
\[
    c=(c_1,\ldots,c_m)\in\{0,1\}^{m}
\]
with
\[
    \phi(c)
    :=
    (c_1,1-c_1,\ c_2,1-c_2,\ \ldots,\ c_m,1-c_m)
    \in D_f.
\]
Then
\[
    f(\phi(c))
    =
    c_1\oplus\cdots\oplus c_m
    =
    \mathrm{PARITY}_m(c).
\]

Define
\[
    X_f:=\{\phi(c):\mathrm{PARITY}_m(c)=0\},
    \qquad
    Y_f:=\{\phi(c):\mathrm{PARITY}_m(c)=1\}.
\]
For $x=\phi(c)\in X_f$ and $y=\phi(c')\in Y_f$, define
\[
    (x,y)\in R_f
    \quad\Longleftrightarrow\quad
    c'=c\oplus e_r
    \text{ for some } r\in[m].
\]
Thus $x$ and $y$ differ exactly on the pair of coordinates
\[
    2r-1,\ 2r.
\]
Every $x\in X_f$ is related to exactly $m$ inputs in $Y_f$, one for each
choice of $r\in[m]$. Similarly, every $y\in Y_f$ is related to exactly
$m$ inputs in $X_f$. Hence
\[
    m_f=m,
    \qquad
    m_f'=m.
\]
Now fix any $x\in X_f$ and any coordinate $i\in[d]$. The coordinate $i$
belongs to a unique pair $(2r-1,2r)$. Among all $y$ with $(x,y)\in R_f$,
only the one obtained by flipping the corresponding free bit $c_r$ differs
from $x$ at coordinate $i$. Therefore
\[
    \ell_f=1.
\]
The same argument gives
\[
    \ell_f'=1.
\]
By Ambainis' theorem~\cite[Theorem 2]{ambainis2000quantum},
\[
    Q(f)
    =
    \Omega\left(
        \sqrt{\frac{m_fm_f'}{\ell_f\ell_f'}}
    \right)
    =
    \Omega(m)
    =
    \Omega\left(\frac{d}{2}\right)=\Omega(d).
\]
Equivalently, this Ambainis relation gives
\[
    \mathrm{ADV}^{+}(f)\ge \Omega(d),
\]
and hence
\[
    \mathrm{ADV}^{\pm}(f)\ge \Omega(d).
\]

Finally, we use the perfect composition property of the general adversary
bound for partial Boolean functions \cite[Lemma A.3]{Kimmel_2013}:
\[
    \mathrm{ADV}^{\pm}(f\circ g^d)
    =
    \mathrm{ADV}^{\pm}(f)\,\mathrm{ADV}^{\pm}(g).
\]
Therefore
\[
    \mathrm{ADV}^{\pm}(F)
    =
    \mathrm{ADV}^{\pm}(f\circ g^d)
    \ge
    \Omega(d)\cdot \Omega(\sqrt K)
    =
    \Omega(d\sqrt K).
\]
Since the general adversary bound characterizes bounded-error quantum
query complexity up to constant factors~\cite[Theorem 1.3]{reichardt2011reflections},
\[
    Q(F)=\Omega(\mathrm{ADV}^{\pm}(F))
    =
    \Omega(d\sqrt K).
\]
The preceding adversary argument establishes the claimed lower bound
for success probability at least $2/3$. We finally explain why the same
asymptotic lower bound holds for every fixed
$\rho\in(1/2,1)$.

If $\rho\geq2/3$, the result follows immediately from the lower bound
for success probability $2/3$. Now suppose that
$\rho\in(1/2,2/3)$. Let $\mathcal A$ be an algorithm that succeeds with
probability at least $\rho$ using $T$ queries. Run $\mathcal A$
independently $s$ times and output the majority of its answers, where
$s$ is an odd integer. By Hoeffding's inequality, the failure
probability of the amplified algorithm is at most
\[
    \exp\left(
        -2s\left(\rho-\frac12\right)^2
    \right).
\]
Choosing
\[
    s
    \geq
    \frac{\ln 3}{
        2\left(\rho-\frac12\right)^2
    }
\]
makes this failure probability at most $1/3$. Since $\rho$ is fixed,
$s$ is a constant independent of $d$ and $K$.

The amplified algorithm therefore succeeds with probability at least
$2/3$ using $sT$ queries. Applying the already established
$2/3$-success lower bound gives
\[
    sT=\Omega(d\sqrt K).
\]
Since $s$ depends only on $\rho$, we obtain
\[
    T=\Omega_{\rho}(d\sqrt K).
\]
This proves the result for every fixed $\rho>1/2$.

\end{proof}


\begin{definition}[$\texttt{Recovery}^{d}\circ\texttt{Search}^{K}$ problem]
\label{def:recovery-search}
Let $d,K\in\mathbb N$ with $d$ even. Define the set of valid inputs by
\[
    \mathcal D_{d,K}
    :=
    \left\{
        A\in\{0,1\}^{d\times K}:
        \|A_{i,*}\|_1\in\{0,1\}\ \text{for all }i\in[d],
        \quad
        \sum_{i=1}^d
        \mathbf 1\{\|A_{i,*}\|_1=1\}
        =\frac d2
    \right\}.
\]
For each $A\in\mathcal D_{d,K}$, define the hidden vector
$b^{(A)}\in\{0,1\}^d$ by
\[
    b_i^{(A)}
    :=
    \mathbf 1\{\|A_{i,*}\|_1=1\}.
\]
The algorithm is given query access to the following oracle
\[
    O_A:\ket{i,j,c}\mapsto(-1)^{cA_{i,j}}\ket{i,j,c},
    \qquad (i,j)\in[d]\times[K], c\in\{0,1\}.
\]
The goal is to recover, or approximately recover, the vector $b^{(A)}$.
\end{definition}

\begin{theorem}[Quantum lower bound for exact recovery in $\texttt{Recovery}^{d}\circ\texttt{Search}^{K}$]
\label{thm:pair-parity-reduces-to-recovery}
Let $d,K\in\mathbb N$ with $d$ even, and let
$\rho\in(1/2,1)$ be a constant. Any quantum algorithm that
outputs the vector $b^{(A)}$ for every valid instance of
$\texttt{Recovery}^{d}\circ\texttt{Search}^{K}$ with success probability
at least $\rho$ must make
\[
    \Omega_{\rho}(d\sqrt K)
\]
queries to the oracle
\[
    O_A:\ket{i,j,c}\mapsto(-1)^{cA_{i,j}}\ket{i,j,c}.
\]
\end{theorem}
\begin{proof}
Let $\mathcal A$ be a quantum algorithm that outputs $b^{(A)}$ for every valid instance of
$\texttt{Recovery}^{d}\circ\texttt{Search}^{K}$ with success probability
at least $\rho$, using $T$ queries. We construct an algorithm
$\mathcal B$ for
$\texttt{PairParity}^{d}\circ\texttt{Search}^{K}$ using the same number
of queries.

Recall that a valid input of
$\texttt{PairParity}^{d}\circ\texttt{Search}^{K}$ satisfies
\[
    b_{2r-1}^{(A)}+b_{2r}^{(A)}=1,
    \qquad
    r\in[d/2].
\]
Consequently,
\[
    \sum_{i=1}^{d}b_i^{(A)}
    =
    \sum_{r=1}^{d/2}
    \left(
        b_{2r-1}^{(A)}+b_{2r}^{(A)}
    \right)
    =
    \frac d2.
\]
Thus every valid PairParity instance is also a valid instance of
$\texttt{Recovery}^{d}\circ\texttt{Search}^{K}$.

On input $A$, algorithm $\mathcal B$ runs $\mathcal A$ with the same
oracle $O_A$. If $\mathcal A$ outputs $\widehat b\in\{0,1\}^d$, then
$\mathcal B$ outputs
\[
    \bigoplus_{r=1}^{d/2}\widehat b_{2r-1}.
\]
Whenever $\widehat b=b^{(A)}$, this output equals
\[
    \bigoplus_{r=1}^{d/2}b_{2r-1}^{(A)}
    =
    \texttt{PairParity}^{d}\circ\texttt{Search}^{K}(A).
\]
Therefore, $\mathcal B$ succeeds whenever $\mathcal A$ succeeds. It has
success probability at least $\rho$ and uses exactly $T$ queries. By Theorem~\ref{thm:pair-parity-search-lower-bound},
\[
    T=\Omega_{\rho}(d\sqrt K).
\]
\end{proof}

\begin{lemma}[$\ell_1$ approximate recovery lower bound]
\label{lem:l1-approx-row-rare-search}
There exist universal constants $\alpha_0>1$ and $c_0>0$ such that the
following holds. For every even $d\in\mathbb N$, every $K\in\mathbb N$,
and every quantum algorithm $\mathcal A$ that, on every instance
$A\in\mathcal D_{d,K}$ of
$\texttt{Recovery}^{d}\circ\texttt{Search}^{K}$, makes at most
\[
    \frac{d\sqrt K}{\alpha_0}
\]
queries to $O_A$ and outputs a vector
$\widetilde b\in\mathbb R^d$, there exists an input
$A\in\mathcal D_{d,K}$ such that
\[
    \Pr\left[
        \|\widetilde b-b^{(A)}\|_1
        \geq c_0d
    \right]
    \geq\frac13.
\]
\end{lemma}

\begin{proof}

By Theorem \ref{thm:pair-parity-reduces-to-recovery}, we know that there exists a universal constant \(\beta>0\) such that any
quantum algorithm that exactly recovers \(b^{(A)}\) with success probability at
least \(5/9\) requires at least
\[
    \beta d\sqrt K
\]
queries to \(O_A\).  

We now prove the approximate lower bound by contradiction.  Suppose that for
some \(\delta>0\) there exists a quantum algorithm \(\mathcal A\) making at most
\[
    \frac{d\sqrt K}{\alpha_0}
\]
queries to \(O_A\) such that, for every valid input \(A\),
\[
    \Pr\!\left[
        \|\widetilde b-b^{(A)}\|_1\le \delta d
    \right]\ge \frac23 .
\]
We will show that, for suitable universal constants \(\alpha_0\) and \(\delta\),
this gives an exact-recovery algorithm using fewer than \(\beta d\sqrt K\)
queries, contradicting the exact-recovery lower bound.

Condition on the event
\[
    \|\widetilde b-b^{(A)}\|_1\le \delta d .
\]
Define the thresholded vector \(\widehat b\in\{0,1\}^d\) by
\[
    \widehat b_i:=\mathbf 1\{\widetilde b_i\ge 1/2\}.
\]
Let
\[
    E:=\{i\in[d]:\widehat b_i\ne b_i^{(A)}\},
    \qquad
    r:=|E|.
\]
Whenever \(i\in E\), we have
\[
    |\widetilde b_i-b_i^{(A)}|\ge \frac12 .
\]
Therefore,
\[
    \frac r2
    \le
    \|\widetilde b-b^{(A)}\|_1
    \le
    \delta d,
\]
and hence
\[
    r\le 2\delta d .
\]

It remains to locate and correct the indices in \(E\).  Define the mistake
predicate
\[
    P(i)=1
    \quad\Longleftrightarrow\quad
    \widehat b_i\ne b_i^{(A)} .
\]
We next show that the corresponding phase oracle
\[
    O_P:\ket i\mapsto (-1)^{P(i)}\ket i
\]
can be implemented using \(O(\sqrt K)\) queries to \(O_A\).

For each row \(i\), define
\[
    \chi_i(j):=A_{i,j},\qquad j\in[K].
\]
Under the row-wise promise, \(\chi_i\) has either no marked element or a unique
marked element.  In the latter case, the uniform initial state
\[
    \ket{s}:=\frac1{\sqrt K}\sum_{j=1}^K\ket j
\]
has success probability
\[
    a=\frac1K,
\]
which is known.

We use the exact amplitude amplification procedure of
\cite[Section~2.1, Theorem~4]{brassard2002quantum}.  Let
\[
    \theta:=\arcsin(1/\sqrt K),\qquad
    q:=\left\lceil \frac{\pi}{4\theta}-\frac12\right\rceil,
\]
and set
\[
    \theta':=\frac{\pi}{4q+2},
    \qquad
    a':=\sin^2\theta' .
\]
Then \(a'\le a\).  Introduce a one-qubit ancilla and let \(B\) be the rotation
\[
    B\ket0
    =
    \sqrt{1-\frac{a'}a}\ket0
    +
    \sqrt{\frac{a'}a}\ket1 .
\]
Let \(U_K\ket0=\ket{s}\), and define
\[
    \mathcal U:=U_K\otimes B .
\]
The good predicate for the amplified search is
\[
    \chi_A'(i,j,c):=c\,A_{i,j}.
\]
The associated phase oracle is
\[
    S_{\chi_A'}\ket{i,j,c}
    =
    (-1)^{cA_{i,j}}\ket{i,j,c}.
\]
By Definition~\ref{def:recovery-search}, this is precisely the oracle
\(O_A\). Hence each application of \(S_{\chi_A'}\) uses one query to
\(O_A\).

Let \(S_0\) denote the phase flip about the all-zero state of the \((j,c)\)
registers, and define the amplitude amplification iterate
\[
    Q_A:=
    -\mathcal U S_0\mathcal U^\dagger S_{\chi_A'} ,
    \qquad
    G_A:=Q_A^q\mathcal U .
\]
The row index \(i\) is not measured.  The same circuit acts coherently on
superpositions of row indices.  Equivalently,
\[
    G_A
    =
    \sum_{i=1}^d
    \ket i\!\bra i\otimes G_i ,
\]
where \(G_i\) is the exact amplitude amplification circuit for the \(i\)-th row.


Let $\mathcal W$ denote all the work registers used by the exact
amplitude-amplification unitary $G_A$, including the registers $j$ and
$c$. By construction, $G_A$ uses $q=O(\sqrt K)$ queries to $O_A$ and
satisfies
\begin{equation} \label{eq:exact-row-detection}
        S_{\chi_A'}G_A
    \ket{i}\ket{0}_{\mathcal W}
    =
    (-1)^{b_i^{(A)}}
    G_A\ket{i}\ket{0}_{\mathcal W},
    \qquad i\in[d].
\end{equation}


To obtain a controlled phase oracle for $b_i^{(A)}$, introduce an
additional control qubit $z\in\{0,1\}$. Define
\[
    S_{\chi_A'}^{(z)}
    \ket{i,j,c,z}
    :=
    (-1)^{zcA_{i,j}}
    \ket{i,j,c,z}.
\]
This unitary can be implemented using one query to $O_A$. Indeed,
introduce an auxiliary qubit $t$ initialized to $\ket{0}$ and compute
$t=zc$ using a Toffoli gate. Applying $O_A$ with $t$ as its response
register, and then uncomputing $t$, gives
\[
\begin{aligned}
    \ket{i,j,c,z,0}_t
    &\longmapsto
    \ket{i,j,c,z,zc}_t                                      \\
    &\xmapsto{\,O_A^{(i,j,t)}\,}
    (-1)^{zcA_{i,j}}
    \ket{i,j,c,z,zc}_t                                      \\
    &\longmapsto
    (-1)^{zcA_{i,j}}
    \ket{i,j,c,z,0}_t.
\end{aligned}
\]
Thus, the auxiliary qubit $t$ is returned to $\ket{0}$ and
$S_{\chi_A'}^{(z)}$ uses exactly one query to $O_A$.

Now define
\[
    \widetilde O_b
    :=
    \bigl(G_A^\dagger\otimes I_z\bigr)
    S_{\chi_A'}^{(z)}
    \bigl(G_A\otimes I_z\bigr).
\]
Equation~\eqref{eq:exact-row-detection} implies that
\[
    S_{\chi_A'}^{(z)}
    G_A\ket{i}\ket{0}_{\mathcal W}\ket z
    =
    (-1)^{zb_i^{(A)}}
    G_A\ket{i}\ket{0}_{\mathcal W}\ket z.
\]
Consequently,
\[
    \widetilde O_b
    \ket{i,z}\ket{0}_{\mathcal W}
    =
    (-1)^{zb_i^{(A)}}
    \ket{i,z}\ket{0}_{\mathcal W}.
\]

Let $\widehat b\in\{0,1\}^d$ be the rounded estimate obtained above, and
define the error predicate
\[
    P(i)
    :=
    \widehat b_i\oplus b_i^{(A)}.
\]
Since $\widehat b$ is known classically, the unitary
\[
    D_{\widehat b}^{(z)}
    \ket{i,z}
    :=
    (-1)^{z\widehat b_i}\ket{i,z}
\]
requires no queries to $O_A$. Define
\[
    \widetilde O_P
    :=
    \bigl(D_{\widehat b}^{(z)}\otimes I_{\mathcal W}\bigr)
    \widetilde O_b.
\]
Then
\begin{align*}
    \widetilde O_P
    \ket{i,z}\ket{0}_{\mathcal W}
    &=
    (-1)^{z(\widehat b_i+b_i^{(A)})}
    \ket{i,z}\ket{0}_{\mathcal W} \\
    &=
    (-1)^{zP(i)}
    \ket{i,z}\ket{0}_{\mathcal W}.
\end{align*}
Hence $\widetilde O_P$ is the controlled phase-query oracle for the
error predicate $P$.

Finally, applying Hadamard gates to the response qubit gives the standard
bit-query oracle
\[
    B_P
    :=
    \bigl(I_i\otimes H_z\otimes I_{\mathcal W}\bigr)
    \widetilde O_P
    \bigl(I_i\otimes H_z\otimes I_{\mathcal W}\bigr),
\]
which satisfies
\[
    B_P
    \ket{i,z}\ket{0}_{\mathcal W}
    =
    \ket{i,z\oplus P(i)}\ket{0}_{\mathcal W}.
\]
Each application of $B_P$ uses
\[
    q+1+q
    =
    2q+1
    =
    O(\sqrt K)
\]
queries to $O_A$.

We now use quantum enumeration over the index set \([d]\) to find all marked
indices of \(P\). Let
\[
    R:=\lceil 2\delta d\rceil .
\]
On the conditioned event, \(r=|E|\le R\). Grover enumeration with an
upper bound \(R\) on the number of marked elements finds all marked indices
with success probability at least \(5/6\) using
\[
    O(\sqrt{dR})
\]
queries to \(B_P\).

If \(2\delta d<1\), then \(r=0\), since \(r\) is an integer, and no
correction is needed. Otherwise, since
\[
    R=\lceil 2\delta d\rceil\le 4\delta d
\]
whenever \(2\delta d\ge 1\), the number of queries to \(B_P\) is at most
\[
    C_{\mathrm{enum}}d\sqrt{\delta}.
\]
Since each application of \(B_P\) uses at most \(C\sqrt K\) queries to
\(O_A\), the total number of queries to \(O_A\) used in the correction step
is at most
\[
    C_{\mathrm{enum}}d\sqrt{\delta}\cdot C\sqrt K
    =
    C'd\sqrt K\sqrt{\delta},
\]
for another universal constant \(C'>0\).

After all indices in \(E\) have been found, we flip the corresponding bits of
\(\widehat b\).  On the event
\(\|\widetilde b-b^{(A)}\|_1\le \delta d\), and conditioned on the success of the
quantum enumeration procedure, this produces the exact vector \(b^{(A)}\).

Therefore, the approximate-recovery algorithm followed by the correction
procedure exactly recovers \(b^{(A)}\) with success probability at least
\[
    \frac23\cdot \frac56
    =
    \frac59 .
\]
Its total number of queries to \(O_A\) is at most
\[
    \frac{d\sqrt K}{\alpha_0}
    +
    C'd\sqrt K\sqrt{\delta}.
\]


Choose
\[
    \alpha_0
    >
    \max\left\{
        1,\frac{2}{\beta}
    \right\}.
\]
Then
\[
    \frac1{\alpha_0}<\frac{\beta}{2}.
\]
Next, choose
\[
    0<\delta<
    \min\left\{
        1,
        \left(
            \frac{\beta}{2C'}
        \right)^2
    \right\}.
\]
It follows that
\[
    C'\sqrt{\delta}<\frac{\beta}{2},
\]
and hence
\[
    \frac1{\alpha_0}
    +
    C'\sqrt{\delta}
    <
    \beta.
\]
Therefore, the resulting exact-recovery algorithm uses fewer than
\[
    \beta d\sqrt K
\]
queries to $O_A$, contradicting the exact-recovery lower bound.

Setting
\[
    c_0:=\delta
\]
completes the proof.

\end{proof}

\begin{lemma}[$\ell_2$ approximate recovery lower bound]
\label{lem:l2-approx-row-rare-search}
There exist universal constants $\alpha_0>1$ and $c_0>0$ such that the
following holds. For every even $d\in\mathbb N$, every $K\in\mathbb N$,
and every quantum algorithm $\mathcal A$ that, on every instance
$A\in\mathcal D_{d,K}$ of
$\texttt{Recovery}^{d}\circ\texttt{Search}^{K}$, makes at most
\[
    \frac{d\sqrt K}{\alpha_0}
\]
queries to $O_A$ and outputs a vector
$\widetilde b\in\mathbb R^d$, there exists an input
$A\in\mathcal D_{d,K}$ such that
\[
    \Pr\left[
        \|\widetilde b-b^{(A)}\|_2
        \geq c_0\sqrt d
    \right]
    \geq\frac13.
\]
\end{lemma}
\begin{proof}
Let $\alpha_0>1$ and $c_0>0$ be the constants in
Lemma~\ref{lem:l1-approx-row-rare-search}. 
Consider any quantum algorithm $\mathcal A$ making at most
\[
    \frac{d\sqrt K}{\alpha_0}
\]
queries to $O_A$ and outputting $\widetilde b\in\mathbb R^d$.
By Lemma~\ref{lem:l1-approx-row-rare-search}, there exists a valid input $A$
such that
\[
    \Pr\!\left[
        \|\widetilde b-b^{(A)}\|_1\ge c_0 d
    \right]\ge \frac13 .
\]
For every vector $x\in\mathbb R^d$, we have
\[
    \|x\|_1\le \sqrt d\,\|x\|_2 .
\]
Applying this to
\[
    x=\widetilde b-b^{(A)}
\]
gives
\[
    \|\widetilde b-b^{(A)}\|_2
    \ge
    \frac{\|\widetilde b-b^{(A)}\|_1}{\sqrt d}.
\]
Hence, on the event
\[
    \|\widetilde b-b^{(A)}\|_1\ge c_0 d,
\]
we also have
\[
    \|\widetilde b-b^{(A)}\|_2
    \ge
    c_0\sqrt d.
\]
Therefore,
\[
    \Pr\!\left[
        \|\widetilde b-b^{(A)}\|_2\ge c_0\sqrt d
    \right]
    \ge
    \Pr\!\left[
        \|\widetilde b-b^{(A)}\|_1\ge c_0 d
    \right]
    \ge
    \frac13 .
\]
This proves the lemma.
\end{proof}

\begin{theorem}[Quantum lower bound for heavy-tailed multivariate mean estimation]
\label{thm:row-wise-rare-spike-lower-bound}
Let $p\in(1,2]$. There exists a universal constant $c>0$ such that the
following holds. Let $d\geq2$ and suppose
\[
    n\geq d.
\]
For every quantum mean estimator that uses at most $n$ queries to the
quantum sampling oracle of a $d$-dimensional random variable, there
exists a random vector $\X$ with mean $\bmu=\EBP{\X}$ satisfying
\[
    \EBP{\|\X-\bmu\|^p}\leq\sigma^p
\]
such that the estimator outputs $\widetilde\bmu$ satisfying
\[
    \Pr\left[
        \|\widetilde\bmu-\bmu\|
        \geq
        c\,
        \frac{
            \sigma
            d^{\frac{2(p-1)}p-\frac12}
        }{
            n^{\frac{2(p-1)}p}
        }
    \right]
    \geq\frac13.
\]
\end{theorem}

\begin{proof}
We first prove the result when $d$ is even. Let $\alpha_0>1$ and
$c_0>0$ be the constants from
Lemma~\ref{lem:l2-approx-row-rare-search}. Choose
\[
    K
    :=
    \left\lceil
        \left(
            \frac{\alpha_0n}{d}
        \right)^2
    \right\rceil.
\]
Since $n\geq d$ and $\alpha_0>1$, we have
\[
    \frac{\alpha_0n}{d}\geq\alpha_0>1.
\]
Therefore,
\[
    \left(
        \frac{\alpha_0n}{d}
    \right)^2
    \leq K
    \leq
    \left(
        \frac{\alpha_0n}{d}
    \right)^2+1
    \leq
    2\left(
        \frac{\alpha_0n}{d}
    \right)^2.
\]
Consequently,
\[
    K
    =
    \Theta\left(
        \frac{n^2}{d^2}
    \right).
\]
Moreover,
\[
    K
    \geq
    \left(
        \frac{\alpha_0n}{d}
    \right)^2
\]
implies
\[
    n
    \leq
    \frac{d\sqrt K}{\alpha_0}.
\]
Thus Lemma~\ref{lem:l2-approx-row-rare-search} applies.

Fix an arbitrary input $A\in\{0,1\}^{d\times K}\in\mathcal D_{d,K}$ of
$\texttt{Recovery}^{d}\circ\texttt{Search}^{K}$. Let
\[
    \Omega=[d]\times[K],
    \qquad
    \mathbb P(\omega=(i,j))=\frac1{dK}.
\]
For a scale parameter $R>0$ to be chosen later, define the random vector
\[
    \X^{(A)}:\Omega\to\mathbb R^d
\]
by
\[
    \X^{(A)}(i,j):=R A_{i,j}e_i,
\]
where $e_i\in\mathbb R^d$ is the $i$-th standard basis vector.

First, compute the mean. For each coordinate $i\in[d]$,
\begin{align*}
    \bmu_i
    :=
    \EBP{\X^{(A)}_i}
    &=
    \sum_{j=1}^K \frac1{dK} R A_{i,j} \\
    &=
    \frac{R}{dK}\sum_{j=1}^K A_{i,j} \\
    &=
    \frac{R}{dK}b_i^{(A)}.
\end{align*}
Hence
\begin{equation}\label{equ: mean and b}
        \bmu:=\mathbb E[\X^{(A)}]
    =
    \frac{R}{dK}b^{(A)}.
\end{equation}

Next, we verify the $p$-th central moment condition. By the balance promise, exactly
$d/2$ entries of $A$ are equal to $1$. Therefore
\[
    \mathbb P(\X^{(A)}\ne 0)
    =
    \frac{d/2}{dK}
    =
    \frac1{2K}.
\]
Since $\|X^{(A)}(i,j)\|_2=R$ whenever $A_{i,j}=1$ and equals $0$ otherwise, we have
\[
    \mathbb E\|\X^{(A)}\|^p
    =
    \frac{R^p}{2K}.
\]
Also, since $b^{(A)}$ has Hamming weight $d/2$,
\[
    \|b^{(A)}\|=\sqrt{\frac d2}.
\]
Thus
\[
    \|\bmu\|
    =
    \frac{R}{dK}\sqrt{\frac d2}
    =
    \frac{R}{K\sqrt{2d}},
\]
and therefore
\[
    \|\bmu\|^p
    =
    \frac{R^p}{K^p(2d)^{p/2}}.
\]
Using the inequality
\[
    \|x-y\|^p
    \le
    2^{p-1}\bigl(\|x\|^p+\|y\|^p\bigr),
    \qquad p\ge 1,
\]
we obtain
\begin{align*}
    \mathbb E\|\X^{(A)}-\bmu\|^p
    &\le
    2^{p-1}\mathbb E\|\X^{(A)}\|^p
    +
    2^{p-1}\|\bmu\|^p \\
    &=
    2^{p-1}\frac{R^p}{2K}
    +
    2^{p-1}\frac{R^p}{K^p(2d)^{p/2}}.
\end{align*}
Since $K\ge 1$, $d\ge 1$, and $p>1$,
\[
    \frac1{K^p(2d)^{p/2}}
    \le
    \frac1K.
\]
Hence
\[
    \mathbb E\|\X^{(A)}-\bmu\|^p
    \le
    2^{p-1}\frac{R^p}{2K}
    +
    2^{p-1}\frac{R^p}{K}
    \le
    2^p\frac{R^p}{K}.
\]
Choose
\[
    R:=\frac{\sigma}{2}K^{1/p}.
\]
Then
\[
    2^p\frac{R^p}{K}
    =
    2^p\frac{\sigma^p K/2^p}{K}
    =
    \sigma^p.
\]
Therefore
\[
    \mathbb E\|\X^{(A)}-\bmu\|^p\le \sigma^p.
\]


Let $\mathcal M$ be an arbitrary quantum mean estimator that uses at
most $n$ queries to the quantum sampling oracle. We now show explicitly
that a quantum sampling oracle for $\X^{(A)}$ can be implemented using
one query to $O_A$.

Introduce index registers $\mathsf I$ and $\mathsf J$, a one-qubit
register $\mathsf C$, and a value register $\mathsf X$. Let
$U_{\mathrm{unif}}$ be an input-independent unitary satisfying
\[
    U_{\mathrm{unif}}
    \ket{0}_{\mathsf I}\ket{0}_{\mathsf J}
    =
    \frac{1}{\sqrt{dK}}
    \sum_{i=1}^d\sum_{j=1}^K
    \ket{i}_{\mathsf I}\ket{j}_{\mathsf J}.
\]
Since $U_{\mathrm{unif}}$ is independent of $A$, it requires no queries
to $O_A$.

Recall that the oracle for $A$ is
\[
    O_A\ket{i,j,c}
    =
    (-1)^{cA_{i,j}}\ket{i,j,c}.
\] 
Conjugating the register $\mathsf C$ by Hadamard gates gives the
corresponding bit-query oracle
\[
    B_A
    :=
    \bigl(I_{\mathsf I\mathsf J}\otimes H_{\mathsf C}\bigr)
    O_A
    \bigl(I_{\mathsf I\mathsf J}\otimes H_{\mathsf C}\bigr).
\]
Indeed,
\[
    B_A\ket{i,j,c}
    =
    \ket{i,j,c\oplus A_{i,j}}.
\]
Therefore, $B_A$ uses exactly one query to $O_A$.

Next, let $V_R$ be the query-free reversible operation satisfying
\[
    V_R
    \ket{i}_{\mathsf I}\ket{j}_{\mathsf J}
    \ket{c}_{\mathsf C}\ket{0}_{\mathsf X}
    =
    \ket{i}_{\mathsf I}\ket{j}_{\mathsf J}
    \ket{c}_{\mathsf C}\ket{Rc e_i}_{\mathsf X}.
\]
Here the vector $Rc e_i$ is stored using the finite-precision encoding
assumed for the sampling-oracle model.

Define the quantum sampling oracle $O_{\X^{(A)}}$ by the circuit that
first applies $U_{\mathrm{unif}}$, then $B_A$, and finally $V_R$. On the
all-zero input, it prepares
\begin{align*}
    O_{\X^{(A)}}
    \ket{0}_{\mathsf X}
    \ket{0}_{\mathsf I}
    \ket{0}_{\mathsf J}
    \ket{0}_{\mathsf C}
    &=
    \frac{1}{\sqrt{dK}}
    \sum_{i=1}^d\sum_{j=1}^K
    \ket{R A_{i,j}e_i}_{\mathsf X}
    \ket{i,j,A_{i,j}}_{\mathsf I\mathsf J\mathsf C}.
\end{align*}

We verify that this is a quantum sampling oracle in the sense of
Definition~\ref{def:quantum-sampling-oracle}. Let
\[
    I_1(A):=\{i\in[d]:b_i^{(A)}=1\}.
\]
For every $i\in I_1(A)$, let $j_i$ denote the unique index such that
$A_{i,j_i}=1$. Define
\[
    \ket{\psi_0^{(A)}}
    :=
    \frac{1}{\sqrt{dK-d/2}}
    \sum_{\substack{i\in[d],\,j\in[K]\\A_{i,j}=0}}
    \ket{i,j,0}
\]
and, for every $i\in I_1(A)$,
\[
    \ket{\psi_{R e_i}^{(A)}}
    :=
    \ket{i,j_i,1}.
\]
Because every valid $A\in\mathcal D_{d,K}$ contains exactly $d/2$
entries equal to one, the state prepared above can be written as
\begin{align*}
    O_{\X^{(A)}}\ket{0}\ket{0}
    &=
    \sqrt{1-\frac{1}{2K}}\,
    \ket{0}\ket{\psi_0^{(A)}}
    +
    \sum_{i\in I_1(A)}
    \frac{1}{\sqrt{dK}}\,
    \ket{R e_i}\ket{\psi_{R e_i}^{(A)}}.
\end{align*}
On the other hand,
\[
    \Pr[\X^{(A)}=0]
    =
    1-\frac{1}{2K},
\]
and, for every $i\in I_1(A)$,
\[
    \Pr[\X^{(A)}=R e_i]
    =
    \frac{1}{dK}.
\]
Thus,
\[
    O_{\X^{(A)}}\ket{0}\ket{0}
    =
    \sum_x
    \sqrt{p_{\X^{(A)}}(x)}
    \ket{x}\ket{\psi_x^{(A)}},
\]
so $O_{\X^{(A)}}$ is a valid quantum sampling oracle for
$\X^{(A)}$.

Notice that the register containing $A_{i,j}$ need not be uncomputed:
it is included in the auxiliary state $\ket{\psi_x^{(A)}}$, as permitted
by the definition of a quantum sampling oracle. Consequently, one
application of $O_{\X^{(A)}}$ uses exactly one query to $O_A$.
Likewise, $O_{\X^{(A)}}^\dagger$ can be implemented using one query to
$O_A$ by reversing the above circuit.

Therefore, every query made by $\mathcal M$ to the quantum sampling
oracle, or its inverse, can be simulated using one query to $O_A$.
By \eqref{equ: mean and b}, running $\mathcal M$ with this oracle induces
an algorithm for
$\texttt{Recovery}^{d}\circ\texttt{Search}^{K}$ using at most $n$
queries to $O_A$. 

Let $\widetilde\bmu$ be the output of the mean estimator and define
\[
    \widetilde b
    :=
    \frac{dK}{R}\widetilde\bmu.
\]
Since
\[
    \bmu=\frac{R}{dK}b^{(A)},
\]
we have
\[
    \widetilde b-b^{(A)}
    =
    \frac{dK}{R}(\widetilde\bmu-\bmu),
\]
and hence
\[
    \|\widetilde b-b^{(A)}\|
    =
    \frac{dK}{R}\|\widetilde\bmu-\bmu\|.
\]
Because $n\le d\sqrt K/\alpha_0$, Lemma~\ref{lem:l2-approx-row-rare-search} implies that
there exists an input $A$ such that, with probability at least $1/3$,
\[
    \|\widetilde b-b^{(A)}\|
    \ge
    c_0\sqrt d.
\]
For this input $A$,
\[
    \frac{dK}{R}\|\widetilde\bmu-\bmu\|
    \ge
    c_0\sqrt d,
\]
and therefore
\[
    \|\widetilde\bmu-\bmu\|
    \ge
    c_0\frac{R\sqrt d}{dK}
    =
    c_0\frac{R}{\sqrt d\,K}.
\]
Substituting $R=(\sigma/2)K^{1/p}$ gives
\[
    \|\widetilde\bmu-\bmu\|
    \ge
    \frac{c_0}{2}
    \frac{\sigma K^{1/p}}{\sqrt d\,K}
    =
    \frac{c_0}{2}
    \frac{\sigma}{\sqrt d\,K^{(p-1)/p}}.
\]
Since
\[
    K=\Theta\!\left(\frac{n^2}{d^2}\right),
\]
we obtain
\[
    \frac1{\sqrt d\,K^{(p-1)/p}}
    =
    \Omega\!\left(
        \frac{
            d^{\frac{2(p-1)}p-\frac12}
        }{
            n^{\frac{2(p-1)}p}
        }
    \right).
\]
Thus
\[
    \|\widetilde\bmu-\bmu\|
    \ge
    \Omega\!\left(
        \frac{
            \sigma\, d^{\frac{2(p-1)}p-\frac12}
        }{
            n^{\frac{2(p-1)}p}
        }
    \right)
\]
with probability at least $1/3$, as claimed.

It remains to consider odd $d$. Let
\[
    d_0:=d-1.
\]
Then $d_0$ is even and $d/2\leq d_0\leq d$. Given any
$d_0$-dimensional random vector $\Y$, define its zero-padded embedding
\[
    \overline{\Y}:=(\Y,0)\in\mathbb R^d.
\]
A quantum sampling oracle for $\overline{\Y}$ can be implemented from one
for $\Y$ without any additional oracle queries. Moreover,
\[
    \mathbb E
    \left\|
        \overline{\Y}-\mathbb E[\overline{\Y}]
    \right\|^p
    =
    \mathbb E\|\Y-\mathbb E[\Y]\|^p.
\]

Let $\mathcal M$ be an arbitrary mean estimator in dimension $d$. Running
$\mathcal M$ on $\overline{\Y}$ and returning the first $d_0$ coordinates
of its output defines a mean estimator in dimension $d_0$ using the same
number of queries. Since $n\geq d\geq d_0$, the even-dimensional result
applies.

Let
\[
    \beta:=\frac{2(p-1)}{p}-\frac12.
\]
Since $p\in(1,2]$, we have $-1/2<\beta\leq1/2$. Therefore,
using $d/2\leq d_0\leq d$,
\[
    d_0^\beta
    \geq
    2^{-1/2}d^\beta.
\]
Furthermore, projection onto the first $d_0$ coordinates does not increase
the Euclidean norm. Hence, a lower bound for the projected estimate also
gives the same lower bound, up to a universal constant factor, for the
full $d$-dimensional estimate. Absorbing the factor $2^{-1/2}$ into constant $c$
proves the result for odd $d$.
\end{proof}

\begin{corollary}[Bounded-error quantum query lower bound for
heavy-tailed mean estimation]
\label{cor:epsilon-heavy-tailed-mean-lower-bound}
Let $p\in(1,2]$ be fixed. There exist constants $c_0,c_p>0$ such
that the following holds. Let $d\geq2$ and suppose
\[
    0<\epsilon
    \leq
    c_0\frac{\sigma}{\sqrt d}.
\]
Any quantum mean estimator that, for every $d$-dimensional random vector
$\X$ satisfying
\[
    \mathbb E
    \left\|
        \X-\mathbb E[\X]
    \right\|^p
    \leq
    \sigma^p,
\]
outputs an estimate $\widetilde{\bmu}$ satisfying
\[
    \Pr\left[
        \left\|
            \widetilde{\bmu}-\mathbb E[\X]
        \right\|
        \leq\epsilon
    \right]
    \geq
    \frac23
\]
must use at least
\[
    c_p\,
    d^{\frac{3p-4}{4(p-1)}}
    \left(
        \frac{\sigma}{\epsilon}
    \right)^{\frac{p}{2(p-1)}}
\]
queries to the quantum sampling oracle. Equivalently, the quantum query
complexity is
\[
    \Omega\left(
        d^{\frac{3p-4}{4(p-1)}}
        \left(
            \frac{\sigma}{\epsilon}
        \right)^{\frac{p}{2(p-1)}}
    \right).
\]
\end{corollary}
\begin{proof}
Define
\[
    \alpha:=\frac{2(p-1)}{p}.
\]
By Theorem~\ref{thm:row-wise-rare-spike-lower-bound}, a quantum mean
estimator using $n$ queries incurs, on some admissible random vector, an
estimation error of order at least
\[
    \frac{\sigma d^{\alpha-\frac12}}{n^\alpha}
\]
with probability at least $1/3$. Therefore, achieving error at most
$\epsilon$ with constant success probability requires
\[
    n^\alpha
    =
    \Omega\left(
        \frac{\sigma d^{\alpha-\frac12}}{\epsilon}
    \right).
\]
It follows that
\[
    n
    =
    \Omega\left(
        \left(\frac{\sigma}{\epsilon}\right)^{1/\alpha}
        d^{\frac{\alpha-\frac12}{\alpha}}
    \right).
\]
Since
\[
    \frac1\alpha
    =
    \frac{p}{2(p-1)}
\]
and
\[
    \frac{\alpha-\frac12}{\alpha}
    =
    \frac{3p-4}{4(p-1)},
\]
we obtain
\[
    n
    =
    \Omega\left(
        d^{\frac{3p-4}{4(p-1)}}
        \left(
            \frac{\sigma}{\epsilon}
        \right)^{\frac{p}{2(p-1)}}
    \right).
\]

Finally, the assumed condition
$\epsilon\leq c_0\sigma/\sqrt d$, for a sufficiently small constant
$c_0$, ensures that the resulting lower bound is at least $d$.
Therefore, it lies in the query regime required by
Theorem~\ref{thm:row-wise-rare-spike-lower-bound}.
\end{proof}

\begin{corollary}[Expected-query quantum lower bound under $L_p$ error]
\label{cor:expected-query-lower-bound-low-dimensional with pth moment error}
Let $p\in(1,2]$ be fixed. There exist constants
$\bar c_0,\bar c_p>0$ such that the following holds. Let
$\mathcal P_\sigma$ denote the set of all $d$-dimensional random vectors
$\X$ with mean $\bmu=\EBP{\X}$ satisfying
\[
    \EBP{\|\X-\bmu\|^p}\leq\sigma^p.
\]
Let $d\geq2$ and suppose
\[
    0<\epsilon
    \leq
    \bar c_0\frac{\sigma}{\sqrt d}.
\]
Then any quantum algorithm that, for every $\X\in\mathcal P_\sigma$,
outputs an estimator $\widetilde{\bmu}$ satisfying
\[
    \EBP{
        \|\widetilde{\bmu}-\bmu\|^p
    }
    \leq
    \epsilon^p
\]
must make at least
\[
    \Omega\left(
        d^{\frac{3p-4}{4(p-1)}}
        \left(
            \frac{\sigma}{\epsilon}
        \right)^{\frac{p}{2(p-1)}}
    \right)
\]
queries in expectation to the quantum sampling oracle $O_{\X}$. 
\end{corollary}

\begin{proof}
Let $\mathcal A$ be a quantum algorithm satisfying the guarantees in
Problem~\ref{prob: obtain unbiased mean estimate for heavy tail random variable}.
For each $\X\in\mathcal P_\sigma$, let $Q_{\X}$ denote the number of
queries made by $\mathcal A$, and define its worst-case expected query
complexity by
\[
    T
    :=
    \sup_{\X\in\mathcal P_\sigma}
    \EBP{Q_{\X}}.
\]
If $T=\infty$, the result is immediate, so assume $T<\infty$.

For every $\X\in\mathcal P_\sigma$, Markov's inequality gives
\[
    \Pr(Q_{\X}>6T)
    \leq
    \frac{\EBP{Q_{\X}}}{6T}
    \leq
    \frac16.
\]
Moreover, since
\[
    \EBP{
        \|\widetilde{\bmu}-\bmu\|^p
    }
    \leq
    \epsilon^p,
\]
another application of Markov's inequality gives
\[
    \Pr\left(
        \|\widetilde{\bmu}-\bmu\|
        >
        6^{1/p}\epsilon
    \right)
    \leq
    \frac16.
\]
Therefore, by the union bound,
\[
    \Pr\left(
        Q_{\X}\leq6T
        \ \text{and}\
        \|\widetilde{\bmu}-\bmu\|
        \leq6^{1/p}\epsilon
    \right)
    \geq
    \frac23.
\]

Define a truncated algorithm $\mathcal A'$ as follows. Run $\mathcal A$,
but terminate it once it attempts to make more than $6T$ queries. If
this happens, output an arbitrary vector. Then $\mathcal A'$ uses at
most $6T$ queries in the worst case and, for every
$\X\in\mathcal P_\sigma$, outputs an estimate satisfying
\[
    \|\widetilde{\bmu}-\bmu\|
    \leq
    6^{1/p}\epsilon
\]
with probability at least $2/3$. The truncation may destroy the unbiasedness of the estimator. This causes
no difficulty, because
Corollary~\ref{cor:epsilon-heavy-tailed-mean-lower-bound} applies to
arbitrary estimators and requires only a bounded-error guarantee.

Set
\[
    \epsilon'
    :=
    6^{1/p}\epsilon.
\]
Choose
\[
    \bar c_0
    :=
    6^{-1/p}c_0,
\]
where $c_0$ is the constant from
Corollary~\ref{cor:epsilon-heavy-tailed-mean-lower-bound}. The assumed
condition on $\epsilon$ then implies
\[
    \epsilon'
    =
    6^{1/p}\epsilon
    \leq
    c_0\frac{\sigma}{\sqrt d}.
\]
Thus, Corollary~\ref{cor:epsilon-heavy-tailed-mean-lower-bound} applies
to $\mathcal A'$ and gives
\[
    6T
    \geq
    c_p\,
    d^{\frac{3p-4}{4(p-1)}}
    \left(
        \frac{\sigma}{\epsilon'}
    \right)^{\frac{p}{2(p-1)}}.
\]
Substituting $\epsilon'=6^{1/p}\epsilon$, we obtain
\begin{align*}
    6T
    &\geq
    c_p\,
    d^{\frac{3p-4}{4(p-1)}}
    \left(
        \frac{\sigma}{6^{1/p}\epsilon}
    \right)^{\frac{p}{2(p-1)}} \\
    &=
    c_p\,
    6^{-\frac{1}{2(p-1)}}
    d^{\frac{3p-4}{4(p-1)}}
    \left(
        \frac{\sigma}{\epsilon}
    \right)^{\frac{p}{2(p-1)}}.
\end{align*}
Consequently,
\[
    T
    \geq
    \frac{c_p}{
        6^{1+\frac{1}{2(p-1)}}
    }
    d^{\frac{3p-4}{4(p-1)}}
    \left(
        \frac{\sigma}{\epsilon}
    \right)^{\frac{p}{2(p-1)}}.
\]
Absorbing the constant into the
$\Omega(\cdot)$ notation proves
\[
    T
    =
    \Omega\left(
        d^{\frac{3p-4}{4(p-1)}}
        \left(
            \frac{\sigma}{\epsilon}
        \right)^{\frac{p}{2(p-1)}}
    \right).
\]
\end{proof}

\begin{corollary}[Expected-query quantum lower bound for unbiased
heavy-tailed mean estimation]
\label{cor:expected-query-lower-bound-low-dimensional-unbiased}
Let $p\in(1,2]$ be fixed. There exist constants
$\bar c_0,\bar c_p>0$ such that the following holds. Let
$\mathcal P_\sigma$ denote the set of all $d$-dimensional random vectors
$\X$ with mean $\bmu=\EBP{\X}$ satisfying
\[
    \EBP{\|\X-\bmu\|^p}\leq\sigma^p.
\]
Let $d\geq2$ and suppose
\[
    0<\epsilon
    \leq
    \bar c_0\frac{\sigma}{\sqrt d}.
\]
Then any quantum algorithm that, for every $\X\in\mathcal P_\sigma$,
outputs an estimator $\widetilde{\bmu}$ satisfying
\[
    \EBP{\widetilde{\bmu}}=\bmu, \quad\text{and} \quad
    \EBP{
        \|\widetilde{\bmu}-\bmu\|^p
    }
    \leq
    \epsilon^p
\]
must make at least
\[
    \Omega\left(
        d^{\frac{3p-4}{4(p-1)}}
        \left(
            \frac{\sigma}{\epsilon}
        \right)^{\frac{p}{2(p-1)}}
    \right)
\]
queries in expectation to the quantum sampling oracle $O_{\X}$. 
\end{corollary}
\begin{proof}
    Any unbiased estimator satisfying
\[
    \EBP{\|\widetilde{\bmu}-\bmu\|^p}\leq\epsilon^p
\]
is, in particular, an estimator satisfying the assumptions of
Corollary~\ref{cor:expected-query-lower-bound-low-dimensional with pth moment error}.
The claimed lower bound therefore follows immediately from that corollary.
\end{proof}

\subsection{Proof of Theorem \ref{thm:lower-bounds for high probability}}\label{appendix: proof of quantum lower bounds for high probability}

\begin{lemma}[Quantum query lower bound for bounded-variance mean estimation, \cite{Cornelissen_2022}]
\label{lem:bounded-variance-quantum-lower-bound}
Let $d\geq1$, $\sigma>0$, and $0<\epsilon\leq\sigma$. Let
$\mathcal P_{2,\sigma}$ denote the class of all $d$-dimensional random
vectors $\X$ with mean $\bmu=\mathbb E[\X]$ and covariance matrix $\Sigma$
satisfying
\[
    \operatorname{Tr}(\Sigma)
    =
    \mathbb E\|\X-\bmu\|^2
    \leq\sigma^2.
\]
Any quantum algorithm that, for every $\X\in\mathcal P_{2,\sigma}$,
outputs an estimate $\widetilde\bmu$ satisfying
\[
    \Pr\left(
        \|\widetilde\bmu-\bmu\|\leq\epsilon
    \right)
    \geq\frac23
\]
must make
\[
    \Omega\left(
        \min\left\{
            \sqrt d\,\frac{\sigma}{\epsilon},
            \left(\frac{\sigma}{\epsilon}\right)^2
        \right\}
    \right)
\]
queries to the quantum sampling oracle in the worst case.
\end{lemma}

\begin{proof}
This is a direct consequence of the bounded-variance quantum lower bounds
in Theorems~3.7 and~3.8 of \cite{Cornelissen_2022}. More precisely, these results show that an $n$-query quantum mean
estimator must incur error
\[
    \Omega\left(
        \frac{\sigma}{\sqrt n}
    \right)
\]
in the regime $n=O(d)$, and error
\[
    \Omega\left(
        \frac{\sqrt d\,\sigma}{n}
    \right)
\]
in the regime $n=\Omega(d)$, on some random vector satisfying
$\operatorname{Tr}(\Sigma)\leq\sigma^2$.

Inverting these error lower bounds shows that achieving error at most
$\epsilon$ with probability at least $2/3$ requires
\[
    \Omega\left(
        \left(\frac{\sigma}{\epsilon}\right)^2
    \right)
\]
queries in the first regime and
\[
    \Omega\left(
        \sqrt d\,\frac{\sigma}{\epsilon}
    \right)
\]
queries in the second regime. Combining the two regimes gives
\[
    \Omega\left(
        \min\left\{
            \sqrt d\,\frac{\sigma}{\epsilon},
            \left(\frac{\sigma}{\epsilon}\right)^2
        \right\}
    \right).
\]
Equivalently, the lower bound can be written as
\[
    \begin{cases}
        \displaystyle
        \Omega\left(
            \sqrt d\,\frac{\sigma}{\epsilon}
        \right),
        &
        1\leq d\leq
        \left(\dfrac{\sigma}{\epsilon}\right)^2,
        \\[3mm]
        \displaystyle
        \Omega\left(
            \left(\dfrac{\sigma}{\epsilon}\right)^2
        \right),
        &
        d\geq
        \left(\dfrac{\sigma}{\epsilon}\right)^2.
    \end{cases}
\]
The lower bounds in \cite{Cornelissen_2022} are formulated using the
binary-oracle access model. For their hard instances, one query to the
quantum sampling oracle can be simulated using a constant number of
queries to the corresponding binary oracles. Therefore, the same lower
bound holds, up to a constant factor, in the quantum sampling oracle used here.
\end{proof}

We now prove Theorem \ref{thm:lower-bounds for high probability}
\begin{proof}
Let $\mathcal A$ be an arbitrary quantum algorithm that solves
Problem~\ref{prob: obtain mean estimate for heavy tail random variable with bounded mean}
with success probability at least $2/3$. We consider the three cases separately.

First, suppose that $p\in(1,4/3]$. By
Theorem~\ref{thm: lower bounds 1 for heavy-tailed mean estimation},
every such quantum algorithm must make
\[
    \Omega\left(
        \left(
            \frac{\sigma}{\epsilon}
        \right)^{\frac{p}{2(p-1)}}
    \right)
\]
queries in the worst case. This proves the first statement.

Next, suppose that $p\in(4/3,2]$ and
\[
    1\leq d\leq
    \left(
        \frac{\sigma}{\epsilon}
    \right)^2.
\]
We show that
\[
    \Omega\left(
        d^{\frac{3p-4}{4(p-1)}}
        \left(
            \frac{\sigma}{\epsilon}
        \right)^{\frac{p}{2(p-1)}}
    \right)
\]
queries are necessary.

When $d=1$, the desired lower bound reduces to
\[
    \Omega\left(
        \left(
            \frac{\sigma}{\epsilon}
        \right)^{\frac{p}{2(p-1)}}
    \right),
\]
which follows directly from
Theorem~\ref{thm: lower bounds 1 for heavy-tailed mean estimation}.
We may therefore assume that $d\geq2$.

We divide the argument into two subcases.

Suppose first that
\[
    \epsilon
    \leq
    c_0\frac{\sigma}{\sqrt d}\Leftrightarrow d\leq c_0^{2}\left(\frac{\sigma}{\epsilon}\right)^{2},
\]
where $c_0>0$ is the small constant from
Corollary~\ref{cor:epsilon-heavy-tailed-mean-lower-bound}.
The corollary directly gives
\[
    \Omega\left(
        d^{\frac{3p-4}{4(p-1)}}
        \left(
            \frac{\sigma}{\epsilon}
        \right)^{\frac{p}{2(p-1)}}
    \right)
\]
queries.

It remains to consider the complementary subcase
\[
    \epsilon
    >
    c_0\frac{\sigma}{\sqrt d}\Leftrightarrow d> c_0^{2}\left(\frac{\sigma}{\epsilon}\right)^{2}.
\]
We use the bounded-variance lower bound to cover this subcase.

Indeed, let $X$ be any random vector satisfying
\[
    \mathbb E\|\X-\mathbb E[\X]\|^2\leq\sigma^2.
\]
Since $p\in(1,2]$, Lyapunov's inequality gives
\[
    \mathbb E\|\X-\mathbb E[\X]\|^p
    \leq
    \left(
        \mathbb E\|\X-\mathbb E[\X]\|^2
    \right)^{p/2}
    \leq\sigma^p.
\]
Therefore, every bounded-variance instance in
$\mathcal P_{2,\sigma}$ is also an admissible instance of
Problem~\ref{prob: obtain mean estimate for heavy tail random variable with bounded mean}.
Thus, the lower bound in
Lemma~\ref{lem:bounded-variance-quantum-lower-bound}
also applies to $\mathcal A$.

Because when
\[
    d\leq
    \left(
        \frac{\sigma}{\epsilon}
    \right)^2,
\]
Lemma~\ref{lem:bounded-variance-quantum-lower-bound} gives
\[
    \Omega\left(
        \sqrt d\,\frac{\sigma}{\epsilon}
    \right)
\]
queries.

We now compare this expression with the desired lower bound. We have
\begin{align*}
&d^{\frac{3p-4}{4(p-1)}}
\left(
    \frac{\sigma}{\epsilon}
\right)^{\frac{p}{2(p-1)}}
\\
&\qquad=
\sqrt d\,\frac{\sigma}{\epsilon}
\left(
    \frac{\sigma}{\epsilon\sqrt d}
\right)^{\frac{2-p}{2(p-1)}}.
\end{align*}
The assumption
\[
    \epsilon
    >
    c_0\frac{\sigma}{\sqrt d}
\]
implies
\[
    \frac{\sigma}{\epsilon\sqrt d}
    <
    \frac1{c_0}.
\]
Consequently,
\[
\begin{aligned}
&d^{\frac{3p-4}{4(p-1)}}
\left(
    \frac{\sigma}{\epsilon}
\right)^{\frac{p}{2(p-1)}}
\\
&\qquad\leq
c_0^{-\frac{2-p}{2(p-1)}}
\sqrt d\,\frac{\sigma}{\epsilon}.
\end{aligned}
\]
Since $p$ is fixed, the factor
\[
    c_0^{-\frac{2-p}{2(p-1)}}
\]
is a constant. Hence
\[
    \sqrt d\,\frac{\sigma}{\epsilon}
    =
    \Omega\left(
        d^{\frac{3p-4}{4(p-1)}}
        \left(
            \frac{\sigma}{\epsilon}
        \right)^{\frac{p}{2(p-1)}}
    \right).
\]
Therefore, the bounded-variance lower bound implies the desired lower
bound in the complementary subcase as well.

Combining the two subcases proves that, whenever $p\in(4/3,2]$ and
\[
    1\leq d\leq
    \left(
        \frac{\sigma}{\epsilon}
    \right)^2,
\]
the query complexity is at least
\[
    \Omega\left(
        d^{\frac{3p-4}{4(p-1)}}
        \left(
            \frac{\sigma}{\epsilon}
        \right)^{\frac{p}{2(p-1)}}
    \right).
\]

Finally, suppose that $p\in(4/3,2]$ and
\[
    d\geq
    \left(
        \frac{\sigma}{\epsilon}
    \right)^2.
\]
As above, the bounded-variance class $\mathcal P_{2,\sigma}$ is contained
in the heavy-tailed class under consideration. Therefore,
Lemma~\ref{lem:bounded-variance-quantum-lower-bound} applies.

The dimensional assumption implies
\[
    \sqrt d\,\frac{\sigma}{\epsilon}
    \geq
    \left(
        \frac{\sigma}{\epsilon}
    \right)^2.
\]
Hence
\begin{align*}
&\min\left\{
    \sqrt d\,\frac{\sigma}{\epsilon},
    \left(
        \frac{\sigma}{\epsilon}
    \right)^2
\right\}
=
\left(
    \frac{\sigma}{\epsilon}
\right)^2.
\end{align*}
Lemma~\ref{lem:bounded-variance-quantum-lower-bound} therefore gives
\[
    \Omega\left(
        \left(
            \frac{\sigma}{\epsilon}
        \right)^2
    \right)
\]
queries in the worst case. This proves the third statement and completes
the proof.
\end{proof}

    

\subsection{Proof of Theorem \ref{thm:lower-bounds for unbiased}}\label{appendix: proof of quantum lower bounds for unbiased}

\begin{lemma}[Inherited bounded-variance expected-query quantum lower bound for heavy-tailed mean estimation under $L_p$ error]
\label{lem:bounded-variance-subclass-lower-bound with p-th moment error}
Fix $p\in(1,2]$, $d\geq1$, $\sigma>0$, and
$0<\epsilon\leq\sigma$. Let $\mathcal P_{p,\sigma}$ denote the class of
all $d$-dimensional random vectors $\X$ with mean
$\bmu=\mathbb E[\X]$ satisfying
\[
    \mathbb E\|\X-\bmu\|^p\leq\sigma^p.
\]
Suppose a quantum algorithm, for every $\X\in\mathcal P_{p,\sigma}$,
outputs an estimator $\widetilde\bmu$ satisfying
\[
    \mathbb E\|\widetilde\bmu-\bmu\|^p
    \leq\epsilon^p.
\]
Then its worst-case expected number of queries to the quantum sampling
oracle is at least
\[
    \Omega\left(
        \min\left\{
            \sqrt d\,\frac{\sigma}{\epsilon},
            \left(\frac{\sigma}{\epsilon}\right)^2
        \right\}
    \right).
\]
\end{lemma}

\begin{proof}
Let $\mathcal A$ be an algorithm satisfying the assumptions of the lemma. 
Let
\[
    T
    :=
    \sup_{\X\in\mathcal P_{p,\sigma}}
    \mathbb E[Q_\X]
\]
be its worst-case expected number of queries. If $T=\infty$, the result
is immediate, so suppose that $T<\infty$.

We first observe that the bounded-variance class
$\mathcal P_{2,\sigma}$ is contained in $\mathcal P_{p,\sigma}$. Indeed,
for every $\X\in\mathcal P_{2,\sigma}$, Lyapunov's inequality gives
\[
    \mathbb E\|\X-\bmu\|^p
    \leq
    \left(
        \mathbb E\|\X-\bmu\|^2
    \right)^{p/2}
    \leq\sigma^p.
\]
Therefore, the guarantees of $\mathcal A$ hold, in particular, for every
$\X\in\mathcal P_{2,\sigma}$.

For any such $\X$, Markov's inequality gives
\[
    \Pr(Q_\X>6T)
    \leq
    \frac{\mathbb E[Q_\X]}{6T}
    \leq\frac16.
\]
Moreover, the assumed error guarantee implies
\[
    \Pr\left(
        \|\widetilde\bmu-\bmu\|
        >
        6^{1/p}\epsilon
    \right)
    \leq
    \frac{
        \mathbb E\|\widetilde\bmu-\bmu\|^p
    }{
        6\epsilon^p
    }
    \leq\frac16.
\]
By the union bound,
\[
    \Pr\left(
        Q_\X\leq6T
        \ \text{and}\
        \|\widetilde\bmu-\bmu\|
        \leq6^{1/p}\epsilon
    \right)
    \geq\frac23.
\]

Now define a truncated algorithm $\mathcal A'$ as follows. Run
$\mathcal A$, but terminate it once it is about to exceed
$\lceil6T\rceil$ queries. If truncation occurs, output an arbitrary
vector. Then $\mathcal A'$ makes at most $\lceil6T\rceil$ queries in the
worst case and, for every $\X\in\mathcal P_{2,\sigma}$, satisfies
\[
    \Pr\left(
        \|\widetilde\bmu-\bmu\|
        \leq6^{1/p}\epsilon
    \right)
    \geq\frac23.
\]

Set
\[
    \epsilon':=6^{1/p}\epsilon.
\]
We distinguish two cases.

First, suppose that $\epsilon'\leq\sigma$. Then
Lemma~\ref{lem:bounded-variance-quantum-lower-bound} applies to
$\mathcal A'$ and gives
\[
    \lceil6T\rceil
    \geq
    \Omega\left(
        \min\left\{
            \sqrt d\,\frac{\sigma}{\epsilon'},
            \left(\frac{\sigma}{\epsilon'}\right)^2
        \right\}
    \right).
\]
Since $\epsilon'=6^{1/p}\epsilon$,
\[
\begin{aligned}
&\min\left\{
    \sqrt d\,\frac{\sigma}{\epsilon'},
    \left(\frac{\sigma}{\epsilon'}\right)^2
\right\}\geq
6^{-2/p}
\min\left\{
    \sqrt d\,\frac{\sigma}{\epsilon},
    \left(\frac{\sigma}{\epsilon}\right)^2
\right\}.
\end{aligned}
\]
Absorbing constants gives
\[
    T
    =
    \Omega\left(
        \min\left\{
            \sqrt d\,\frac{\sigma}{\epsilon},
            \left(\frac{\sigma}{\epsilon}\right)^2
        \right\}
    \right).
\]

It remains to consider the case $\epsilon'>\sigma$. In this case,
\[
    \frac{\sigma}{\epsilon}=\frac{\sigma}{6^{-1/p}\epsilon'}<6^{1/p},
\]
and hence
\[
\begin{aligned}
\min\left\{
    \sqrt d\,\frac{\sigma}{\epsilon},
    \left(\frac{\sigma}{\epsilon}\right)^2
\right\}
&\leq
\left(\frac{\sigma}{\epsilon}\right)^2\\
&<
6^{2/p}\\
&\leq36.
\end{aligned}
\]
Thus the claimed lower bound is only of constant order. The deterministic subclass $\X\equiv\bmu$, where $\bmu\in\mathbb R^d$
is arbitrary, shows that every algorithm satisfying the stated guarantee
uniformly over $\bmu$ must make at least one query. Hence $T\geq1$, which
implies
\[
    T
    =
    \Omega\left(
        \min\left\{
            \sqrt d\,\frac{\sigma}{\epsilon},
            \left(\frac{\sigma}{\epsilon}\right)^2
        \right\}
    \right).
\]
The two cases together cover the entire range
$0<\epsilon\leq\sigma$.

\end{proof}

\begin{corollary}[Inherited bounded-variance expected-query quantum lower bound for unbiased heavy-tailed mean estimation]
\label{lem:unbiased-bounded-variance-subclass-lower-bound}
Fix $p\in(1,2]$, $d\geq1$, $\sigma>0$, and
$0<\epsilon\leq\sigma$. Let $\mathcal P_{p,\sigma}$ denote the class of
all $d$-dimensional random vectors $\X$ with mean
$\bmu=\mathbb E[\X]$ satisfying
\[
    \mathbb E\|\X-\bmu\|^p\leq\sigma^p.
\]
Suppose a quantum algorithm, for every $\X\in\mathcal P_{p,\sigma}$,
outputs an estimator $\widetilde\bmu$ satisfying
\[
    \EBP{\tilde{\bmu}}=\bmu, \quad
    \mathbb E\|\widetilde\bmu-\bmu\|^p
    \leq\epsilon^p.
\]
Then its worst-case expected number of queries to the quantum sampling
oracle is at least
\[
    \Omega\left(
        \min\left\{
            \sqrt d\,\frac{\sigma}{\epsilon},
            \left(\frac{\sigma}{\epsilon}\right)^2
        \right\}
    \right).
\]

Equivalently, the expected-query lower bound is
\[
    \begin{cases}
        \displaystyle
        \Omega\left(
            \sqrt d\,\frac{\sigma}{\epsilon}
        \right),
        &
        1\leq d\leq
        \left(\dfrac{\sigma}{\epsilon}\right)^2,
        \\[3mm]
        \displaystyle
        \Omega\left(
            \left(\dfrac{\sigma}{\epsilon}\right)^2
        \right),
        &
        d\geq
        \left(\dfrac{\sigma}{\epsilon}\right)^2.
    \end{cases}
\]
\end{corollary}
\begin{proof}
     Any unbiased estimator satisfying
\[
    \EBP{\|\widetilde{\bmu}-\bmu\|^p}\leq\epsilon^p
\]
is, in particular, an estimator satisfying the assumptions of
Lemma~\ref{lem:bounded-variance-subclass-lower-bound with p-th moment error}.
The claimed lower bound therefore follows immediately from that lemma.
\end{proof}

We now prove Theorem \ref{thm:lower-bounds for unbiased}.
\begin{proof}
Let $\mathcal A$ be an arbitrary quantum algorithm that solves
Problem~\ref{prob: obtain unbiased mean estimate for heavy tail random variable},
and let
\[
    T
    :=
    \sup_{\X\in\mathcal P_{p,\sigma}}
    \mathbb E[Q_{\X}]
\]
denote its worst-case expected number of queries.

We consider the three cases separately.

First, suppose that $p\in(1,4/3]$. By
Corollary~\ref{coro:expected-query-lower-bound-for-unbaised-heavy-tailed-mean-estimation},
we immediately have
\[
    T
    =
    \Omega\left(
        \left(
            \frac{\sigma}{\epsilon}
        \right)^{\frac{p}{2(p-1)}}
    \right).
\]
This proves the first statement.

For completeness, in this range of $p$, the other two available lower
bounds do not improve the dependence on $\sigma/\epsilon$. Indeed,
\[
    \frac{3p-4}{4(p-1)}
    \leq0,
\]
so, for $d\geq1$,
\[
    d^{\frac{3p-4}{4(p-1)}}
    \left(
        \frac{\sigma}{\epsilon}
    \right)^{\frac{p}{2(p-1)}}
    \leq
    \left(
        \frac{\sigma}{\epsilon}
    \right)^{\frac{p}{2(p-1)}}.
\]
Moreover,
\[
    \frac{p}{2(p-1)}\geq2,
\]
and, since $\epsilon\leq\sigma$,
\[
\begin{aligned}
\min\left\{
    \sqrt d\,\frac{\sigma}{\epsilon},
    \left(
        \frac{\sigma}{\epsilon}
    \right)^2
\right\}
&\leq
\left(
    \frac{\sigma}{\epsilon}
\right)^2 \\
&\leq
\left(
    \frac{\sigma}{\epsilon}
\right)^{\frac{p}{2(p-1)}}.
\end{aligned}
\]
Thus the dimension-independent lower bound gives the claimed landscape
for $p\in(1,4/3]$.

Next, suppose that $p\in(4/3,2]$ and
\[
    1\leq d\leq
    \left(
        \frac{\sigma}{\epsilon}
    \right)^2.
\]
We show that
\[
    T
    =
    \Omega\left(
        d^{\frac{3p-4}{4(p-1)}}
        \left(
            \frac{\sigma}{\epsilon}
        \right)^{\frac{p}{2(p-1)}}
    \right).
\]

When $d=1$, the desired expression reduces to
\[
    \left(
        \frac{\sigma}{\epsilon}
    \right)^{\frac{p}{2(p-1)}},
\]
so the result follows from
Corollary~\ref{coro:expected-query-lower-bound-for-unbaised-heavy-tailed-mean-estimation}.
We may therefore assume that $d\geq2$.

Let $\bar c_0>0$ be the constant from
Corollary~\ref{cor:expected-query-lower-bound-low-dimensional-unbiased}.
We divide the argument into two subcases.

Suppose first that
\[
    \epsilon
    \leq
    \bar c_0\frac{\sigma}{\sqrt d}.
\]
Then
Corollary~\ref{cor:expected-query-lower-bound-low-dimensional-unbiased}
directly gives
\[
    T
    =
    \Omega\left(
        d^{\frac{3p-4}{4(p-1)}}
        \left(
            \frac{\sigma}{\epsilon}
        \right)^{\frac{p}{2(p-1)}}
    \right).
\]

It remains to consider the complementary subcase
\[
    \epsilon
    >
    \bar c_0\frac{\sigma}{\sqrt d}.
\]
Since
\[
    d\leq
    \left(
        \frac{\sigma}{\epsilon}
    \right)^2,
\]
Lemma~\ref{lem:unbiased-bounded-variance-subclass-lower-bound}
gives
\[
    T
    =
    \Omega\left(
        \sqrt d\,\frac{\sigma}{\epsilon}
    \right).
\]

We compare this bounded-variance lower bound with the desired expression.
A direct calculation gives
\begin{align*}
&d^{\frac{3p-4}{4(p-1)}}
\left(
    \frac{\sigma}{\epsilon}
\right)^{\frac{p}{2(p-1)}}
\\
&\qquad=
\sqrt d\,\frac{\sigma}{\epsilon}
\left(
    \frac{\sigma}{\epsilon\sqrt d}
\right)^{\frac{2-p}{2(p-1)}}.
\end{align*}
The assumption
\[
    \epsilon
    >
    \bar c_0\frac{\sigma}{\sqrt d}
\]
implies
\[
    \frac{\sigma}{\epsilon\sqrt d}
    <
    \frac1{\bar c_0}.
\]
Therefore,
\begin{align*}
&d^{\frac{3p-4}{4(p-1)}}
\left(
    \frac{\sigma}{\epsilon}
\right)^{\frac{p}{2(p-1)}}
\\
&\qquad\leq
\bar c_0^{-\frac{2-p}{2(p-1)}}
\sqrt d\,\frac{\sigma}{\epsilon}.
\end{align*}
Since $p$ is fixed, the factor
\[
    \bar c_0^{-\frac{2-p}{2(p-1)}}
\]
is a constant. Consequently,
\[
    \sqrt d\,\frac{\sigma}{\epsilon}
    =
    \Omega\left(
        d^{\frac{3p-4}{4(p-1)}}
        \left(
            \frac{\sigma}{\epsilon}
        \right)^{\frac{p}{2(p-1)}}
    \right).
\]
Thus the desired lower bound also holds in the complementary subcase.

Combining the two subcases proves that, throughout
\[
    1\leq d\leq
    \left(
        \frac{\sigma}{\epsilon}
    \right)^2,
\]
we have
\[
    T
    =
    \Omega\left(
        d^{\frac{3p-4}{4(p-1)}}
        \left(
            \frac{\sigma}{\epsilon}
        \right)^{\frac{p}{2(p-1)}}
    \right).
\]

Finally, suppose that $p\in(4/3,2]$ and
\[
    d\geq
    \left(
        \frac{\sigma}{\epsilon}
    \right)^2.
\]
Lemma~\ref{lem:unbiased-bounded-variance-subclass-lower-bound}
gives
\[
    T
    =
    \Omega\left(
        \min\left\{
            \sqrt d\,\frac{\sigma}{\epsilon},
            \left(
                \frac{\sigma}{\epsilon}
            \right)^2
        \right\}
    \right).
\]
The dimensional assumption implies
\[
    \sqrt d\,\frac{\sigma}{\epsilon}
    \geq
    \left(
        \frac{\sigma}{\epsilon}
    \right)^2.
\]
Hence
\[
\begin{aligned}
&\min\left\{
    \sqrt d\,\frac{\sigma}{\epsilon},
    \left(
        \frac{\sigma}{\epsilon}
    \right)^2
\right\}
\\
&\qquad=
\left(
    \frac{\sigma}{\epsilon}
\right)^2.
\end{aligned}
\]
It follows that
\[
    T
    =
    \Omega\left(
        \left(
            \frac{\sigma}{\epsilon}
        \right)^2
    \right),
\]
which proves the third statement.

At the boundary
\[
    d=
    \left(
        \frac{\sigma}{\epsilon}
    \right)^2,
\]
the lower bounds in the second and third statements coincide. Indeed,
\begin{align*}
&d^{\frac{3p-4}{4(p-1)}}
\left(
    \frac{\sigma}{\epsilon}
\right)^{\frac{p}{2(p-1)}}
\\
&\qquad=
\left(
    \frac{\sigma}{\epsilon}
\right)^{
    \frac{3p-4}{2(p-1)}
    +
    \frac{p}{2(p-1)}
}
\\
&\qquad=
\left(
    \frac{\sigma}{\epsilon}
\right)^2.
\end{align*}
Thus the piecewise lower-bound landscape is consistent at the boundary,
and the proof is complete.
\end{proof}



\section{Proof in Section~\ref{sec:non-convex}}

\subsection{Proof of Theorem~\ref{thm:non-convex_heavy_tailed}}
\begin{proof}
    Lemma~\ref{lm:descent} implies that
    \begin{align*}
        \|\nabla f(\x_t)\| \leq \frac{ 3(f(\x_t)-f(\x_{t+1}))}{\eta} + 8 \|\bepsilon_t\| + \frac{3L\eta}{2},
    \end{align*}
    Conditioned on the randomness generated before iteration $t$, Theorem~\ref{thm:QHTME-for prob 1} guarantees
    \begin{align*}
        \Pr\!\left(\|\bepsilon_t\|>\epsilon_{\rm g}\mid \x_0,\ldots,\x_t\right)\leq\frac{\delta}{T}.
    \end{align*}
    A union bound therefore shows that, with probability at least $1-\delta$, $\|\bepsilon_t\|\leq\epsilon_{\rm g}$ simultaneously for all $t=0,\ldots,T-1$. On this event, summing the descent bound over $t$ and using $f(\x_T)\geq f^*$ gives
    \begin{align*}
       \frac{1}{T} \sum_{t=0}^{T-1} \|\nabla f(\x_t)\|
       \leq \frac{3\Delta_f}{\eta T}
       + 8\epsilon_{\rm g}
       + \frac{3L\eta}{2}.
    \end{align*}
    With $\eta=\epsilon/(3L)$, $T=\lceil36L\Delta_f\epsilon^{-2}\rceil$, and $\epsilon_{\rm g}=\epsilon/32$, it follows that
    \begin{align*}
       \frac{1}{T} \sum_{t=0}^{T-1} \|\nabla f(\x_t)\|
       \leq \frac{\epsilon}{4}+8\frac{\epsilon}{32}+\frac{\epsilon}{2}
       =\epsilon.
    \end{align*}
    In particular, at least one of the generated iterates is an $\epsilon$-stationary point. Each call to \texttt{QHTME} uses
    \begin{align*}
        n = \tilde{\OM}\!\left(\sqrt d\,\sigma^{\frac{p}{2(p-1)}}\epsilon^{-\frac{p}{2(p-1)}}\right)
    \end{align*}
    queries. Therefore, the total query complexity of the quantum stochastic gradient oracle is bounded by
    \begin{align*}
        T\cdot n = \tilde{\OM}\!\left(L\Delta_f\sqrt d\,\sigma^{\frac{p}{2(p-1)}}\epsilon^{-\frac{5p-4}{2(p-1)}}\right).
    \end{align*}
\end{proof}


\section{Proof in Section~\ref{sec:convex}}
\subsection{Proof of Lemma~\ref{lm:sgd}}
\begin{proof}
Our proof follows Theorem 1 and Corollary 1 in \citet{liu2025online}. 
    The optimality of the projection step indicates that
    \begin{align*}
        \inner{(\x_t-\eta\g_t)-\x_{t+1}}{\x-\x_{t+1}}\leq 0
    \end{align*}
    for all $\x\in\fX$. Rearranging it, we have
    \begin{align*}
        \eta_t\inner{\g_t}{\x_{t+1}-\x}\leq \inner{\x_{t}-\x_{t+1}}{\x_{t+1}-\x} = \frac{\|\x_t-\x\|^2-\|\x_{t+1}-\x\|^2 - \|\x_t-\x_{t+1}\|^2}{2}.
    \end{align*}
    Then, we have 
    \begin{align*}
        \inner{\g_t}{\x_t-\x}\leq \frac{\|\x_t-\x\|^2-\|\x_{t+1}-\x\|^2}{2\eta} + \inner{\g_t}{\x_t-\x_{t+1}} - \frac{\|\x_t-\x_{t+1}\|^2}{2\eta}.
    \end{align*}
    We bound $\inner{\g_t}{\x_t-\x_{t+1}}$ according to 
    \begin{align*}
        \inner{\g_t}{\x_{t}-\x_{t+1}}\leq \|\nabla f(\x_t)\|\|\x_t-\x_{t+1}\| + \|\bepsilon_t\|\|\x_{t}-\x_{t+1}\|.
    \end{align*}
 We have
    \begin{align*}
        \|\nabla f(\x_t)\|\|\x_t-\x_{t+1}\|\leq \eta G^2 + \frac{\|\x_t-\x_{t+1}\|^2}{4\eta},
    \end{align*}
    and using the equation (5) in \citet{liu2025online} leads to
    \begin{align*}
      \|\bepsilon_t\|\|\x_t-\x_{t+1}\| \leq C_p \eta^{p-1}\|\bepsilon\|^p  D^{2-p} + \frac{\|\x_t-\x_{t+1}\|^2}{4\eta},
    \end{align*}
    where $C_p = \frac{(4p-4)^{p-1}}{p^p}$.
Then, we have 
\begin{align*}
    \inner{\g_t}{\x_t-\x}\leq \frac{\|\x_t-\x\|^2-\|\x_{t+1}-\x\|^2}{2\eta}+\eta G^2 +C_p\eta^{p-1}\|\bepsilon\|^pD^{2-p}.
\end{align*}
Taking expectations on both sides and using the convexity of $f(\cdot)$ gives
\begin{align*}
  \EBP{f(\x_t)-f(\x)}
  &\leq \EBP{\inner{\g_t}{\x_t-\x}} \\
  &\leq \frac{\EBP{\|\x_t-\x\|^2}-\EBP{\|\x_{t+1}-\x\|^2}}{2\eta}
  +\eta G^2+C_p\eta^{p-1}D^{2-p}\EBP{\|\bepsilon_t\|^p}.
\end{align*}
Taking $\x=\x^*$ in the above inequality and summing both sides we have
\begin{align*}
    \EBP{f(\bar{\x}_T)}-f(\x^*)
    \leq \frac{D^2}{2\eta T} + \eta G^2
    + C_p\eta^{p-1}D^{2-p}\frac1T\sum_{t=0}^{T-1}\EBP{\|\bepsilon_t\|^p}.
\end{align*}

\end{proof}

\subsection{Proof of Theorem~\ref{thm:convex_heavy_tailed}}
\begin{proof}
By letting $\eta = \frac{G^{-2}\epsilon}{3}$ and $T= 3G^2D^2\epsilon^{-2}$, we have
\begin{align*}
    \frac{D^2}{2\eta T} + \eta G^2\leq \frac{5\epsilon}{6}.
\end{align*}
Let
\[
\tilde{\epsilon}^p= \frac{3^{p-1}}{6C_p}\epsilon^{2-p}G^{2(p-1)}D^{-(2-p)},
\qquad \EBP{\|\bepsilon_t\|^p}\leq\tilde{\epsilon}^p.
\]
Here $G$ enters through the deterministic gradient-norm bound, while $\sigma$ remains the heavy-tailed noise parameter used by \texttt{QUHTME}. Then we have
\begin{align*}
   C_p\eta^{p-1} D^{2-p} \EBP{\|\bepsilon_t\|^p}\leq \frac{\epsilon}{6}.
\end{align*}
Using Lemma~\ref{lm:sgd}, we have
\begin{align*}
   \EBP{f(\hat{\x})}-f(\x^*)= \EBP{f(\bar{\x}_T)}-f(\x^*)\leq \epsilon.
\end{align*}
Theorem~\ref{theorem: query complexity of QUHTME}, together with the one-query fallback in Algorithm~\ref{alg: QPSGD} when $\tilde\epsilon\geq\sigma$, guarantees $\EBP{\bepsilon_t\mid\x_t}=0$ and $\EBP{\|\bepsilon_t\|^p}\leq\tilde{\epsilon}^{p}$. The expected number of queries per iteration is
\begin{align}
 n =  \tilde{\OM}\!\left(\max\!\left\{\sqrt d\,\sigma^{\frac{p}{2(p-1)}}\tilde{\epsilon}^{-\frac{p}{2(p-1)}},1\right\}\right)
 = \tilde{\OM}\!\left(\sqrt d\,G^{-1}\sigma^{\frac{p}{2(p-1)}}D^{\frac{2-p}{2(p-1)}}\epsilon^{-\frac{2-p}{2(p-1)}}+1\right)
\end{align}
queries to the quantum stochastic gradient oracle. Thus, the expected total query complexity of Algorithm~\ref{alg: QPSGD} is bounded by
\begin{align*}
    T\cdot n =\tilde{\OM}\!\left(\sqrt d\,G D^{\frac{3p-2}{2(p-1)}}\sigma^{\frac{p}{2(p-1)}}\epsilon^{-\frac{3p-2}{2(p-1)}} + G^2D^2\epsilon^{-2}\right).
\end{align*}

\end{proof}

\end{document}